\documentclass{article} % For LaTeX2e
\usepackage{nips15submit_e,times}
\usepackage{hyperref}
\usepackage{url}

\usepackage{graphicx}
\usepackage{amsmath,amssymb,amsthm}
\usepackage{algorithm,algorithmic}
\usepackage{verbatim}

\title{Stop Wasting My Gradients: Practical SVRG}

\author{
Reza Babanezhad\textsuperscript{1}, Mohamed Osama Ahmed\textsuperscript{1}, Alim Virani\textsuperscript{2}, Mark Schmidt\textsuperscript{1}\\
Department of Computer Science\\
University of British Columbia\\
\textsuperscript{1}\{rezababa, moahmed, schmidtm\}@cs.ubc.ca,\textsuperscript{2}alim.virani@gmail.com 
\And
Jakub Kone{\v{c}}n{\'y}\\
School of Mathematics\\
University of Edinburgh\\
kubo.konecny@gmail.com
\And
Scott Sallinen\\
Department of Electrical and Computer Engineering\\
University of British Columbia\\
scotts@ece.ubc.ca 
}

\newtheorem{thm}{Proposition}
\newtheorem{lem}{Lemma}
\def\norm#1{\|#1\|}
\def\batch{\mathcal{B}}

\definecolor{blu}{rgb}{0,0,1}

\def\ex#1{\mathbb E\left[#1\right]}
\newcommand{\argmin}[1]{\mathop{\hbox{argmin}}_{#1}}

\def\batch{\mathcal{B}}
\newcommand{\prox}{\mathop{\rm prox}}

\def\x{x^*}
\def\la{\bar{L}}
\def\lr{\bar{L}_r}
\def\xt{x_t}
\def\tx{x^s}

\def\mb{\mu^s}
\def\tm{\mu^s}

\def\ex#1{\mathbb E\left[#1\right]}

\newcommand{\Exp}{\mathbb{E}}

\newcommand\numberthis{\addtocounter{equation}{1}\tag{\theequation}}

\nipsfinalcopy % Uncomment for camera-ready version

\begin{document}

\maketitle

\begin{abstract}
We present and analyze several strategies for improving the performance of stochastic variance-reduced gradient (SVRG) methods. We first show that the convergence rate of these methods can be preserved under a decreasing sequence of errors in the control variate, and use this to derive variants of SVRG that use growing-batch strategies to reduce the number of gradient calculations required in the early iterations. We further (i) show how to exploit support vectors to  reduce the number of gradient computations in the later iterations, (ii) prove that the commonly--used regularized SVRG iteration is justified and improves the convergence rate, (iii) consider alternate mini-batch selection strategies, and (iv) consider the generalization error of the method.
%The last three years have seen tremendous advances in the problem of minimizing
%a finite sum of smooth functions, a core optimization problem at the heart
%of many machine learning models. These advances are due to the development
%of linearly-convergent methods which only process one training example at a time.
%However, many of these these methods require a memory of previous gradient
%values and this makes them infeasible for many practical problems. Several authors have shown
%that this memory requirement can be relaxed, at the cost of evaluating two gradients per iteration
%and using occasional full passes through the data. 
%Stochastic variance-reduced gradient (SVRG) methods achieve a linear convergence rate
%However, these extra gradient evaluations
%significantly slow down the algorithm in the early iterations. In this work, we show that a fast
%linear convergence rate can be achieved even if the full-gradient evaluations are only done approximately,
%and we further propose a variant that typically only requires one gradient evaluation and has a faster convergence rate when far from the solution.
%We further show how identifying support vectors can further reduce the number of gradient evaluations, and 
%show that an implementation trick that is commonly-used for handling sparse datasets actually improves the convergence rate.
%The latter analysis also suggests a new more effective strategy for constructing mini-batches. Finally,
%while previous works consider the training objective, we also give a result showing the method has
%appealing generalization error properties.
\end{abstract}

\section{Introduction}

We consider the problem of optimizing the average of a finite but large sum of smooth functions,
\begin{equation}
\label{eq:one}
 \min_{x\in \mathbb R^d} f(x)=\frac 1 n \sum_{i=1}^{n} f_{i}(x).
\end{equation}
A huge proportion of the model-fitting procedures in machine learning can be mapped to this problem. This includes classic models like
least squares and logistic regression but also includes more advanced methods like conditional random fields and deep neural network models.
In the high-dimensional setting (large $d$), the traditional approaches for solving~\eqref{eq:one} are: \emph{full gradient} (FG) methods which have linear convergence rates but need to evaluate the gradient $f_i$ for all $n$ examples on every iteration, and \emph{stochastic gradient} (SG) methods which make rapid initial progress as they only use a single gradient on each iteration but ultimately have slower sublinear convergence rates.

Le Roux et al. \cite{roux2012stochastic} proposed the first general method, \emph{stochastic average gradient} (SAG), that only considers one training example on each iteration but still achieves a linear convergence rate. Other methods have subsequently been shown to have this property~\cite{schwartz12,mairal2013surrogate,defazio2014saga}, but these all
%such as stochastic dual coordinate ascent (SDCA)~\cite{schwartz12}, incremental surrogate optimization (MISO)~\cite{mairal2013surrogate}, and SAGA~\cite{defazio2014saga}. 
%However, these methods 
require storing a previous evaluation of the gradient $f_i'$ or the dual variables for each $i$. For many objectives this only requires $O(n)$ space, but for general problems this requires $O(np)$ space making them impractical.

Recently, several methods have been proposed with similar convergence rates to SAG but without the memory requirements~\cite{mahdavi2013mixedGrad,johnson2013accelerating,zhang2013linear,konevcny2013semi}. They are known as \emph{mixed gradient}, \emph{stochastic variance-reduced gradient} (SVRG), and \emph{semi-stochastic gradient} methods (we will use SVRG). We give a canonical SVRG algorithm in the next section, but the salient features of these methods are that they evaluate two gradients on each iteration and occasionally must compute the gradient on all examples. SVRG methods often dramatically outperform classic FG and SG methods, but these extra evaluations mean that SVRG is slower than SG methods in the important early iterations. They also mean that SVRG methods are typically slower than memory-based methods like SAG.

In this work we first show that SVRG is robust to inexact calculation of the full gradients it requires (\S\ref{sec:error}), provided the accuracy increases over time. We use this to explore growing-batch strategies that require fewer gradient evaluations when far from the solution, and we propose a mixed SG/SVRG method that may improve performance in the early iterations (\S\ref{sec:batching}). We next explore  using support vectors to  reduce the number of gradients required when close to the solution (\S\ref{sec:support}), give a justification for the regularized SVRG update that is commonly used in practice (\S\ref{sec:exact}), consider alternative mini-batch strategies (\S\ref{sec:mini}), and finally consider the generalization error of the method (\S\ref{sec:test}).

\section{Notation and SVRG Algorithm}

SVRG assumes $f$ is $\mu$-strongly convex, each $f_i$ is convex, and each gradient $f_i'$ is Lipschitz-continuous with constant $L$. The method begins with an initial estimate $x^0$, sets $x_0 = x^0$ and then generates a sequence of iterates $x_t$ using
\begin{equation}
\label{eq:SVRG}
x_t = x_{t-1} - \eta(f_{i_t}'(x_{t-1}) - f_{i_t}'(x^s) + \mu^s),
\end{equation}
where $\eta$ is the positive step size, we set $\mu^s = f'(x^s)$, and $i_t$ is chosen uniformly from $\{1,2,\dots,n\}$. After every $m$ steps, we set $x^{s+1} =x_t$ for a random $t \in \{1,\dots,m\}$, and we reset $t=0$ with $x_0 = x^{s+1}$. 

To analyze the convergence rate of SVRG, we will find it convenient to define the function
\[
\rho(a,b) = \frac{1}{1-2\eta a} \left(\frac1{m\mu\eta}+ 2b\eta\right).
\]
as it appears repeatedly in our results. We will use $\rho(a)$ to indicate the value of $\rho(a,b)$ when $a=b$, and we will simply use $\rho$ for the special case when $a=b=L$. Johnson \& Zhang~\cite{johnson2013accelerating} show that if $\eta$ and $m$ are chosen such that $ 0< \rho < 1$, the algorithm achieves a linear convergence rate of the form
\[
 \mathbb{E}[f(x^{s+1}) - f(x^*) ] \leq \rho \mathbb{E}[f(x^s)- f(x^*)],
\]
where $x^*$ is the optimal solution.
This convergence rate is very fast for appropriate $\eta$ and $m$. While this result relies on constants we may not know in general, practical choices with good empirical performance include setting $m=n$, $\eta = 1/L $, and using $x^{s+1} = x_m$ rather than a random iterate.

Unfortunately, the SVRG algorithm requires $2m + n$ gradient evaluations for every $m$ iterations of~\eqref{eq:SVRG}, since updating $x_t$ requires two gradient evaluations and computing $\mu^s$ require $n$ gradient evaluations. We can reduce this to $m+n$ if we store the gradients $f_i'(x^s)$, but this is not practical in most applications.
Thus, SVRG requires many more gradient evaluations than classic SG iterations of memory-based methods like SAG.

\section{SVRG with Error}
\label{sec:error}

We first give a result for the SVRG method where we assume that $\mu^s$ is equal to $f'(x^s)$ up to some error $e^s$. This is in the spirit of the analysis of~\cite{SchmidtLeRouxBach11}, who analyze FG methods under similar assumptions. We assume that $\norm{x_t-x^*} \leq Z$ for all $t$, which has been used in related work~\cite{kwok2009asg} and is reasonable because of the coercity implied by strong-convexity.
\begin{thm}
\label{thm:thm1}
If $\mu^s = f'(x^s) + e^s$ and we set $\eta$ and $m$ so that $\rho < 1$, then the SVRG algorithm~\eqref{eq:SVRG} with $x^{s+1}$ chosen randomly from $\{x_1,x_2,\dots,x_m\}$ satisfies
\[
\mathbb{E}[f(x^{s+1})-f(x^{*})] \leq \rho \mathbb E[f(x^s)-f(x^{*})] +  \frac {Z \mathbb E \|e^s\|+\eta \mathbb E\|e^s\|^{2}}{1-2\eta L}.
\]
\end{thm}
We give the proof in Appendix~A. This result implies that SVRG does not need a very accurate approximation of $f'(x^s)$ in the crucial early iterations since the first term in the bound will dominate. Further, this result implies that we can maintain the \emph{exact} convergence rate of SVRG as long as the errors $e^s$ decrease at an appropriate rate. For example, we obtain the same convergence rate provided that $\max\{\mathbb{E}\norm{e^s},\mathbb{E}\norm{e^s}^2\} \leq \gamma\tilde{\rho}^s$ for any $\gamma \geq 0$ and some $\tilde{\rho} < \rho$. Further, we still obtain a linear convergence rate as long as $\norm{e^s}$ converges to zero with a linear convergence rate.

\subsection{Non-Uniform Sampling}

Xiao \& Zhang~\cite{xiao2014proximal} show that non-uniform sampling (NUS) improves the performance of SVRG. They assume each $f_i'$ is $L_i$-Lipschitz continuous, and sample $i_t = i$ with probability $L_i/n\bar{L}$ where $\bar{L}=\frac 1 n \sum_{i=1}^n L_i$. The iteration is then changed to
\[
x_{t}=x_{t-1}-\eta\left(\frac {\bar{L}} {L_{i_t}}[f_{i_{t}}^{'}(x_{t-1})-f_{i_{t}}^{'}(\tilde{x})]+\mu^s\right),
\]
%which maintains that the search direction is an unbiased approximation of the gradient provided that $\mu^s$ is the exact gradient $f'(x^s)$. 
which maintains that the search direction is unbiased.
In Appendix~A, we show that if $\mu^s$ is computed with error for this algorithm and if we set $\eta$ and $m$ so that $0 < \rho(\bar{L}) < 1$,
then we have a convergence rate of
\[
\mathbb{E}[f(x^{s+1})-f(x^{*})] \leq \rho(\bar{L}) \mathbb E[f(x^s)-f(x^{*})] + \frac {Z \mathbb E \|e^s\|+\eta \mathbb E\|e^s\|^{2}}{1-2\eta \bar{L}},
\]
which can be faster since the average $\bar{L}$ may be much smaller than the maximum value $L$.

\begin{algorithm}[tb]
 \caption{Batching SVRG}
   \label{alg:SVRGB}
 \begin{algorithmic}
 \STATE {\bfseries Input:} initial vector $x^0$, update frequency {\it m}, learning rate $\eta$. \FOR{$s=0,1,2,\dots$} 
	\STATE Choose batch size $|\batch^s|$
\STATE $\batch^s$ = $|\batch^s|$ elements sampled without replacement from $\{1,2,\dots,n\}$.
\STATE $\mu^s=\frac{1}{|\batch^s|}\sum_{i\in\batch^s}f_{i}^{'}(x^s)$\\
\STATE $x_{0}$=$x^s$\\
  \FOR{$t=1,2,\dots,m$} 
           \STATE Randomly pick $i_{t} \in {1,\dots, n}$\\
           \STATE $x_{t}=x_{t-1}-\eta(f_{i_{t}}^{'}(x_{t-1})-f_{i_{t}}^{'}(x^s)+\mu^s)\hfill(*)$\\
  \ENDFOR
 \STATE {\bfseries option I:} set $x^{s+1}=x_{m}$\\
\STATE {\bfseries option II:} set $x^{s+1}=x_{t}$ for random $t \in \{1,\dots, m\}$\\
\ENDFOR
\end{algorithmic}
\end{algorithm}

\subsection{SVRG with Batching}

There are many ways we could allow an error in the calculation of $\mu^s$ to speed up the algorithm. For example, if evaluating each $f_i'$ involves solving an optimization problem, then we could solve this optimization problem inexactly. For example, if we are fitting a graphical model with an iterative approximate inference method, we can terminate the iterations early to save time.

When the $f_i$ are simple but $n$ is large, a natural way to approximate $\mu^s$ is with a subset (or `batch') of training examples $\batch^s$ (chosen \emph{without} replacement),
\[
\mu^s = \frac{1}{|\batch^s|}\sum_{i\in\batch^s}f_i'(x^s).
\]
%This is justified because when we are far from the solution, most $f_i'(x^s)$ likely point in directions of progress. 
The batch size $|\batch^s|$ controls the error in the approximation, and we can drive the error to zero by increasing it to $n$. Existing SVRG methods correspond to the special case where $|\batch^s|=n$ for all $s$. 

Algorithm~\ref{alg:SVRGB} gives pseudo-code for an SVRG implementation that uses this sub-sampling strategy. 
If we assume that the sample variance of the norms of the gradients is bounded by $S^2$ for all $x^s$,
\[
\frac{1}{n-1}\sum_{i=1}^n\left[\norm{f_i'(x^s)}^2 - \norm{f'(x^s)}^2\right] \leq S^2,
\]
then we have that~\cite[Chapter~2]{lohr2009sampling}
\[ 
\mathbb{E}\norm{e^s}^2 \leq \frac{n-|\batch^s|}{n|\batch^s|}S^2.
\]
So if we want $\mathbb{E}\norm {e^s}^2 \leq \gamma\tilde{\rho}^{2s}$, where $\gamma \geq 0$ is a constant for some $\tilde{\rho} < 1$,  we need
 \begin{equation}
\label{eq:Bs}
 |\batch^s| \geq \frac{nS^2}{S^2 + n\gamma\tilde{\rho}^{2s}}. 
\end{equation}
If the batch size satisfies the above condition then 
\begin{align*}
 Z\mathbb E \|e^{s-1}\|+\eta \mathbb E\|e^{s-1}\|^{2} & \leq Z\sqrt{\gamma}\tilde{\rho}^s+\eta \gamma\tilde{\rho}^{2s}\\
%& \leq 2\max\{Z\sqrt{\gamma}\tilde{\rho}^s,\eta \gamma\tilde{\rho}^{2s}\}\\
& \leq 2\max\{Z\sqrt{\gamma},\eta\gamma\tilde{\rho}\}\tilde{\rho}^s,
\end{align*}
and \emph{the convergence rate of SVRG is unchanged compared to using the full batch} on all iterations. 

The condition~\eqref{eq:Bs}  guarantees a linear convergence rate under any exponentially-increasing sequence of batch sizes, the strategy suggested by~\cite{friedlander2011hybrid} for classic SG methods.
 However, a tedious calculation shows that~\eqref{eq:Bs} has an inflection point at $s = \log (S^2/\gamma n)/2\log(1/\tilde{\rho})$, corresponding to $|\batch^s| = \frac n2$. This was previously observed empirically~\cite[Figure~3]{aravkin2012robust}, and occurs because we are sampling without replacement. This transition means we don't need to increase the batch size exponentially.
% Because of this transition, it is only necessary to increase the batch size exponentially until we use half the data to compute $\mu^s$, and for larger values of $s$ the batch-size can increase more slowly. 

\section{Mixed SG and SVRG Method}
\label{sec:batching}

An approximate $\mu^s$ can drastically reduce the computational cost of the SVRG algorithm, but does not affect the $2$ in the $2m+n$ gradients required for $m$ SVRG iterations. This factor of $2$ is significant in the early iterations, since this is when stochastic methods make the most progress and when we typically see the largest reduction in the \emph{test} error.

To reduce this factor, we can consider a \emph{mixed} strategy: if $i_t$ is in the batch $\batch^s$ then perform an SVRG iteration, but if $i_t$ is not in the current batch then use a classic SG iteration. We illustrate this modification in Algorithm~\ref{alg:SVRSGD}. This modification allows the algorithm to take advantage of the rapid initial progress of SG, since it  predominantly uses SG iterations when far from the solution. Below we give a convergence rate for this mixed strategy.

\begin{algorithm}[tb]
 \caption{Mixed SVRG and SG Method}
  \label{alg:SVRSGD}
  \begin{algorithmic}
\STATE Replace (*) in Algorithm 1 with the following lines: 
 \IF {$f_{i_t} \in \batch^s$ }
           \STATE $x_{t}=x_{t-1}-\eta(f_{i_{t}}^{'}(x_{t-1})-f_{i_{t}}^{'}(x^s)+\mu^s)$
 \ELSE
           \STATE $x_{t}=x_{t-1}-\eta f_{i_{t}}^{'}(x_{t-1})$
 \ENDIF
  
\end{algorithmic}
\end{algorithm}

\begin{thm}
\label{thm:thm2}
Let $\mu^s = f'(x^s) + e^s$ and we set $\eta$ and  $m$ so that $0 < \rho(L,\alpha L) < 1$ with $\alpha = |\batch^s|/n$. If we assume $\mathbb E \norm{f_i'(x)}^2 \leq \sigma^2$ then Algorithm~\ref{alg:SVRSGD} has
\[
\mathbb{E}[f(x^{s+1})-f(x^{*})] \leq \rho(L,\alpha L) \mathbb E[f(x^s)-f(x^{*})] + \frac{Z \mathbb E \|e^{s}\|+\eta \mathbb E\|e^{s}\|^{2} + \frac{\eta\sigma^2}{2}(1-\alpha)}{1-2\eta L}
%+ \frac {1-2\eta L}\left(Z \mathbb E \|e^{s}\|+\eta \mathbb E\|e^{s}\|^{2} + \frac{\eta(1-|\batch^s|/n)\sigma^2}{2}\right).
\]
\end{thm}
We give the proof in Appendix~B. The extra term depending on the variance $\sigma^2$ is typically the bottleneck for SG methods. Classic SG methods require the step-size $\eta$ to converge to zero because of this term. However, the mixed SG/SVRG method can keep the fast progress from using a constant $\eta$ since the term depending on $\sigma^2$ converges to zero as $\alpha$ converges to one.
Since $\alpha < 1$ implies that  $\rho(L,\alpha L) < \rho$, this result implies that when $[f(x^s) - f(x^*)]$ is large compared to $e^s$ and $\sigma^2$ that the mixed SG/SVRG method actually converges faster.

Sharing a single step size $\eta$ between the SG and SVRG iterations in Proposition~\ref{thm:thm2} is sub-optimal. For example, if $x$ is close to $x^*$ and $|\batch^s| \approx n$, then the SG iteration might actually take us far away from the minimizer.  Thus, we may want to use a decreasing sequence of step sizes for the SG iterations.
%Thus, it may make sense to use a decreasing sequence of step sizes for the SG iterations. 
In Appendix~B, we show that using $\eta = O^*(\sqrt{(n-|\mathcal B|)/n|\mathcal B|})$ for the SG iterations can improve the dependence on the error $e^s$ and variance $\sigma^2$.

\section{Using Support Vectors}
\label{sec:support}

Using a batch $\batch^s$ decreases the number of gradient evaluations required when SVRG is far from the solution, but its benefit diminishes over time. However, for certain objectives we can further reduce the number of gradient evaluations by identifying \emph{support vectors}. For example, consider minimizing the Huberized hinge loss (HSVM) with threshold $\epsilon$~\cite{rosset2007piecewise},
\[
\min_{x\in\mathbb{R}^d}  \frac{1}{n}\sum_{i=1}^nf(b_ia_i^Tx), \quad
f(\tau) = \begin{cases}
0 & \text{if $ \tau > 1 + \epsilon$,}\\
1 - \tau & \text{if $\tau < 1 - \epsilon$,}\\
\frac{(1 + \epsilon - \tau)^2}{4\epsilon} & \text{if $|1-\tau| \leq \epsilon$,}
\end{cases}
\]
In terms of~\eqref{eq:one}, we have $f_i(x) = f(b_ia_i^Tx)$.
The performance of this loss function is similar to logistic regression and the hinge loss, but it has the appealing properties of both: it is \emph{differentiable} like logistic regression meaning we can apply methods like SVRG, but it has \emph{support vectors} like the hinge loss meaning that many examples will have $f_i(x^*)=0$ and $f_i'(x^*) = 0$. We can also construct Huberized variants of many non-smooth losses for regression and multi-class classification.

If we knew the support vectors where $f_i(x^*) > 0$, we could solve the problem faster by ignoring the non-support vectors. For example, if there are $100000$ training examples but only $100$ support vectors in the optimal solution, we could solve the problem $1000$ times faster.
While we typically don't know the support vectors, in this section we outline a heuristic that gives large practical improvements by trying to identify them as the algorithm runs.

Our heuristic has two components. The first component is maintaining the \emph{list of non-support vectors at $x^s$}. Specifically, we maintain a list of examples $i$ where $f_i'(x^s)=0$. When SVRG picks an example $i_t$ that is part of this list, we know that $f_{i_t}'(x^s)=0$ and thus the iteration only needs one gradient evaluation. This modification is not a heuristic, in that it still applies the exact SVRG algorithm. However, at best it can only cut the runtime in half. 

The heuristic part of our strategy is to skip $f_i'(x^s)$ or $f_i'(x_t)$ if our evaluation of $f_i'$ has been zero more than two consecutive times (and skipping it an exponentially larger number of times each time it remains zero). Specifically, for each example $i$ we maintain two variables, $sk_i$ (for `skip') and $ps_i$ (for `pass'). Whenever we need to evaluate $f_i'$ for some $x^s$ or $x_t$, we run Algorithm~\ref{alg:SV} which may skip the evaluation. This strategy can lead to huge computational savings in later iterations if there are few support vectors, since many iterations will require no gradient evaluations.

\begin{algorithm}[tb]
 \caption{Heuristic for skipping evaluations of $f_i$ at $x$}
  \label{alg:SV}
  \begin{algorithmic}
 \IF {$sk_i = 0$ }
 	\STATE compute $f_i'(x)$.
	\IF {$f_i'(x) = 0$}
	\STATE $ps_i = ps_i + 1$. \hfill\{Update the number of consecutive times $f_i'(x)$ was zero.\}
	\STATE $sk_i = 2^{\max\{0,ps_i-2\}}$. \hfill\{Skip exponential number of future evaluations if it remains zero.\}
	\ELSE
	\STATE $ps_i = 0$. \hfill\{This could be a support vector, do not skip it next time.\}
	\ENDIF
 	\STATE return $f_i'(x)$.
 \ELSE
\STATE $sk_i = sk_i - 1$. \hfill\{In this case, we skip the evaluation.\}
\STATE return 0.
 \ENDIF
  
\end{algorithmic}
\end{algorithm}

%\textbf{1. Maintaining a list of support vectors at $x^s$}: Let $\tau_i^s = b_ia_i^Tx^s$ and $\tau_i^t = b_ia_i^Tx_t$. When computing $\mu^s$, it is straightforward to maintain a set of $n$ indicator variables that measure whether $f'(\tau_i^s) = 0$ for each $i$. Using these indicators, we can skip evaluating $f'(\tau_{i_t}^s)$ when we know it will be zero. These SVRG iterations thus only require one gradient evaluation, even when $i_t \in \batch^s$.

%\textbf{2. Bounding potential support vectors at $x_t$}: By storing the scalars $\tau_i^s$, we can use them to bound whether $\tau_i^t > 1+\epsilon$ and hence $f'(\tau_i^t) = 0$. With such a bound, the SVRG iteration may require $0$ gradients (i.e., we ignore the example) and we can also reduce the cost of computing $\mu^s$ if we know that $f'(\tau_i^{s+1})=0$. A simple bound of this type is based on  Cauchy-Schwartz is (using $d_t = x^s - x_t$),
%\[
%\tau_i^t = b_ia_i^Tx_t= b_ia_i^T(x^s - d_t) = \tau_i^s - b_ia_i^Td_t \geq \tau_i^s - |a_i^Td_t| \geq \tau_i^s -\norm{a_i}\norm{d_t}.
%\]
%We can compute this bound in $O(1)$ given $\norm{d_t}$ and $\norm{a_i}$.
%If we can't efficiently track $\norm{d_t}$, the triangle inequality can provide more crude bounds that are easier to track such as 
%\begin{equation}
%\norm{d_t} = \norm{x^s - x_t} = \left|\left|\sum_{j=1}^t [x_{t-1}-x_t]\right|\right| \leq \sum_{j=1}^t\norm{x_{t-1}-x_t}.
%\label{eq:dBound}
%\end{equation}
%This weakens over time, but we can always reset it by computing $d_t = x^s-x_t$ for the current $t$.

Identifying support vectors to speed up computation has long been an important part of SVM solvers, and is related to the classic shrinking heuristic~\cite{Joachims99a}. While it has previously been explored in the context of dual coordinate ascent methods~\cite{usunier2010guarantees}, this is the first work exploring it for linearly-convergent stochastic gradient methods.
%The inner-outer iteration structure of SVRG seems well-suited to exploit support vectors, and this is the first work to explore utilizing support vectors to speed up computation in the context of linearly-convergent stochastic optimization methods. However, note that the above strategies are not heuristics: 
%using them yields the exact same algorithm but with fewer gradient calculations. We should cite Bordes and Bottou here.

\section{Regularized SVRG}
\label{sec:exact}

We are often interested in the special case where problem~\eqref{eq:one} has the decomposition
\begin{equation}
\label{eq:tps}
 \min_{x\in \mathbb R^d} \;f(x)\equiv h(x)+ \frac 1 n \sum_{i=1}^{n} g_{i}(x).
\end{equation}
A common choice of $h$ is a scaled $1$-norm of the parameter vector, $h(x) = \lambda\norm{x}_1$. This non-smooth regularizer encourages sparsity in the parameter vector, and can be addressed with the proximal-SVRG method  of Xiao \& Zhang~\cite{xiao2014proximal}. Alternately, if we want an explicit $Z$ we could set $h$ to the indicator function for a $2$-norm ball containing $x^*$. In Appendix~C, we give a variant of Proposition~\ref{thm:thm1} that allows errors in the proximal-SVRG method for non-smooth/constrained settings like this.

Another common choice is the $\ell_2$-regularizer, $h(x) = \frac{\lambda}{2}\norm{x}^2$. With this regularizer, the SVRG updates can be equivalently written in the form
%Although this regularizer can also be addressed with proximal-gradient methods, practical implementations of SVRG instead change the iteration to use
\begin{equation}
\label{eq:regSVRG}
x_{t+1}=x_t-\eta\left( h'(x_t) + g'_{i_t}(x_t)-g'_{i_t}(x^s)+\mu^s\right),
\end{equation}
where $\mu^s = \frac 1n\sum_{i=1}^ng_i(x^s)$.
That is, they take an exact gradient step with respect to the regularizer and an SVRG step with respect to the $g_i$ functions. When the $g_i'$ are sparse, this form of the update allows us to implement the iteration without needing full-vector operations. A related update is used by Le Roux et al. to avoid full-vector operations in the SAG algorithm~\cite[\S 4]{roux2012stochastic}. In Appendix~C, we prove the below convergence rate for this update.
\begin{thm}
\label{thm:er}
Consider instances of problem~\eqref{eq:one} that can be written in the form~\eqref{eq:tps} where $h'$ is $L_h$-Lipschitz continuous and each $g_i'$ is $L_g$-Lipschitz continuous, and assume that we set $\eta$ and $m$ so that $0 < \rho(L_m) < 1$ with $L_m=\max\{ L_g, L_h\}$. Then the regularized SVRG iteration~\eqref{eq:regSVRG} has
\[
\mathbb{E}[f(x^{s+1})-f(x^{*})] \leq \rho(L_m) \mathbb{E}[f(x^s)-f(x^{*})],
\]
\end{thm}
Since $L_m \leq L$, and strictly so in the case of $\ell_2$-regularization, this result shows that for $\ell_2$-regularized problems SVRG actually converges faster than the standard analysis would indicate (a similar result appears in Kone\v{c}n\'{y} et al.~\cite{konevcny2014ms2gd}).
%While the regularized iteration is appealing due to its low cost for sparse data sets, since $L_m \leq L$ this result shows that the trick can actually improve the convergence rate. 
Further, this result gives a theoretical justification for using the update~\eqref{eq:regSVRG} for other $h$ functions where it is not equivalent to the original SVRG method.

\section{Mini-Batching Strategies}
\label{sec:mini}

Kone\v{c}n\'{y} et al.~\cite{konevcny2014ms2gd}  have also recently considered using batches of data within SVRG. They consider using `mini-batches' in the inner iteration (the update of $x_t$) to decrease the variance of the method, but still use full passes through the data to compute $\mu^s$. This prior work is thus complimentary to the current work (in practice, both strategies can be used to improve performance). In Appendix~D we show that sampling the inner mini-batch proportional to $L_i$ achieves a convergence rate of
\[
\ex{f(x^{s+1})-f(x^*)} \leq \rho_M\ex{f(x^{s})-f(x^*)},
\]
where $M$ is the size of the mini-batch while
\[
\rho_M = \frac{1}{M-2\eta \bar{L}}\left(\frac {M}{m\mu\eta}+ 2\bar{L}\eta\right),
\]
and we assume $0 < \rho_M < 1$.
This generalizes the standard rate of SVRG and improves on the result of Kone\v{c}n\'{y} et al.~\cite{konevcny2014ms2gd} in the smooth case. This rate can be faster than the rate of the standard SVRG method at the cost of a more expensive iteration, and may be clearly advantageous in settings where parallel computation allows us to compute several gradients simultaneously.

The regularized SVRG form~\eqref{eq:regSVRG} suggests an alternate mini-batch strategy for problem~\eqref{eq:one}: consider a mini-batch that contains a `fixed' set $\mathcal B_f$ and a `random' set $\mathcal B_t$. Without loss of generality, assume that we sort the  $f_i$ based on their $L_i$ values so that $L_1 \geq L_2 \geq \dots \geq L_n$. For the fixed $\mathcal B_f$ we will \emph{always} choose the $M_f$ values with the largest $L_i$, $\mathcal B_f=\{f_1,f_2,\dots,f_{M_f}\}$. In contrast, we choose the members of the random set $\mathcal B_t$ by sampling from $B_r = \{ f_{M_f+1},\dots, f_n\}$ proportional to their Lipschitz constants, $p_i = \frac {L_i}{(M_r)\bar{L}_r}$ with $\bar{L}_r=(1/M_r)\sum_{i=M_f+1}^n L_i$. In Appendix~D, we show the following convergence rate for this strategy:
%In real scenario setting, if we have a GPU that can compute several gradients extremely quickly but it is expensive to copy data into the GPU from memory, the GPU can repeatedly compute the gradients for the same examples. In this setting, $\mathcal B_f$ would be the examples on the GPU and CPU can process examples in $\mathcal B_r$. 
\begin{thm}
\label{thm:bmb}
Let $g(x)=(1/n)\sum_{i \notin [\batch_f]} f_i(x)$ and $h(x)=(1/n)\sum_{i\in[\batch_f]} f_i(x)$. If we replace the SVRG update with 
\[
x_{t+1}=x_t-\eta\left( h'(x_t) +(1/M_r) \sum_{i \in \mathcal B_t }\frac{\bar{L}_r}{L_i}(f'_i(x_t)-f'_i(x^s))+g'(x^s)\right), 
\]
 then the convergence rate is 
\[
\mathbb E[{f(x^{s+1})-f(x^*)}] \leq \rho(\kappa,\zeta)\mathbb E[{F(x^s)-f(x^*)}].
\]
where $\zeta=\frac{(n-M_f)\bar{L}_r}{(M-M_f)n} $ and $\kappa = \max\{\frac {L_1}{n},\zeta\}$. 
\end{thm}
If $L_1 \leq n\bar{L}/M$ and $M_f < \frac {(\alpha-1)nM}{\alpha n - M}$ with $\alpha = \frac {\bar{L}}{\bar{L}_r}$, then we get a faster convergence rate than SVRG with a mini-batch of size $M$.
The scenario where this rate is slower than the existing mini-batch SVRG strategy is when $L_1 \leq n\bar{L}/M$. But we could relax this assumption by dividing each element of the fixed set $\mathcal B_f$ into two functions: $\beta f_i$ and $(1-\beta) f_i$, where $\beta = 1/M$, then replacing each function $f_i$ in $\mathcal B_f$ with $\beta f_i$ and adding $(1-\beta) f_i$ to the random set $B_r$. This result may be relevant if we have access to a field-programmable gate array (FPGA) or graphical processing unit (GPU) that can compute the gradient for a fixed subset of the examples very efficiently. However, our experiments (Appendix~F) indicate this strategy only gives marginal gains. 

In Appendix~F, we also consider constructing mini-batches by sampling proportional to $f_i(x^s)$ or $\norm{f_i'(x^s)}$. These seemed to work as well as Lipschitz sampling on all but one of the datasets in our experiments, and this strategy is appealing because we have access to these values while we may not know the $L_i$ values. However, these strategies diverged on one of the datasets.

%. In these settings, we can often compute the gradient for a subset of the examples $h$ very efficiently, but it is expensive to transfer examples onto the FPGA or GPU. In this setting, we can view $h$ as the examples we compute on the FPGA/GPU and this result shows that it reasonable to have these computations focus on a subset of the training data.

\section{Learning efficiency}
\label{sec:test}

In this section we compare the performance of SVRG as a large-scale learning algorithm compared to FG and SG methods. Following Bottou \& Bousquet~\cite{bottou-bousquet-2011},
%A more detialed explanation is provided in the Appendix. 
we can formulate the generalization error $\mathcal{E}$ of a learning algorithm as the sum of three terms
$$ \mathcal{E} = \mathcal{E}_\text{app}+\mathcal{E}_\text{est}+\mathcal{E}_\text{opt} $$
where the approximation error $\mathcal{E}_\text{app}$ measures the effect of using a limited class of models, the estimation error  $\mathcal{E}_\text{est}$ measures the effect of using a finite training set, and the optimization error $\mathcal{E}_\text{opt}$ measures the effect of inexactly solving problem~\eqref{eq:one}.
Bottou \& Bousquet~\cite{bottou-bousquet-2011} study asymptotic performance of various algorithms for a fixed approximation error and under certain conditions on the distribution of the data depending on parameters $\alpha$ or $\nu$. In Appendix~E, we discuss how SVRG can be analyzed in their framework. The table below includes SVRG among their results.
\begin{table}[!h]
\begin{center}
\begin{tabular}{c c c c}
\hline
Algorithm & Time to reach $ \mathcal{E}_{\text{opt}} \leq \epsilon$ & Time to reach $\mathcal{E} = O(\mathcal{E}_{app} + \epsilon)$ & Previous with $\kappa \sim n$ \\
\hline
FG & $\mathcal{O} \left( n \kappa d \log \left(\frac{1}{\epsilon}\right) \right)$ & $\mathcal{O} \left( \frac{d^2 \kappa}{\epsilon^{1/\alpha}} \log^2 \left(\frac{1}{\epsilon}\right) \right) $ & $\mathcal{O} \left( \frac{d^3}{\epsilon^{2/\alpha}} \log^3 \left(\frac{1}{\epsilon}\right) \right) $ \\
SG & $\mathcal{O} \left( \frac{d \nu \kappa^2}{\epsilon} \right) $ & $ \mathcal{O} \left( \frac{d \nu \kappa^2}{\epsilon} \right) $ & $ \mathcal{O} \left( \frac{d^3 \nu}{\epsilon} \log^2 \left(\frac{1}{\epsilon}\right) \right) $ \\
SVRG & $\mathcal{O} \left( (n + \kappa)d \log \left(\frac{1}{\epsilon}\right) \right) $ & $ \mathcal{O} \left( \frac{d^2}{\epsilon^{1/\alpha}} \log^2 \left(\frac{1}{\epsilon} \right) + \kappa d \log \left(\frac{1}{\epsilon} \right) \right) $ & $ \mathcal{O} \left( \frac{d^2}{\epsilon^{1/\alpha}} \log^2 \left(\frac{1}{\epsilon} \right) \right) $ \\
\hline
\end{tabular}
\label{tbl:learningComplexity}
\end{center}
\end{table}

\noindent In this table, the condition number is $\kappa = L/\mu$. In this setting, linearly-convergent stochastic gradient methods can obtain better bounds for ill-conditioned problems, with a better dependence on the dimension and without depending on the noise variance $\nu$.

\section{Experimental Results}

\begin{figure*}[!ht]
\centering
\includegraphics[width=.4\textwidth]{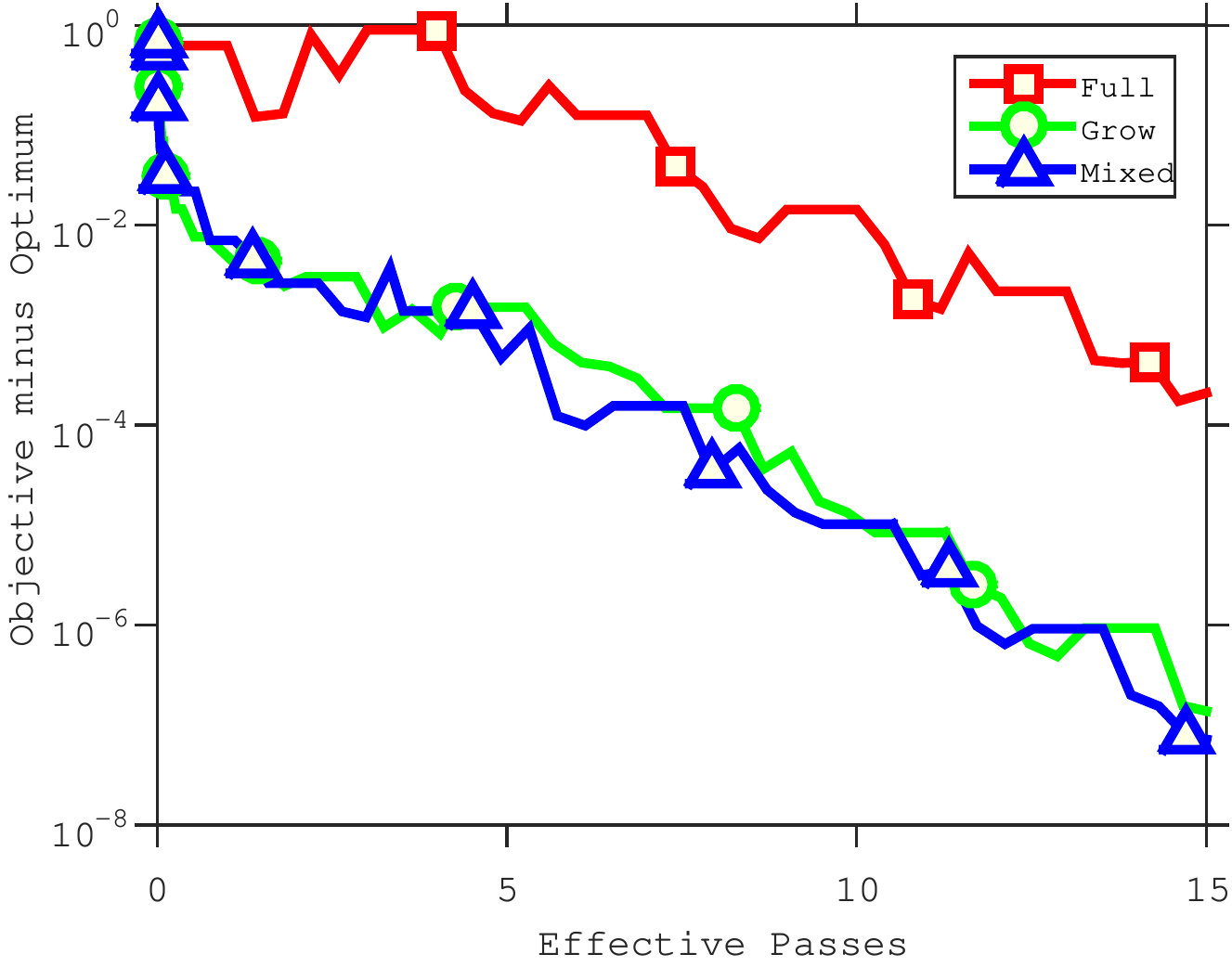}
\includegraphics[width=.4\textwidth]{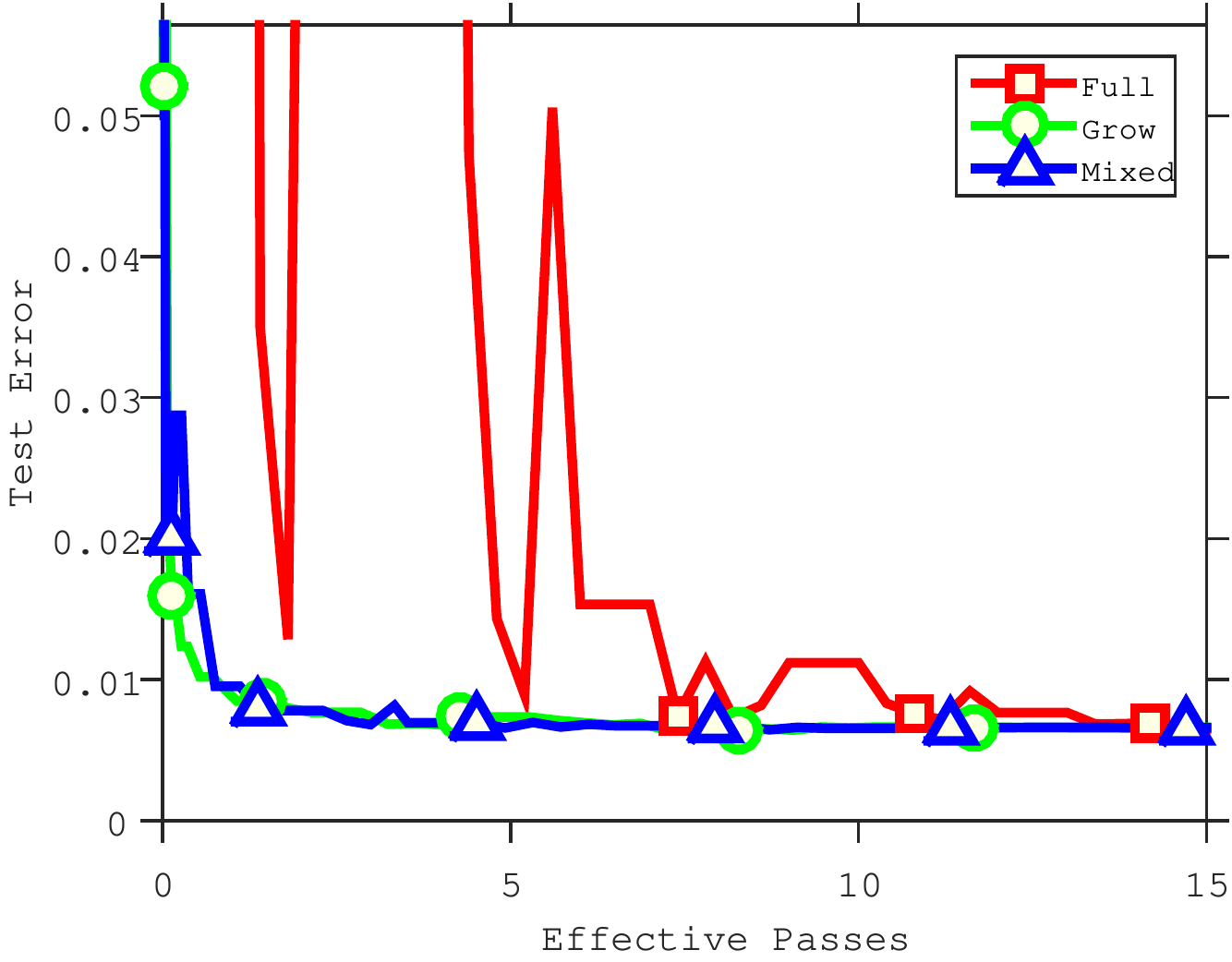}
\includegraphics[width=.4\textwidth]{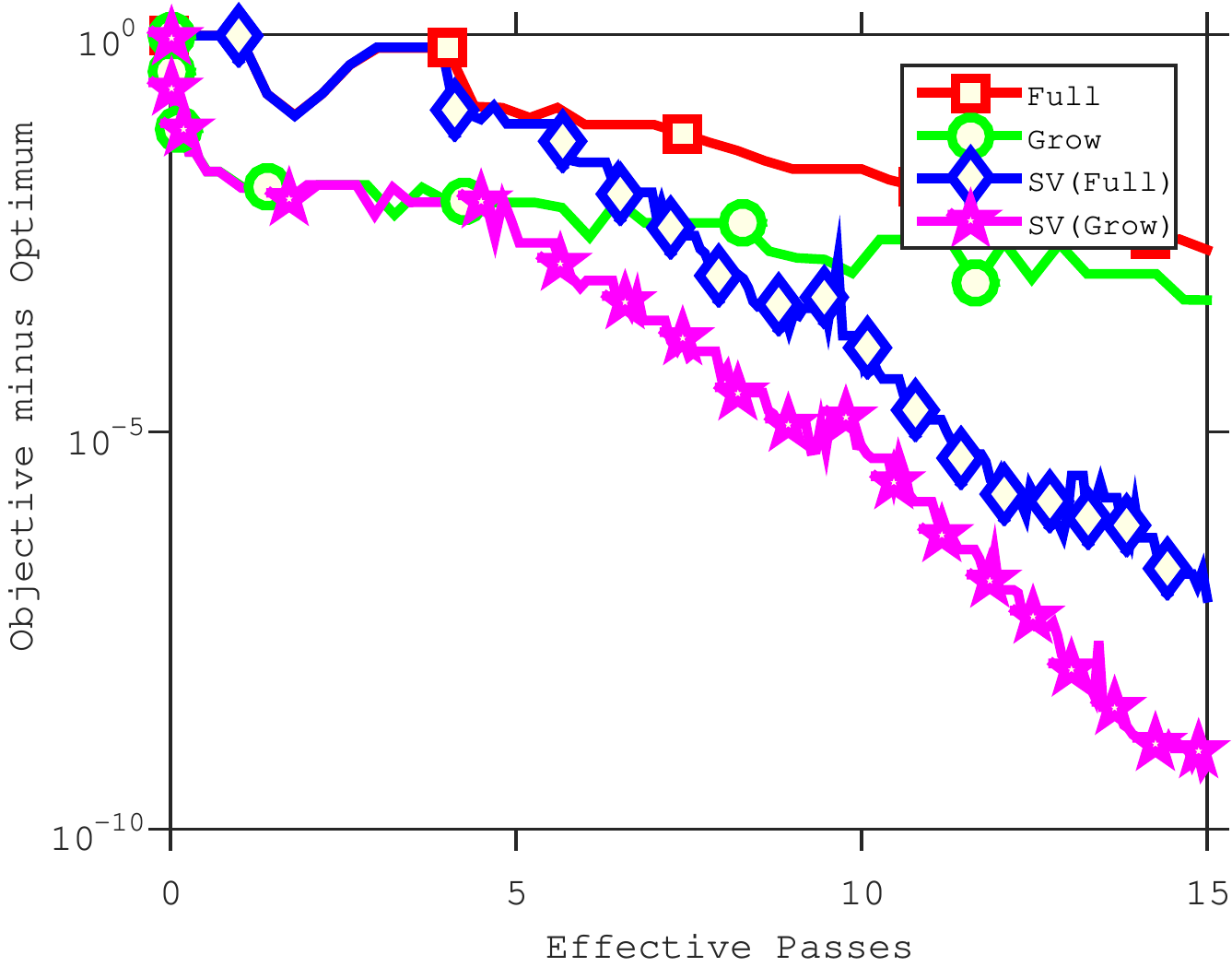}
\includegraphics[width=.4\textwidth]{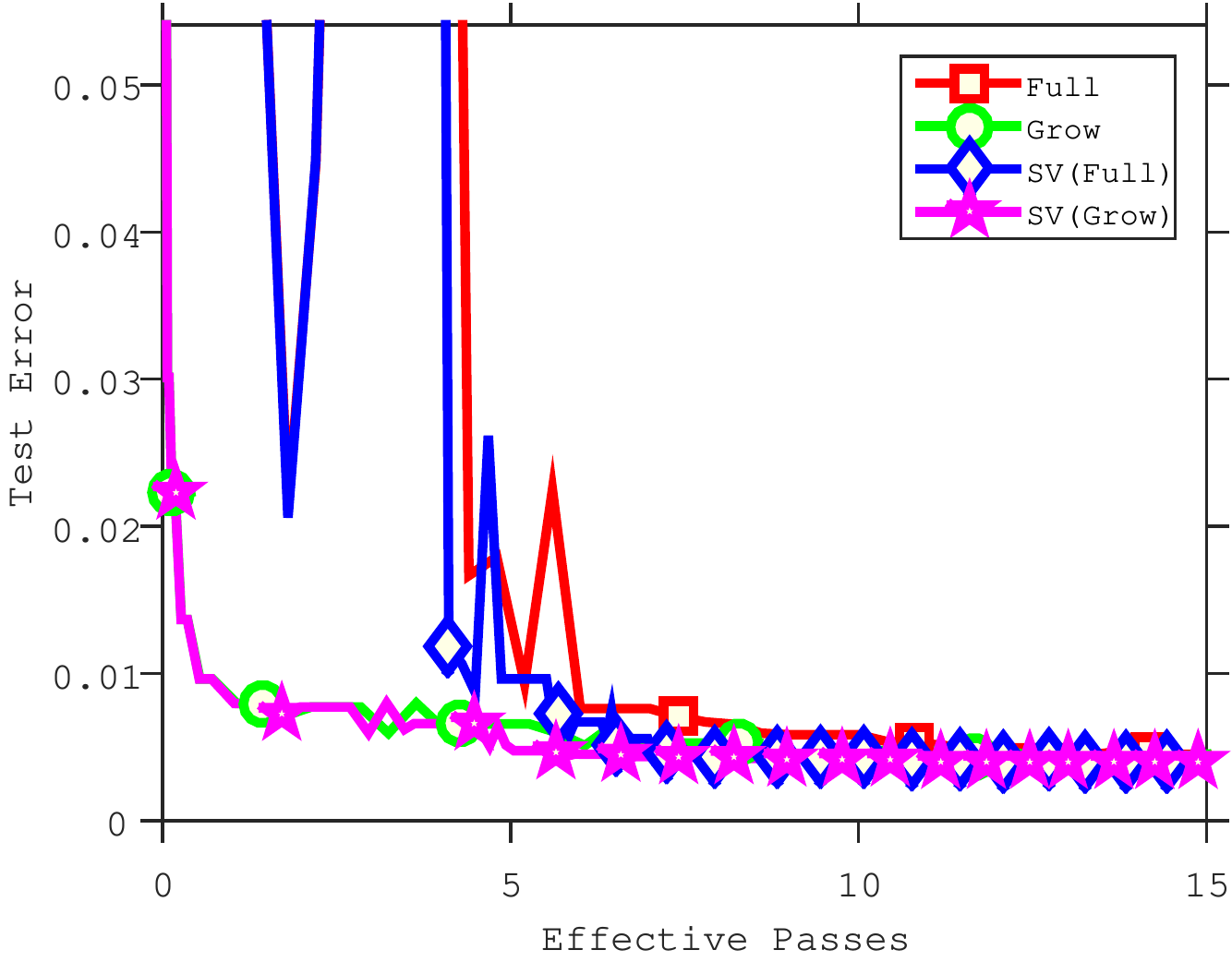}

\caption{Comparison of training objective (left) and test error (right) on the \emph{spam} dataset for the logistic regression (top) and the HSVM (bottom) losses under different batch strategies for choosing $\mu^s$ (\emph{Full}, \emph{Grow}, and \emph{Mixed}) and whether we attempt to identify support vectors (\emph{SV}).}
\label{fig:1}
\end{figure*}

In this section, we present experimental results that evaluate our proposed variations on the SVRG method.
We focus on logistic regression classification: given a set of training data $(a_1,b_1)\dots(a_n,b_n)$ where $a_i \in \mathbb R^d$ and $b_i \in \{-1,+1\}$, the goal is to find the $x \in \mathbb R^d$ solving
\[
\argmin{x\in \mathbb R^d}\; \frac \lambda 2 \|x\|^2 +  \frac 1 n \sum _{i=1}^n \log(1+\exp(-b_ia_i^Tx)),
\] 
We consider the datasets used by~\cite{roux2012stochastic}, whose properties are listed in the supplementary material. As in their work we add a bias variable, normalize dense features, and set the regularization parameter  $\lambda$ to $1/n$. We used a step-size of $\alpha=1/L$ and we used $m=|\batch^s|$ which gave good performance across methods and datasets.
In our first experiment, we compared three variants of SVRG: the original strategy that uses all $n$ examples to form $\mu^s$ (\emph{Full}), a growing batch strategy that sets $|\batch^s|=2^s$ (\emph{Grow}), and the mixed SG/SVRG described by~Algorithm~\ref{alg:SVRSGD} under this same choice (\emph{Mixed}). While a variety of practical batching methods have been proposed in the literature~\cite{friedlander2011hybrid,byrd2012sample,van2012adaptive}, we did not find that any of these strategies consistently outperformed the doubling used by the simple \emph{Grow} strategy.
Our second experiment focused on the $\ell_2$-regularized HSVM on the same datasets, and we compared the original SVRG algorithm with variants that try to identify the support vectors (\emph{SV}).
% In Appendix~F, we report a third experiment based on the new non-uniform sampling scheme from Section~\ref{sec:mini}.

We plot the experimental results for one run of the algorithms on one dataset in Figure~\ref{fig:1}, while Appendix~F reports results on the other $8$ datasets over $10$ different runs. In our results, the growing batch strategy (\emph{Grow}) always had better test error performance than using the full batch, while for large datasets it also performed substantially better in terms of the training objective. In contrast, the \emph{Mixed} strategy sometimes helped performance and sometimes hurt performance. Utilizing support vectors often improved the training objective, often by large margins, but its effect on the test objective was smaller. 

%Note that the first three data sets have a number of samples $n$ much larger than the dimensionality $p$ while the remaining three data sets have more variables than samples.
%All data sets have an added bias term and they are all normalized to a mean of zero and variance of one. To make the results independents of implementation of the algorithms, we evaluates objective functions as a function of effective passes through the data which equal to number of times we evaluates a gradient over data set size $n$. 

\section{Discussion}

As SVRG is the only memory-free method among the new stochastic linearly-convergent methods, it represents the natural method to use for a huge variety of machine learning problems.
In this work we show that the convergence rate of the SVRG algorithm can be preserved even under an inexact approximation to the full gradient. We also showed that using mini-batches to approximate $\mu^s$ gives a natural way to do this, explored the use of support vectors to further reduce the number of gradient evaluations, gave an analysis of the regularized SVRG update, and considered several new mini-batch strategies. 
Our theoretical and experimental results indicate that many of these simple modifications should be considered in any practical implementation of SVRG.

%We have considered the basic SVRG algorithm, and showed that the algorithm's convergence rate can be preserved even with an inexact approximation to the full gradient. We subsequently considered using sub-sampling of the training data to yield a cheap approximation, and in this context showed that using a mixture of SG and SVRG steps leads to much stronger practical performance. Although we have only evaluated this variant of SVRG in a simple classification scenario, this approach should lead to large performance improvements in any application where SVRG is the current state-of-the-art.

%Our analysis also allows other sources of error in the calculation of the full gradient. For example, it allows us to analyze the effect of using an approximation to the log-partition function in graphical models~\cite{wainwright2003tree}. Further, we expect that our key message that the cost of SVRG can be significantly reduced is also likely to be true for the various extensions of SVRG. We have explicitly shown that it is true under non-uniform sampling and proximal-gradient variants, but it should also be true for accelerated versions of the method~\cite{nitanda2014stochastic}, for coordinate-wise variants~\cite{konevcny2014semi}, and under weaker assumptions than strong convexity~\cite{gong2014linear}.

\section*{Acknowledgements}

We would like to thank the reviewers for their helpful comments. This research was supported by the Natural Sciences and Engineering Research Council of Canada (RGPIN 312176-2010, RGPIN 311661-08, RGPIN-06068-2015). Jakub Kone{\v{c}}n{\'y} is supported by a Google European Doctoral Fellowship.

\pagebreak

\appendix

\section{Convergence Rate of SVRG with Error}

We first give the proof of Proposition~1, which gives a convergence rate for SVRG with an error and uniform sampling. We then turn to the case of non-uniform sampling.
%\begin{thm}
%\label{thm:thm1}
%If we have $\tilde{\mu}_s = f'(\tilde{x}_s) + \tilde{e}_s$ and set the step-size $\eta$ and number of inner iterations $m$ so that $\rho < 1$, we have
%\[
%\mathbb{E}[f(x_{s+1})-f(x^{*})] \leq \rho E[f(\tilde{x}_{s})-f(x^{*})] +C,
%\]
%where 
%\[
% C = \frac 1 {1-2\eta L}(Z \mathbb E \|\tilde{e}_{s}\|+\eta \mathbb E\|\tilde{e}_{s}\|^{2}).
%\]
%\end{thm}
%We give the proof of this proposition below. This result implies that we can maintain the convergence rate of SVRG, provided that the errors $\tilde{e}_s$ decrease at the appropriate rate. In particular, this result implies that we obtain the same convergence rate as the exact SVRG algorithm provided that we have $\max\{\mathbb{E}\norm{\tilde{e}_s},\mathbb{E}\norm{\tilde{e}_s}^2\} = \gamma\tilde{\rho}^k$ for some $\gamma \geq 0$ and some $\tilde{\rho} < \rho$. Further, we still obtain a linear convergence rate as long as $\tilde{e}_s$ converges to zero with a linear convergence rate.

\subsection{Proof of Proposition 1}

We follow a similar argument to Johnson \& Zhang~\cite{johnson2013accelerating}, but propagating the error $e^s$ through the analysis. We begin by deriving a simple bound on the variance of the sub-optimality of the gradients.
\begin{lem}
\label{lem:lem1}
 For any $x$,
\[
\frac{1}{n}\sum_{i=1}^{n}\|f'_{i}(x)-f'_{i}(x^{*})\|^2\leq2L[f(x)-f(x^{*})].
\]
\end{lem}
\begin{proof}
Because each $f_i'$ is $L$-Lipschitz continuous, we have~\cite[Theorem~2.1.5]{nesterov2004introductory}
\[
f_{i}(x)\geq f_{i}(y)+\left\langle f'_{i}(x),x-y\right\rangle+\frac{1}{2L}\|f'_{i}(x)-f'_{i}(y)\|^2. 
\]
Setting $y=x^*$ and summing this inequality times $(1/n)$ over all $i$ we obtain the result.
\end{proof} 
\noindent In this section we'll use $\tilde{x}$ to denote $x^s$, $e$ to denote $e^s$, and we'll use $\nu_t$ to denote the search direction at iteration $t$,
\[
\nu_t = f'_{i_{t}}(x_{t-1})-f'_{i_{t}}(\tilde{x})+\tilde{\mu}+e.
\]
Note that $\mathbb{E}[\nu_t] = f'(x_{t-1}) + e$ and the next lemma bounds the variance of this value.
\begin{lem}
\label{lem:lem2}
In each iteration $t$ of the inner loop, 
\[ \mathbb E\|\nu _{t}\|^{2}
\leq4L[f(x_{t-1})-f(x^{*})]+4L[f(\tilde{x})-f(x^{*})]+2\|e\|^2
\]
\end{lem}
\begin{proof}
By using the inequality 
$\|x+y\|^{2}\leq2\|x\|^{2}+2\|y\|^{2}$ and the property 
$$\mathbb{E}[f_{i_t}'(\tilde{x}) - f_{i_t}'(x^*)] = f'(\tilde{x}),$$ we have
\begin{align*}
\mathbb E\|\nu _{t}\|^{2} 
&=\mathbb E\|f'_{i_{t}}(x_{t-1})-f'_{i_{t}}(\tilde{x})+\tilde{\mu}+e\|^2
\\
&\leq2\mathbb E\|f'_{i_{t}}(x_{t-1})-f'_{i_{t}}(x^{*})\|^{2} 
+2\mathbb E\|[f'_{i_{t}}(\tilde{x})-f'_{i_{t}}(x^{*})]-f'(\tilde{x})-e\|^{2}
\\
&=2\mathbb E\|f_{i_{t}}'(x_{t-1})-f_{i_{t}}'(x^{*})\|^{2} + 2\mathbb E\|[f_{i_{t}}'(\tilde{x})-f_{i_{t}}'(x^{*})]-\mathbb E[f_{i_{t}}'(\tilde{x})-f_{i_{t}}'(x^{*})]-e\|^{2}
\\
&=2\mathbb E\|f_{i_{t}}'(x_{t-1})-f_{i_{t}}'(x^{*})\|^{2} + 2\mathbb E\|[f_{i_{t}}'(\tilde{x})-f_{i_{t}}'(x^{*})]-\mathbb E[f_{i_{t}}'(\tilde{x})-f_{i_{t}}'(x^{*})]\|^2 + 2\|e\|^{2} 
\\
& \qquad - 4\mathbb E\left\langle [f_{i_{t}}'(\tilde{x})-f_{i_{t}}'(x^{*})]-\mathbb E[f_{i_{t}}'(\tilde{x})-f_{i_{t}}'(x^{*})],e\right\rangle \\
& = 2\mathbb E\|f_{i_{t}}'(x_{t-1})-f_{i_{t}}'(x^{*})\|^{2} + 2\mathbb E\|[f_{i_{t}}'(\tilde{x})-f_{i_{t}}'(x^{*})]-\mathbb E[f_{i_{t}}'(\tilde{x})-f_{i_{t}}'(x^{*})]\|^2 + 2\|e\|^{2},
\end{align*}
If we now use that $\mathbb{E}[\|X-\mathbb E[X]\|^{2}]\leq\mathbb{E}\|X\|^{2}$ for any random variable $X$, we obtain the result by applying Lemma~\ref{lem:lem1} to bound $\mathbb{E} \|f_{i_{t}}^{'}(x_{t-1})-f_{i_{t}}^{'}(x^{*})\|^{2}$ and $\mathbb{E} \|[f_{i_{t}}^{'}(\tilde{x})-f_{i_{t}}^{'}(x^{*})]\|^2$.
\end{proof}
\noindent The following Lemma gives a bound on the distance to the optimal solution.
\begin{lem}
\label{lem:lem3}
In every iteration $t$ of the inner loop,
\begin{align*}
\mathbb E\|x_ {t}-x^{*}\|^{2} &\leq \|x_{t-1} - x^{*}\|^{2} - 2\eta(1-2\eta L)[f(x_{t-1})-f(x^{*})]
\\
& \qquad + 4L\eta^{2}[f(\tilde{x})-f(x^{*})]+2\eta(Z\|e\|+\eta\|e\|^{2}).
\end{align*}
\end{lem}
\begin{proof}
We expand the expectation and bound $\mathbb E\|\nu _{t}\|^2$ using Lemma~\ref{lem:lem2} to obtain
\begin{align*}
\mathbb E \|&x_{t} - x^{*}\|^{2} 
\\
&=\|x_{t-1}-x^{*}\|^{2} - 2\eta\left\langle x_{t-1}-x^{*},\mathbb E[\nu_{t}]\right\rangle +\eta^{2}\mathbb E\|\nu _{t}\|^2
\\
&= \|x_{t-1}-x^{*}\|^{2} - 2\eta\left\langle x_{t-1} - x^{*},f'(x_{t-1})+e\right\rangle+\eta^{2}\mathbb E\|\nu _{t}\|^2
\\
& = \|x_{t-1}-x^{*}\|^{2} - 2\eta\left\langle x_{t-1}-x^{*},f'(x_{t-1})\right\rangle - 2\eta\left\langle x_{t-1} - x^{*},e \right\rangle +\eta^{2}\mathbb E\|\nu _{t}\|^2
\\
&\leq\|x_{t-1}-x^{*}\|^{2}-2\eta \left\langle x_{t-1} - x^{*},e \right\rangle - 2\eta[f(x_{t-1})-f(x^{*})]+2\eta^{2}\|e\|^{2}
\\
& \qquad + 4L\eta^{2}[f(x_{t-1})-f(x^{*})] + 4L\eta^{2}[f(\tilde{x})-f(x^{*})]
%& \leq\| x_{t-1} - x^{*}\|^{2}\\
%&-2\eta(1-2\eta L)[f(x_{t-1})-f(x^{*})]+2\eta^{2}\|\tilde{e}\|^{2}\\
%&+4L\eta^{2}[f(\tilde{x})-f(x^{*})]+2\eta Z\|\tilde{e}\|.
\end{align*}
The inequality above follows from convexity of $f$. The result follows from applying Cauchy-Schwartz to the linear term in $e$ and  that $\norm{x_{t-1}-x^*} \leq Z$.
\end{proof}

\noindent To prove Proposition~1 from the main paper, we first sum the inequality in Lemma~\ref{lem:lem3} for all $t=1,...,m$ and take the expectation with respect to the choice of $x^s$ to get
\begin{align*}
\mathbb E\|x_{m}-x^{*}\|^{2}&\leq\mathbb E\|x_{0}-x^{*}\|^{2} - 2\eta(1-2L\eta)m\mathbb E[f(x_{t-1})-f(x^{*})]
\\
& \qquad + 4L\eta^{2}m\mathbb E[f(\tilde{x})-f(x^{*})] + 2m\eta(Z \mathbb E \|e\|+\eta \mathbb E\|e\|^{2}).
\end{align*}
Re-arranging, and noting that $x_0 = \tilde{x}_{s-1}$ and $\mathbb{E} \left[ f(x_{t-1}) \right] = \mathbb{E} \left[ f(x^s) \right]$, we have that
\begin{align*}
2\eta (1 - 2L \eta) &m \mathbb{E}[f(x_{s})-f(x^{*})] 
\\ 
&\leq \mathbb E\|\tilde{x}_{s-1}-x^{*}\|^{2} +4L\eta^{2}m\mathbb E[f(\tilde{x}_{s-1})-f(x^{*})]
+ 2m\eta(Z \mathbb E \|e^{s-1}\|+\eta \mathbb E\|e^{s-1}\|^{2})\\
&\leq\frac{2}{\mu}\mathbb E[f(\tilde{x}_{s-1})-f(x^{*})]+4L\eta^{2}m\mathbb E[f(\tilde{x}_{s-1})-f(x^{*})]
 + 2m\eta(Z\mathbb E \|e\|+\eta \mathbb E\|e\|^{2}),
\end{align*}
where the last inequality uses strong-convexity and that $f'(x^*) = 0$.
By dividing both sides by $2\eta(1-2L\eta)m$ (which is positive due to the constraint $\eta \leq 1 / 2L$ implied by $0 < \rho < 1$ and $\eta > 0$), we get
\begin{align*}
&\mathbb E[f(x_{s})-f(x^{*})]\\
&\leq \left(\frac1{m\mu(1-2\eta L)\eta}+ \frac {2L\eta}{1-2\eta L}\right)\mathbb E[f(\tilde{x}_{s-1})-f(x^{*})]+\frac 1 {1-2\eta L}\left(Z \mathbb E \|e^{s-1}\|+\eta \mathbb E\|e^{s-1}\|^{2}\right).
\end{align*}

%{\color{red} Note the constraint $\eta \leq 1/2L$ is not stated anywhere in Theorems!!! Same holds for Proposition 2.}

\subsection {Non-Uniform Sampling}

If we sample $i_t$ proportional to the individual Lipschitz constants $L_i$, then we have the following analogue of Lemma~\ref{lem:lem1}.
%To make the rest of the proof work, we need to bound $\mathbb E \|\frac {L_i}{\bar{L}}[f'_{i}(x)-f'_i(x^*)]\|^2$. 
\begin{lem}
\label{lem:lem4}
 For any $x$,
\[
\mathbb E \left\|\frac {L_i}{\bar{L}}[f'_{i}(x)-f'_i(x^*)]\right\|^2 \leq2\bar{L}[f(x)-f(x^{*})].
\]
\end{lem}

%{\color{red} I wouldn't give proof here, just refer to \cite[Lemma 1]{xiao2014proximal}}

\begin{proof}
Because each $f_i'$ is $L_i$-Lipschitz continuous, we have~\cite[Theorem~2.1.5]{nesterov2004introductory}
\[
f_{i}(x)\geq f_{i}(y)+\left\langle f'_{i}(x),x-y\right\rangle+\frac{1}{2L_i}\|f'_{i}(x)-f'_{i}(y)\|^2. 
\]
Setting $y=x^*$ and summing this inequality times $(1/n)$ over all $i$ we have
\begin{align*}
&\mathbb E \left\|\frac {L_i}{\bar{L}}[f'_{i}(x)-f'_i(x^*)]\right\|^2  = \sum_{i=1}^{n} \frac {L_i} {n\bar{L}} \frac {\bar{L}^2}{L_i^2}\|f'_{i}(x)-f'_{i}(y)\|^2
=\frac{\bar{L}}{n} \sum_{i=1}^{n} \frac {1}{L_i}\|f'_{i}(x)-f'_{i}(y)\|^2\\
&\leq \frac{\bar{L}}{n} \sum_{i=1}^{n} \frac {1}{L_i} 2L_i[f_i(x)-f_i(x^{*})-\left\langle f'_{i}(x),x-x^*\right\rangle]\\
&=2\bar{L}[f(x)-f(x^{*})]
\end{align*}
\end{proof} 
\noindent With this modified lemma, we can derive the convergence rate under this non-uniform sampling scheme by following an identical sequence of steps but where each instance of $L$ is replaced by $\bar{L}$.

\section{Mixed SVRG and SG Method}

We first give the proof of Proposition~2 in the paper, which analyzes a method that mixes SG and SVRG updates using a constant step size. We then consider a variant where the SG and SVRG updates use different step sizes.

\subsection{Proof of Proposition~2}

% C 
%\begin{thm}
%\label{thm:thm2}
%If we have $\mu^s = f'(x^s) + e^s$ and set the step-size $\eta$ and number of inner iterations $m$ so that
%\[
%\rho(\alpha) \equiv\frac1{m\mu(1-2\eta L)\eta}+ \frac {2\alpha L\eta}{1-2\eta L}  < 1,
%\]
%where we have $\alpha= Pr( f_{i_t} \in \batch^s)=\frac {|\batch^s|} n$ and $\beta=1-\alpha$, then for iterates $x^s$ of Algorithm 2 in main body of the paper it holds that
%\[
%\mathbb{E}[f(x^{s+1})-f(x^{*})] \leq \rho(\alpha) \mathbb{E}[f(x^s)-f(x^{*})] +C,
%\]
%where 
%\[
% C = \frac 1 {1-2\eta L}(Z \mathbb E \|e^{s}\|+\eta \mathbb E\|e^{s}\|^{2}) + \frac{\beta\eta}{2(1-2\eta L)}\sigma^2,
%\]
%and $\sigma^2$ uppder bounds $\mathbb E \norm{f_i'(x)}^2$. 
%\end{thm}
%
%\begin{proof}
Recall that the SG update is
\[
 x_t = x_{t-1} - \eta f_{i_t}'(x_{t-1}).
\]
Using this in Lemma~\ref{lem:lem3} and following a similar argument we have
\begin{align*}
\mathbb E\|x_ {t}-x^{*}\|^{2} &\leq \alpha\{\|x_{t-1} - x^{*}\|^{2}-2\eta(1-2\eta L)[f(x_{t-1})-f(x^{*})]
 + 4L\eta^{2}[f(\tilde{x})-f(x^{*})]+2\eta(Z\|e\|+\eta\|e\|^{2})\}
\\
& \qquad + \beta\{ \|x_{t-1} - x^{*}\|^{2} +\eta^2\mathbb E \|f'_{i_t}(x_{t-1})\|^2
- 2\eta \left\langle x_{t-1} - x^{*},\mathbb E [f'_{i_t}(x_{t-1})]\right\rangle\}
\\
& \leq \|x_{t-1} - x^{*}\|^{2}-2\eta(1-2\eta L)[f(x_{t-1})-f(x^{*})]
 + \alpha 4L\eta^{2}[f(\tilde{x})-f(x^{*})]+ \alpha 2\eta(Z\|e\|+\eta\|e\|^{2})
 + \beta \eta^2\sigma^2,
\end{align*}
where the second inequality uses convexity of $f$ and we have defined $\beta = (1-\alpha)$.
%The second inequality is due to convexity of $f(x)$ and knowing that $1-2\eta L \leq 1$, as well as the boundedness of all $f'_{i_t}(x)$ . 
We now  sum up both sides and take the expectation with respect to the history,
\begin{align*}
\mathbb E\|x_{m}-x^{*}\|^{2} &\leq\mathbb E\|x_{0}-x^{*}\|^{2}-2\eta(1-2L\eta)m\mathbb E[f(x_{t-1})-f(x^{*})]
\\
 &+4\alpha L\eta^{2}m\mathbb E[f(\tilde{x})-f(x^{*})]+2m\alpha\eta(Z \mathbb E \|e\|+\eta \mathbb E\|e\|^{2})\\
 & +  m\beta \eta^2 \sigma^2.
\end{align*}
By re-arranging the terms we get
\begin{align*}
2\eta (1-2L\eta)m\mathbb E[f(x_{s})-f(x^{*})] &\leq \frac{2}{\mu}\mathbb E[f(\tilde{x}_{s-1})-f(x^{*})] 
 + 4\alpha L\eta^{2}m\mathbb E[f(\tilde{x}_{s-1})-f(x^{*})]\\
& + 2m\alpha\eta(Z\mathbb E \|e\|+\eta \mathbb E\|e\|^{2}) + m\beta \eta^2 \sigma^2, 
\end{align*}
and by dividing both sides by $ 2\eta (1-2L\eta)m$ we get the result.
%\end{proof}

\subsection {Mixed SVRG and SG with Different Step Sizes}

Consider a variant where we use a step size of $\eta$ in the SVRG update and a step-size $\eta_s$ in the SG update (which will decrease as the iterations proceed). Analyzing the mixed algorithm in this setting gives
\begin{align*}
\mathbb E\|x_ {t}-x^{*}\|^{2} 
&\leq \alpha\{\|x_{t-1} - x^{*}\|^{2}-2\eta(1-2\eta L)[f(x_{t-1})-f(x^{*})]\\
& \qquad + 4L\eta^{2}[f(\tx)-f(x^{*})]+2\eta(Z\|e^s\|+\eta\|e^s\|^{2})\}
\\
& \qquad + \beta\{ \|x_{t-1} - x^{*}\|^{2} +\eta_s^2\mathbb E \|f'_{i_t}(x_{t-1})\|^2
 - 2\eta_s \left\langle x_{t-1} - x^{*},\mathbb E [f'_{i_t}(x_{t-1})]\right\rangle\}
\\
& = \ex{\norm{x_{t-1}-\x}^2} - 2\alpha\eta(1-2\eta L)[f(x_{t-1})-f(x^{*})]
\\
& \qquad + 4\alpha L\eta^{2}[f(\tx)-f(x^{*})] + 2\alpha \eta(Z\|e^s\|+\eta\|e^s\|^{2}) 
\\
& \qquad + \beta \eta_s^2\mathbb E \|f'_{i_t}(x_{t-1})\|^2-2\beta\eta_s \left\langle x_{t-1} - x^{*},f(x_{t-1})]\right\rangle
\\
&\leq \ex{\norm{x_{t-1}-\x}^2} - \{2\alpha\eta(1-2\eta L)+2\beta\eta_s\}[f(x_{t-1})-f(x^{*})] 
\\
& \qquad + 4\alpha L\eta^{2}[f(\tx)-f(x^{*})] + 2\alpha \eta(Z\|e^s\|+\eta\|e^s\|^{2})+ \beta \eta_s^2\sigma^2.
 \end{align*}
As before, we take the expectation for all $t$ and sum up these values, then rearranage and use strong-convexity of $f$ to get
\begin{align*}
2m &\{ \alpha\eta(1-2\eta L)+\beta\eta_s \} [f(x_{s})-f(x^{*})]
\\
&\leq \left\{ \frac 2 \mu + 4m\alpha L\eta^{2} \right\} [f(x^s)-f(x^{*})] + 2m\alpha \eta(Z\|e^s\|+\eta\|e^s\|^{2})+ m\beta \eta_s^2\sigma^2.
\end{align*}
If we now divide both side by $2m(\alpha\eta(1-2\eta L)+\beta\eta_s)$, we get
\begin{align*}
&\ex{f(x_{s})-f(x^{*})} 
\\
&\leq \Big\{\frac 1 {\mu m(\alpha\eta(1-2\eta L)+\beta\eta_s)} +\frac {2\alpha L\eta^{2}}  {\alpha\eta(1-2\eta L)+\beta\eta_s} \Big\}[f(\tx)-f(x^{*})] 
\\
& \qquad + \frac {\alpha \eta}{\alpha\eta(1-2\eta L)+\beta\eta_s}(Z\ex{\|e^s\|}+\eta\ex{\|e^s\|^{2}})
 + \frac{1}{2(\alpha\eta(1-2\eta L)+\beta\eta_s)}\beta \eta_s^2\sigma^2.
 \end{align*}
To improve the dependence on the error $e^s$ and variance $\sigma^2$ compared to the basic SVRG algorithm with error $e^s$ (Proposition~1), we require that the terms depending on these values are smaller,
%Now we compare this result against the result for mixed SVRG when $\eta=\eta_s=\eta$ and considering in the second algorithm $\eta=\eta$, and $\eta_s\leq \eta$. To do so, we try to make the constant part of the second algorithm smaller than the batched SVRG, i.e. we want to pick $\eta_s$ in a way that: 
\begin{align*}
&\frac {\alpha \eta}{\alpha\eta(1-2\eta L)+\beta\eta_s} \left( Z\ex{\|e^s\|}+\eta\ex{\|e^s\|^{2}} \right) \\
&+ \frac {1}{2(\alpha\eta(1-2\eta L)+\beta\eta_s)}\beta \eta_s^2\sigma^2 \leq 
\frac {1}{1-2\eta L} \left( Z\ex{\|e^s\|}+\eta\ex{\|e^s\|^{2}} \right).
\end{align*}
Let $\kappa = (1-2\eta L)$ and $\zeta=Z\ex{\|e^s\|}+\eta\ex{\|e^s\|^{2}}$, this requires
\begin{align*}
&\frac {\alpha \eta}{\alpha\eta\kappa+\beta\eta_s}\zeta +\frac {\beta \eta_s^2}{2(\alpha\eta\kappa+\beta\eta_s)}\sigma^2 \leq \frac {\zeta}{\kappa}.
\end{align*}
Thus, it is sufficient that $\eta_s$ satisfies
\begin{align*}
\eta_s\leq \frac {2\zeta}{\kappa \sigma^2}. 
\end{align*}
Using the relationship between expected error and $S^2$, while noting that $S^2 \leq \sigma^2$ and $\frac{(n-|\batch|)}{n|\batch|} \leq 1$, a step size of the form $\eta_s = O^*(\sqrt{(n-|\batch|)/n|\batch|})$ will improve the dependence on $e^s$ and $\sigma^2$ compared to the dependence on $e^s$ in the pure SVRG method.

\section{Proximal and Regularized SVRG}

In this section we consider objectives of the form
\[
f(x) = h(x) + g(x),
\]
where $g(x) = \frac{1}{n}\sum_{i=1}^n g_i(x)$. We first consider the case where $h$ is non-smooth and consider a proximal-gradient variant of SVRG where there is an error in the calculation of the gradient (Algorithm~\ref{alg:ProxSVRGB}). We then consider smooth functions $h$ where we use a modified SVRG iteration,
\[
x_{t+1}=x_t-\eta\left( h'(x_t) + g'_{i_t}(x_t)-g'_{i_t}(x^s)+\mu^s\right),
\]
where $\mu^s = g(x^s)$.

\subsection{Composite Case}

%{\color{red} This is probably redundant. [Did not go through this section.] I would maybe only mention in Section 3.1 that you can analogously also derive the composite case.}

Similar to the work of~\cite{xiao2014proximal}, in this section we assume that $f,g$ and $h$ are $\mu$-, $\mu_g$-, $\mu_h$-strongly convex (respectively).
As before, we assume each $g_i$ is convex and has an $L$-Lipschitz continuous gradient, but $h$ can potentially be non-smooth.
The algorithm we propose here extends the algorithm of~\cite{xiao2014proximal}, but adding an error term. In the algorithm we use  the proximal operator which is defined by
\[
\prox_{h}(y)=\argmin{x\in \mathbb R^p}{\{ \frac 1 2 \|x-y\|^2+h(x)\}}
\] 
Below, we give a convergence rate for this algorithm with an error $e^s$.

\begin{algorithm}[tb]
 \caption{Batching Prox SVRG}
   \label{alg:ProxSVRGB}
 \begin{algorithmic}
 \STATE {\bfseries Input:} update frequency {\it m} and learning rate $\eta$ and sample size increasing rate $\alpha$
  \STATE Initialize $\tilde{x}$
 \FOR{$s=1,2,3,\dots$} 
	\STATE Choose batch size $|\batch|$
\STATE $\batch$ = randomly choose $|\batch|$ elements of $\{1,2,\dots,n\}$.
\STATE $\tilde{\mu}_=\frac{1}{|\batch|}\sum_{i\in\batch}g_{i}^{'}(\tilde{x})$\\
\STATE $x_{0}$=$\tilde{x}$\\
  \FOR{$t=1,2,\dots,m$} 
           \STATE Randomly pick $i_{t} \in {1,\dots, n}$\\
           \STATE $\nu_t=g_{i_{t}}^{'}(x_{t-1})-g_{i_{t}}^{'}(\tilde{x})+\tilde{\mu}$
           \STATE $x_{t}=prox_{\eta h}(x_{t-1}-\eta\nu_t)\hfill(*)$\\
  \ENDFOR
 \STATE set $\tilde{x}=\frac 1 m\sum_{t=1}^m x_{t}$\\
\ENDFOR
\end{algorithmic}
\end{algorithm}
\setcounter{thm}{4}
\begin{thm}
\label{thm:thms1}
If we have $\tilde{\mu}_s = g'(x^s) + e^s$ and set the step-size $\eta$ and number of inner iterations $m$ so that
\[
\rho \equiv\frac1{m\mu(1-4\eta L)\eta}+ \frac {4 L\eta(m+1)}{(1-4\eta L)m}  < 1,
\]
then Algorithm~\ref{alg:ProxSVRGB} has
\[
\mathbb{E}[f(x_{s+1})-f(x^{*})] \leq \rho E[f(\tilde{x}_{s})-f(x^{*})] +\frac 1 {1-4\eta L}(Z \mathbb E \|e^{s}\|+\eta \mathbb E\|e^{s}\|^{2},
\]
where$\|x_t-x^*\|<Z$
\end{thm}

\noindent To prove Proposition~\ref{thm:thms1}, we use Lemma 1,2 and 3 from~\cite{xiao2014proximal}, which are unchanged when we allow an error. Below we modify their Corollary 3 and then the proof of their main theorem.

\begin{lem}
\label{lem:cor1}
Consider $\nu_t=g_{i_{t}}^{'}(x_{t-1})-g_{i_{t}}^{'}(\tilde{x})+g'(\tilde{x})+e$. Then, 
\[
\mathbb E\|\nu_t-g'(x_{t-1})\|^2 \leq \|e\|^2+4L[f(x_{t-1})-f(x^*)+f(\tilde{x})-f(x^*)]
\]
 \end{lem}
 
\begin{proof}
\begin{align*}
&\mathbb E\|\nu_t-g'(x_{t-1})\|^2\\
&=\mathbb E\|g_{i_{t}}^{'}(x_{t-1})-g_{i_{t}}^{'}(\tilde{x})+g'(\tilde{x})+e-g'(x_{t-1})\|^2\\
&=\|e\|^2+E\|g_{i_{t}}^{'}(x_{t-1})-g_{i_{t}}^{'}(\tilde{x})+g'(\tilde{x})-g'(x_{t-1})\|^2\\
&\leq \|e\|^2 + E\|g_{i_{t}}^{'}(x_{t-1})-g_{i_{t}}^{'}(\tilde{x})\|^2\\
&\leq \|e\|^2 + 2E\|g_{i_{t}}^{'}(x_{t-1})-g_{i_{t}}^{'}(x^*)\|^2+2E\|g_{i}^{'}(\tilde{x})-g_{i_{t}}^{'}(x^*)\|^2\\
\end{align*}
Using  Lemma 1 from~\cite{xiao2014proximal} and bounding the two expectations gives the result. 
\end{proof}
\noindent Now we turn to proving Proposition~\ref{thm:thms1}. 
\begin{proof}
Following the proof of Theroem~1 in~\cite{xiao2014proximal}, we have
\[
\|x_{t}-x^*\|^2 \leq \|x_{t-1}-x^*\|^2-2\eta[f(x_{t})-f(x^*)]-2\eta \left\langle \Delta_t,x_{t}-x^* \right\rangle
\]
where $\Delta_t=\nu_t - g'(x_{t-1})$ and $\mathbb E [\Delta_t]=e$. Now to bound $\left\langle \Delta_t,x_{t}-x^* \right\rangle$, we define 
\[
\bar{x_t}=prox_{h}(x_{t-1}-\eta g'(x_{t-1})),
\]
and subsequently that
\[
-2\eta \left\langle \Delta_t,x_{t}-x^* \right\rangle \leq 2\eta^2 \|\Delta_t\|^2-2\eta\left\langle \Delta_t,\bar{x}_{t}-x^* \right\rangle
\]
Combining with the two previous inequalities we get
\begin{align*}
\|x_{t}-x^*\|^2 &\leq \|x_{t-1}-x^*\|^2-2\eta[f(x_{t})-f(x^*)]+2\eta^2 \|\Delta_t\|^2-2\eta\left\langle \Delta_t,\bar{x}_{t}-x^* \right\rangle.
\end{align*}
If we take the expectation with respect to $i_t$ we have
\begin{align*}
\mathbb E\|x_{t}-x^*\|^2 &\leq \|x_{t-1}-x^*\|^2-2\eta\mathbb E[f(x_{t})-f(x^*)]+2\eta^2 \mathbb E\|\Delta_t\|^2-2\eta\left\langle \mathbb E\Delta_t,\bar{x}_{t}-x^* \right\rangle.
\end{align*}
Now by using the Lemma~\ref{lem:cor1} and $\|\bar{x}_{t}-x^*\| < Z$ we have
\begin{align*}
& \mathbb E\|x_{t}-x^*\|^2\\ &\leq \|x_{t-1}-x^*\|^2-2\eta\mathbb E[f(x_{t})-f(x^*)]
+8\eta^2L[f(x_{t-1})-f(x^*)+f(\tilde{x})-f(x^*)]
+2\eta^2\|e\|^2+2\eta\|e\|Z.
\end{align*}
The rest of the proof follows the argument of~\cite{xiao2014proximal}, and is simlar to the previous proofs in this appendix. We take the expectation and sum up values, using convexity to give
\begin{align*}
&2\eta(1-4L\eta)m[\mathbb Ef(x^s)-f(x^*)] \leq (\frac 2 \mu + 8L\eta^2(m+1)) [f(\tilde{x}_{s-1}-f(x^*)] +2\eta^2\|e\|^2+2\eta\|e\|Z.
\end{align*}
By dividing both sides to $2\eta(1-4L\eta)m$, we get the result. 
\end{proof}

\subsection{Proof of Proposition 3}
%{\color{red} Statement dependent on equation numbering in main paper.}
%\begin{thm}
%\label{thm:er}
%If we can rewrite the equation (1) in the form (5) of main paper and let $G(x) =  \frac 1 n \sum g_i(x)$, $L_m = \max\{ L_g, L_h\}, \rho = \left(\frac{1}{m\mu(1-2\eta L_m)}+\frac{2L_m\eta}{1-2\eta L_m}\right)$ and replace the updating rule of SVRG with 
%\[
%x_{t+1}=x_t-\eta\left( h'(x_t) + g'_{i_t}(x_t)-g'_{i_t}(x^s)+G'(x^s)\right)
%\]
%then we have
%\[
%\mathbb{E}[f(x^{s+1})-f(x^{*})] \leq \rho \mathbb{E}[f(x^s)-f(x^{*})] .
%\]
%\end{thm}
We now turn to the case where $h$ is differentiable, and we use an iteration that incorporates the gradient $h'(x_t)$. Recall that for this result we assume that $g'$ is $L_g$-Lipschitz continuous, $h'$ is $L_h$-Lipschitz continuous, and we defined $L_m = \max\{L_g,L_h\}$. If we let $\nu_t=h'(\xt) + g'_{i_t}(\xt)-g'_{i_t}(\tx)+\tm$, then note that we have $\ex{\nu_t}=f'(\xt)$. Now as before we want to bound the expected second moment of $\nu_t$,
\begin{align*}
\ex{\norm{\nu_t}^2}&=\ex{\norm{h'(\xt) + g'_{i_t}(\xt)-g'_{i_t}(\tx)+\tm}^2}  \\
&= \mathbb{E} [ \| h'(\xt) + g'_{i_t}(\xt)-g'_{i_t}(\tx)+\tm + h'(\x) 
 + g'_{i_t}(\x)-h'(\x)-g'_{i_t}(\x)+h'(\tx)-h'(\tx) \|^2 ] 
\\
&\leq 2\norm{ h'(\xt)-h'(\x)}^2+2\ex{\norm{g'_{i_t}(\xt)-g'_{i_t}(\x)}^2} \\
& \qquad + 2 \ex{\norm{-g'_{i_t}(\tx)+\tm + h'(\x)+g'_{i_t}(\x) +h'(\tx)-h'(\tx)}^2}
\\
&=2\norm{ h'(\xt)-h'(\x)}^2+2\ex{\norm{g'_{i_t}(\xt)-g'_{i_t}(\x)}^2}
\\
& \qquad + 2 \mathbb{E} [  \| g'_{i_t}(\tx)-\tm - h'(\x)-g'_{i_t}(\x) -h'(\tx)
+ h'(\tx)+h'(\x)+g'(\x) \|^2 ]
\\
&=2\norm{ h'(\xt)-h'(\x)}^2+2\ex{\norm{g'_{i_t}(\xt)-g'_{i_t}(\x)}^2}
 + 2 \ex{\norm{g'_{i_t}(\tx)-\tm -g'_{i_t}(\x)+g'(\x)}^2}
\\
&\leq 2\norm{ h'(\xt)-h'(\x)}^2+2\ex{\norm{g'_{i_t}(\xt)-g'_{i_t}(\x)}^2} 
 + 2 \ex{\norm{g'_{i_t}(\tx) -g'_{i_t}(\x)}^2}&
\end{align*}
Now using that $\norm{ h'(\tx)-h'(\x)}^2 \geq 0$ and $ \norm{f(x)-f(y)}^2 \leq 2L[f(x)-f(y)-\left<f'(y),x-y\right>]$,
\begin{align*}
\ex{\norm{\nu_t}^2}
&\leq 2\norm{ h'(\xt)-h'(\x)}^2+2\ex{\norm{g'_{i_t}(\xt)-g'_{i_t}(\x)}^2} 
\\
& \qquad + 2 \ex{\norm{g'_{i_t}(\tx) -g'_{i_t}(\x)}^2} + 2\norm{h'(\tx)-h'(\x)}^2
\\
&\leq 4L_h[h(\xt)-h(\x)-\left<h'(\x),\xt-\x \right>]
+ 4L_g[g(\xt)-g(\x)-\left<g'(\x),\xt-\x \right>]
\\
& \qquad + 4L_h[h(\tx)-h(\x)-\left<h'(\x),\tx-\x \right>]
 + 4L_g[g(\tx)-g(\x)-\left<g'(\x),\tx-\x \right>]
\\
&\leq 4L_m[f(\xt)-f(\x)] + 4L_m[f(\tx)-f(\x)]
\end{align*}
From this point, we follow the standard SVRG argument to obtain
\[
\ex{f(x_{s+1}-f(\x)}\leq \left(\frac{1}{m\mu(1-2\eta L_m)}+\frac{2L_m\eta}{1-2\eta L_m}\right)[f(x_{s}-f(\x)].
\]

\section{Mini-Batch}

We first give an analysis of SVRG where mini-batches are selected by sampling propotional to the Lipschitz constants of the gradients. We then consider the mixed deterministic/random sampling scheme described in the main paper.

\subsection{SVRG with Mini-batch}
%In SVRG with Mini-batching, in outer loop we set: 
%$$ \mb = \frac1n \sum_i f'(\tx). $$

Here we consider using a `mini-batch' of examples in the \emph{inner} SVRG loop. We use $M$ to denote the batch size, and we assume that the elements of the mini-batch are sampled with a probability of $p_i = L_i/n\la$. This gives a search direction and inner iteration of:
 \begin{align*}
 &\nu_t=\mb + \frac 1 M \left[ \sum_{i \in M } \frac 1 {np_i} \left(f'_i(\xt)-f'_i(\tx)\right)\right],\\
 &x_{t+1} = \xt - \eta \nu_t.
 \end{align*}
 Observe that $\ex{\nu_t}=f'(\xt)$, and since each $f_i$ is $L_i$-smooth we still have that
 \[
 \norm{f'_i(x)-f'_i(y)}^2 \leq L_i \left ( f_i(x) - f_i(y) - \left<f'_i(y), x-y\right> \right ),.
 \]
It follows from the definition of $p_i$ that
\begin{align*}
\ex{ \left\| \frac 1 {np_i} f'_i(x)-f'_i(y) \right\|^2} &= \frac 1 n \sum_i \frac 1 {np_i} \norm{f'_i(x)-f'_i(y)}^2 
\\
&\leq 2 \la \left ( f(x) - f(y) -\left <f'(y), x-y\right> \right ),
\end{align*}
which we use to bound $\ex{\norm{\nu_t}^2}$ as before,
\begin{align*}
\ex{\norm{\nu_t}^2}
&=\ex{ \left\| \frac 1 M \sum_i \left( \frac 1 {np_i} (f'_i(\xt)-f'_i(\tx)+\mb\right) \right\|^2}
\\
&= \ex{ \left\| \frac 1 M \sum_i \left( \frac 1 {np_i} (f'_i(\xt)-f'_i(\x)+f'_i(\x)-f'_i(\tx)+\mb\right) \right\|^2}
\\
&\leq \frac 2 {M^2} \sum_i \ex{ \left\| \left( \frac 1 {np_i} (f'_i(\xt)-f'_i(\x))\right) \right\|^2}
\\
& \qquad + \frac 2 {M^2} \sum_i \ex{ \left\| \left( \frac 1 {np_i} (f'_i(\tx)-f'_i(\x))-\mb\right) \right\|^2} 
\\
&\leq  \frac 2 {M^2} \sum_i \ex{ \left\| \left( \frac 1 {np_i} (f'_i(\xt)-f'_i(\x))\right) \right\|^2}
\\
& \qquad + \frac 2 {M^2} \sum_i \ex{ \left\| \left( \frac 1 {np_i} (f'_i(\tx)-f'_i(\x))\right) \right\|^2}
\\
&\leq \frac {4\la}{M}\left[ f(\xt)-f(\x) \right] + \frac {4\la}{M}\left[ f(\tx)-f(\x) \right]
\end{align*}
It subsequently follows that
\[
\ex{f(x^{s+1})-f(\x)} \leq \left( \frac {M}{m\mu(M-2\eta\la)\eta} + \frac{2\la\eta}{M-2\eta\la} \right)\ex{f(x^{s})-f(\x)}
\]

\subsection{Proof of Proposition 4}

We now consider the case where we have $g(x)=(1/n)\sum_{i \notin [\batch_f]} f_i(x)$ and $h(x)=(1/n)\sum_{i\in[\batch_f]} f_i(x)$ for some batch $\batch_f$. We assume that
%, and we replace the SVRG update with
%\[
%x_{t+1}=x_t-\eta\left( h'(x_t) + \frac{1}{M_r} \sum_{f_i \in \mathcal B_r } \frac 1 {np_i} (f'_{i_t}(x_t)-f'_{i_t}(x^s))+g'(x^s)\right).
%\]
% then the convergence rate is 
%\[
%\mathbb E[{f(x_{s+1})-f(x^*)}] \leq \left( \frac {1}{m\mu(1-2\eta\kappa)\eta} + \frac{2\zeta\eta}{1-2\eta\kappa} \right) \mathbb E[{F(x^s)-f(x^*)}].
%\]
%where $\zeta=\frac{(n-M_f)\bar{L}_r}{(M-M_f)n} $ and $\kappa = \max\{\frac {L_1}{n},\zeta\}$. 
%\end{thm}
we sample $M_r$ elements of $g$ with probability of $p_i=\frac {L_i}{(n-M_f)\lr}$ and that we use:
 \begin{align*}
 \nu_t& = g'(x^s) + h'(x_t)+\frac{1}{M_r} \left[ \sum_{i \in M_r } \frac 1 {np_i} \left(f'_i(\xt)-f'_i(\tx)\right)\right],\\
 & =\mb + h'(x_t)+\frac{1}{M_r} \left[ \sum_{i \in M_r } \frac 1 {np_i} \left(f'_i(\xt)-f'_i(\tx)\right)\right]-h'(\tx),\\
 &x_{t+1} = \xt - \eta \nu_t,
 \end{align*}
 where as usual $\mu^s = \frac{1}{n}\sum_{i=1}^nf_i'(x^s) = g'(x^s) + h'(x^s)$.
Note that $\ex{\nu_t}=f'(\xt)$. We first bound $\ex{\norm{\nu_t}^2}$,
\begin{align*}
&\ex{ \|\nu_t\|^2}=\ex{ \left\| \mb + h'(\xt)+1/M_r \left[ \sum_{i \in M_r } \frac 1 {np_i} \left(f'_i(\xt)-f'_i(\tx)\right)\right]-h'(\tx) \right\|^2}
\\
&=\mathbb E[\norm{\mb + h'(\xt)-h'(\x)+1/M_r \left[ \sum_{i \in M_r } \frac 1 {np_i} \left(f'_i(\xt)-f'_i(\x)\right)\right]
\\
&-1/M_r \left[ \sum_{i \in M_r } \frac 1 {np_i} \left(f'_i(\tx)-f'_i(\x)\right)\right]-h'(\tx)+h'(\x)}^2]
\\
&\leq  2/n^2 \underbrace{\sum_{i \in \batch_f}\norm{f'_i(\xt)-f'_i(\x)}^2 }_{\text{Fixed part }}+ \underbrace{ 2/M_r^2 \sum_{i \in \batch_r}\ex{ \left\| \frac 1 {np_i}(f'_i(\xt)-f'_i(\x)) \right\|^2} }_{\text{Random part}}
\\
&+ 2 \ex{ \left\| 1/M_r \left[ \sum_{i \in M_r } \frac 1 {np_i} \left(f'_i(\tx)-f'_i(\x)\right)\right]+h'(\tx)-h'(\x)-\mb \right\|^2},
\end{align*}
where the inequality uses $\|a+b\|^2 \leq 2\|a\|^2+2\|b\|^2$. Now we bound each of the above terms separately,
\begin{align*}
2/n^2 \sum_{i \in \batch_f}\norm{f'_i(\xt)-f'_i(\x)}^2 & \leq 2/n^2 \sum_{i \in \batch_f} 2L_i (f_i(\xt) - f_i(\x)-\left<f'_i(\x),\xt-\x \right>) \\
& \leq 4L_1/n\left(h(\xt)-h(\x)-\left<h'(\x),\xt-\x\right>\right),
\end{align*}

\begin{align*}
2/M_r^2 &\sum_{i \in \batch_r}\ex{\norm{ \frac 1 {np_i}(f'_i(\xt)-f'_i(\x))}^2} 
\\
&\leq 2/M_r^2 \sum_{i \in \batch_r} 1/n^2 \sum_{j \notin \batch_f} 1/p_i \norm{f'_i(\xt)-f'_i(\x)}^2\\
 &\leq 2/M_r^2 \sum_{i \in \batch_r} 1/n^2 \sum_{j \notin \batch_f} (n-M_f)\lr \left( f_i(\xt) - f_i(\x)-\left<f'_i(\x),\xt-\x \right>\right)\\
 & = \frac {4(n-M_f)\lr}{nM_r}{\left(g(\xt)-g(\x)-\left<g'(\x),\xt-\x \right>\right)}.
\end{align*}
Finally for the last term we have,
\begin{align*}
&2 \ex{ \left\| \frac1{M_r} \left[ \sum_{i \in M_r } \frac 1 {np_i} \left(f'_i(\tx)-f'_i(\x)\right)\right]+\underbrace{h'(\tx)-h'(\x)-\mb}_{=g'(\x)-g'(\tx)} \right\|^2}
\\
& \leq 2 \ex{ \left\| \frac1{M_r} \sum_{i \in M_r } \frac 1 {np_i} \left(f'_i(\tx)-f'_i(\x)\right) \right\|^2} \\
& \leq \frac {4(n-M_f)\lr}{nM_r}\left(g(\tx)-g(\x)-\left<g'(\x),\tx-\x \right>\right)
\end{align*}
where the first inequality uses  variance inequality ($\mathbb E \|X-\mathbb E X\|^2 \leq \mathbb E\|X\|^2 $) and the second one comes from Lemma~\ref{lem:lem1}. Since $h$ is convex we can add $\frac {4(n-M_f)\lr}{nM_r}(h(\tx)-h(\x)-\left<h'(\x),\tx-\x \right>)$ to the right side of the above term, giving
\begin{align*}
2 &\ex{ \left\| \frac1{M_r} \left[ \sum_{i \in M } \frac 1 {np_i} \left(f'_i(\tx)-f'_i(\x)\right) \right] + h'(\tx)-h'(\x)-\mb \right\|^2} 
\\
&\leq \frac {4(n-M_f)\lr}{nM_r}\left(f(\tx)-f(\x)\right).
\end{align*}
Now following the proof technique we used several times, we can show that: 
\[
\ex{f(x^{s+1})-f(\x)} \leq \left( \frac {1}{m\mu(1-2\eta\kappa)\eta} + \frac{2\zeta\eta}{1-2\eta\kappa} \right)\ex{f(x^{s})-f(\x)}
\]
where $\zeta=\frac{(n-M_f)\bar{L}_r}{(M-M_f)n} $ and $\kappa = \max\{\frac {L_1}{n},\zeta\}$.

\section{Learning efficiency}

%{\color{red} I would consider moving this to the top of Appendix, as it can grab attention when someone looks just very briefly.}

In this section we closely follow Bottou and Bousquet~\cite{bottou-bousquet-2011, bottou} to discuss the performance of SVRG, and other linearly-convergent stochastic methods, as learning algorithms.
In the typical supervised learning setting, we are giving $n$ independently drawn input-output pairs $(x_i,y_i)$ from some distribution $P(x,y)$ and we seek to minimize the empirical risk,
\[
E_n(f) = \frac{1}{n} \sum_{i = 1}^n \ell(f(x_i), y_i) = \Exp_n [\ell(f(x), y)],
\]
where $\ell$ is our loss function.
However, in machine learning this is typically just a surrogate for the objective we are ultimately interested in. In particular, we typically want to minimize the \emph{expected} risk,
\[
E(f) = \int \ell (f(x), y) dP(x, y) = \Exp [\ell(f(x), y)],
\]
which tells us how well we do on test data from the same distribution.
We use $f^*$ to denote the minimizer of the expected risk,
\[
f^*(x) = \argmin{\hat{y}} \Exp [\ell(\hat{y}, y) \,|\, x],
\]
which is the best that a learner can hope to achieve.

Consider a family $\mathcal{F}$ of possible functions that we use to predict $y_i$ from $x_i$. We write the minimizer of the expected risk over this restricted set as $f_\mathcal{F}^*$,
\[
f_\mathcal{F}^* = \argmin{f\in\mathcal{F}}E(f),
\]
 while we denote the empirical risk minimizer within this family as  $f_n$, 
 \[
 f_n = \argmin{f\in\mathcal{F}} E_n(f).
 \]
But, since we are applying a numerical optimizer we only assume that we find a $\rho$-optimal minimizer of the empirical risk $\tilde{f}_n$,
\[
E_n(\tilde{f}_n) < E_n(f_n) + \rho,
\]
In this setting, Bottou \& Bousquet consider writing the sub-optimality of the approximate empirical risk minimizer $\tilde{f}_n$ compared to the minimizer of the expected risk $f^*$ as
\begin{align*}
\mathcal{E} &= \Exp [E(\tilde{f}_n) - E(f^*)] \\
&= \underbrace{ \Exp [E(f_\mathcal{F}^*) - E(f^*)]}_{\mathcal{E}_{\text{app}}} + \underbrace{ \Exp [E(f_n) - E(f_\mathcal{F}^*)]}_{\mathcal{E}_{\text{est}}} + \underbrace{ \Exp [E(\tilde{f}_n) - E(f_n)]}_{\mathcal{E}_{\text{opt}}}, \numberthis
\end{align*}
where the expectation is taken with respect to the output of the algorithm and with respect to the training examples that we sample. This decomposition shows how three intuitive terms affect the sub-optimality:
\begin{enumerate}
\item $\mathcal{E}_{\text{app}}$ is the \emph{approximation error}: it measures the effect of restricting attention to the function class $\mathcal{F}$.
\item $\mathcal{E}_{\text{est}}$ is the \emph{estimation error}: it measures the effect of only using a finite number of samples.
\item $\mathcal{E}_{\text{opt}}$ is the \emph{optimization error}: it measures the effect of inexactly solving the optimization problem.
\end{enumerate}
While choosing the family of possible approximating functions $\mathcal{F}$ is an interesting and important issue, for the remainder of this section we will assume that we are given a fixed family. In particular, Bottou \& Bousquet's assumption is that $\mathcal{F}$ is linearly-parameterized by a vector $w \in \mathbb{R}^d$, and that all quantities are bounded ($x_i$, $y_i$, and $w$).
This means that the approximation error $\mathcal{E}_\text{app}$ is fixed so we can only focus on the trade-off between the estimation error $\mathcal{E}_{\text{est}}$ and the optimization error $\mathcal{E}_\text{opt}$. 

All other sections of this work focus on the case of \emph{finite} datasets where we can afford to do several passes through the data (\emph{small-scale learning problems} in the language of Bottou \& Bousquet). In this setting, $\mathcal{E}_{\text{est}}$ is fixed so all we can do to minimize $\mathcal{E}$ is drive the optimization error $\rho$ as small as possible. In this section we consider the case where we do not have enough time to process all available examples, or we have an infinite number of possible examples (\emph{large-scale learning problems} in the language of Bottou \& Bousquet). In this setting, the time restriction means we need to make a trade-off between the optimization error and the estimation error: should we increase $n$ in order to decrease the estimation error $\mathcal{E}_\text{est}$ or should we revisit examples to try to more quickly decrease the optimization error $\mathcal{E}_\text{opt}$ while keeping the estimation error fixed?

Bottou \& Bousquet discuss how under various assumptions we have the variance condition
\[
 \forall f \in \mathcal{F} \quad \Exp \left[ \left( \ell(f(X), Y) - \ell(f_{\mathcal{F}}^*(X), Y) \right)^2 \right] \leq c \left( E(f) - E(f_{\mathcal{F}}^*) \right)^{2 - \frac{1}{\alpha}},
 \]
% and how this implies the bound
 %\[
 % \Exp \left[ \sup_{f \in \mathcal{F}} | E(f) - E_n(f) | \right] \leq O\left(\left( \frac{d}{n} \log \frac{n}{d} \right)^{\alpha}\right).
  %\]
and how this implies the bound
\[
\mathcal{E}  = O\left(\mathcal{E}_\text{app} + \left(\frac d n \log \frac n d\right)^\alpha + \rho\right).
\]
To make the second and third terms comparable, we can take $\rho = \left(\frac d n \log \frac n d\right)^\alpha$. Then to achieve an accuracy of $O(\mathcal{E}_\text{app} + \epsilon)$ it is sufficient to take $n = O\left(\frac{d}{\epsilon^{1/\alpha}}\log(1/\epsilon)\right)$ samples:
\begin{align*}
\mathcal{E} & = O\left(\mathcal{E}_\text{app} + \left(\frac d n \log \frac n d\right)^\alpha + \rho\right)\\
& = O\left(\mathcal{E}_\text{app} + \left(\frac d n \log \frac n d\right)^\alpha + \left(\frac d n \log \frac n d\right)^\alpha\right)\\
& = O\left(\mathcal{E}_\text{app} + \left(\frac d n \log \frac n d\right)^\alpha\right)\\
& = O\left(\mathcal{E}_\text{app} +  \left(\frac {\epsilon^{\frac 1 \alpha}}{\log(\frac 1 \epsilon)}\log\left(\frac {\log(\frac 1 \epsilon)}{\epsilon^{\frac 1 \alpha}}\right)\right)^\alpha\right)\\
& = O\left(\mathcal{E}_\text{app} +  \epsilon\left(\frac{\log(\log(1/\epsilon)) - \frac{1}{\alpha}\log(\epsilon)}{\log(1/\epsilon)}\right)^\alpha\right)\\
& = O(\mathcal{E}_\text{app} +\epsilon).
\end{align*}
%In particular, if we replace $n$ in $\rho$ with $\frac{d}{\epsilon^{1/\alpha}}\log(1/\epsilon)$ then we have $\mathcal{E}  = O(\mathcal{E}_\text{app} + (\epsilon^{\frac 1 \alpha}(\log(\epsilon+\epsilon^{\frac 1 \alpha}\log(1/\epsilon)))^\alpha )$ which indicates the accuracy that we want. 
The results presented in the main paper follow from noting that (i) the iteration cost of SVRG is $O(d)$ and (ii) that the number of iterations for SVRG to reach an accuracy of $\rho$ is $O((n + \kappa)\log(1/\rho))$.%To achieve the accuracy of $\epsilon$ in the second column for SVRG in the main paper's table, we replace $n$ with $\frac{d}{\epsilon^{1/\alpha}}\log(1/\epsilon)$ in $O((n + \kappa)d\log(1/\epsilon))$, considering $d$ as cost of each iteration. 

%To get the $\epsilon < 1$ error:
%\begin{align*}
%&n = \frac{d}{\epsilon^{1/\alpha}}\log(1/\epsilon) \\ 
%&\rho = (\frac {\epsilon^{\frac 1 \alpha}}{\log(\frac 1 \epsilon)}\log(\frac {\log(\frac 1 \epsilon)}{\epsilon^{\frac 1 \alpha}}))^\alpha= ( \epsilon^{\frac 1 \alpha} \frac{\log(\log(\frac 1 \epsilon))-\frac 1 \alpha \log(\epsilon))}{\log(\frac 1 \epsilon)})^\alpha \leq \epsilon
%\end{align*} 

\section{Additional Experimental Results}

\begin{table*}
\centering
\begin{tabular}{|l|r|r|c|}
\hline
Data set & Data Points & Variables & Reference \\
\hline
\emph{quantum} & 50 000 & 78 & \cite{caruana2004kdd}\\
\hline
\emph{protein} & 145 751 & 74 &  \cite{caruana2004kdd}\\
\hline
\emph{sido} & 12 678 & 4 932  & \cite{SIDO}\\
\hline
\emph{rcv1} & 20 242 & 47 236  & \cite{lewis2004rcv1}\\
\hline
\emph{covertype} & 581 012 & 54  & \cite{blake1998uci}\\
\hline
\emph{news} & 19 996 & 1 355 191 & \cite{keerthi2005modified}\\
\hline
\emph{spam} & 92 189 & 823 470 & \cite{cormack2005spam,Carbonetto09}\\
\hline
\emph{rcv1Full} & 697 641 & 47 236  & \cite{lewis2004rcv1}\\
\hline
\emph{alpha} & 500 000 & 500  & Synthetic\\
\hline
\end{tabular}
\caption{Binary data sets used in the experiments.}
\label{table:data}
\end{table*}

We list properties of the dataset considered in the experiments in Table~\ref{table:data}. In Figures 1-4, we plot the performance on the various datasets in terms of both the training objective and test error, showing the maximum/mean/minimum performance across 10 random trials. In these plots, we see a clear advantage for the \emph{Grow} strategy on the largest datasets (bottom row), but less of an advantage or no advantage on the smaller datasets. The advantage of using support vectors seemed less dependent on the data size, as it helped in some small datasets as well as some large datasets, while in some small/large datasets it did not make a big difference.

In Figures 5-6, we give the result of experiments comparing different mini-batch selection strategies. In particular, we consider mini-batch SVRG with a batch size of $16$ and compare the following methods: uniform sampling of the mini-batch (\emph{Uniform}), sampling proportional to the Lipschitz constants (\emph{Lipschitz}), and a third strategy based on Proposition~4 in the main paper (\emph{Lipschitz+}). On each iteration, the \emph{Lipschitz+} strategy constructs the mini-batch using the $100$ examples with the largest Lipschitz constants (the `fixed' set) in addition to $16$ examples sampled according to their Lipschitz constants from among the remaining examples. We assume that the fixed set is computed `for free' by calculating these gradients on a GPU or FPGA. In these experiments, there was often no difference between the various methods because the rows of the data were normalized. For the two Lipschitz sampling strategies, we used a step size of $1/\la$.
In some cases, the new sampling scheme may have given a small improvement, but in general the theoretical advantage of this method was not reflected in our experiments.

In Figure 7-8, we repeat the mini-batch experiment but include two additional method: sampling example $i$ proportional to $f_i(x^s)$ (\emph{Function}) and sampling $i$ proportional to $\norm{f_i'(x^s)}$ (\emph{Gradient}). For these strategies we used a step size of $1/\la$, and on eight of the nine datasets we were surprised that these strategies had similar performance to the Lipschitz sampling strategy (even though they do not have access to the $L_i$). However, both of these strategies had strange behaviour on one of the datasets. On the \emph{covertype} dataset, the \emph{Function} method seemed to diverge in terms of training objective and test error while the \emph{Gradient} seemed to converge to a sub-optimal solution in terms of training objective but achieved close to the optimal test error.

\begin{figure*}
\includegraphics[width=.32\textwidth]{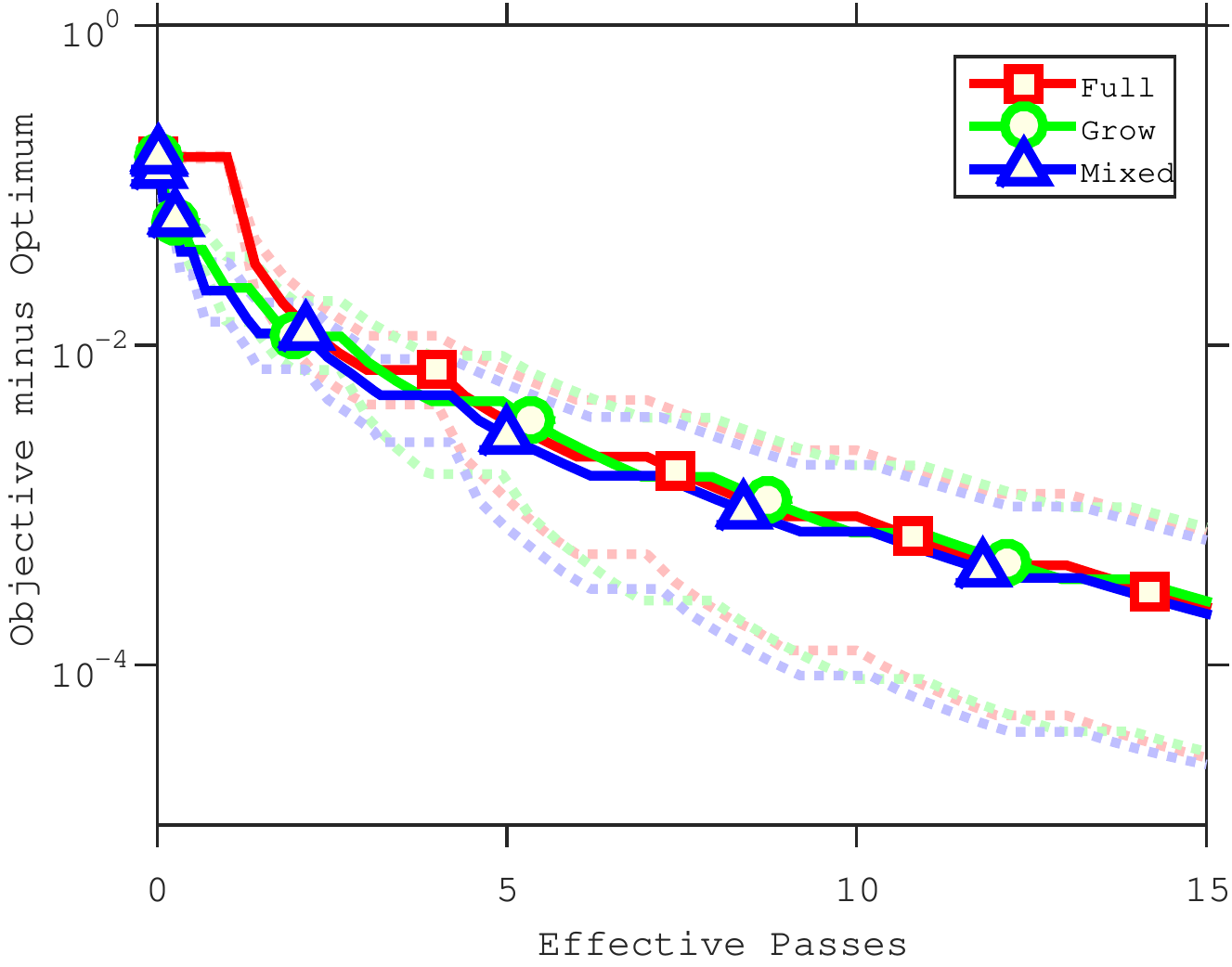}
\includegraphics[width=.32\textwidth]{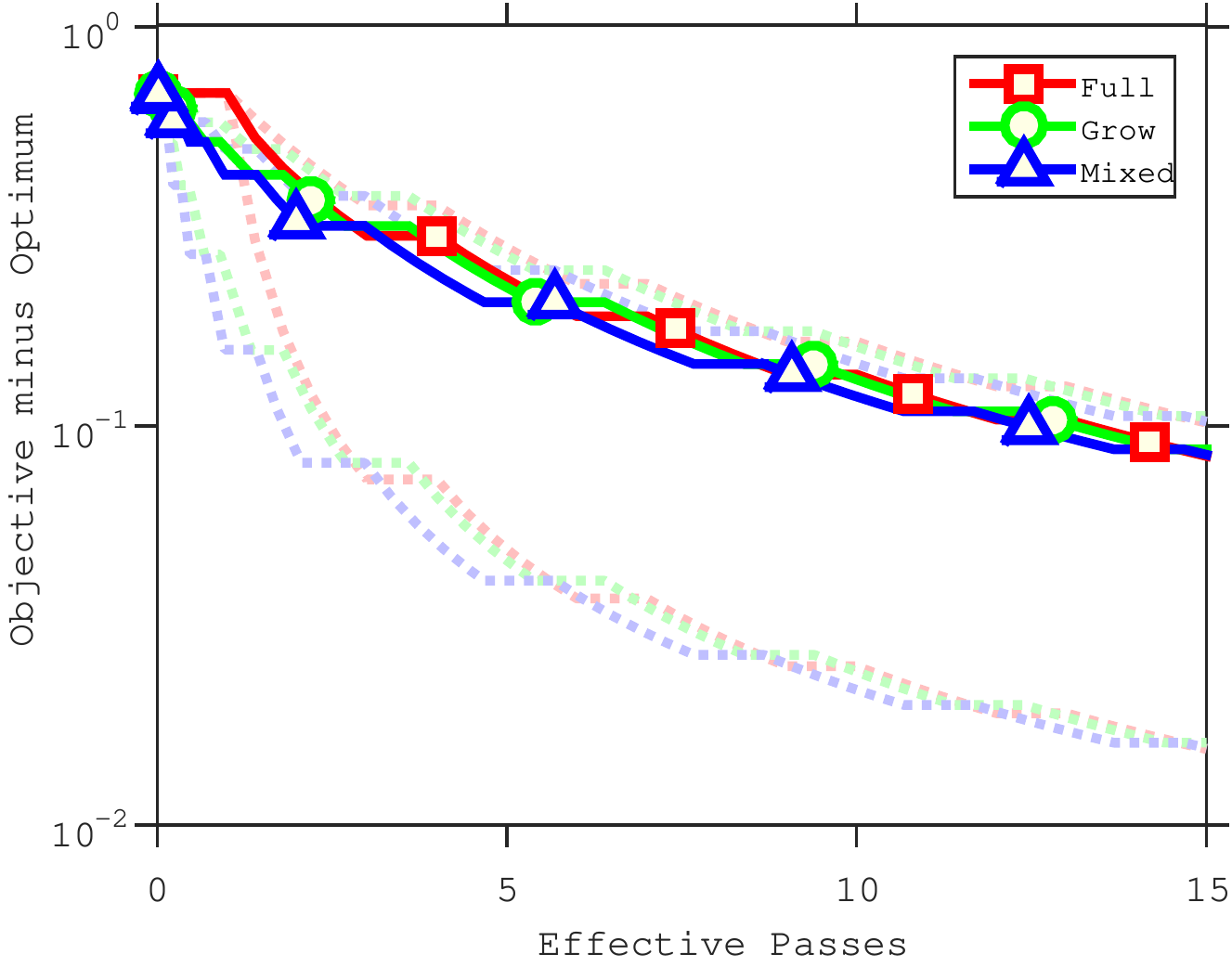}
\includegraphics[width=.32\textwidth]{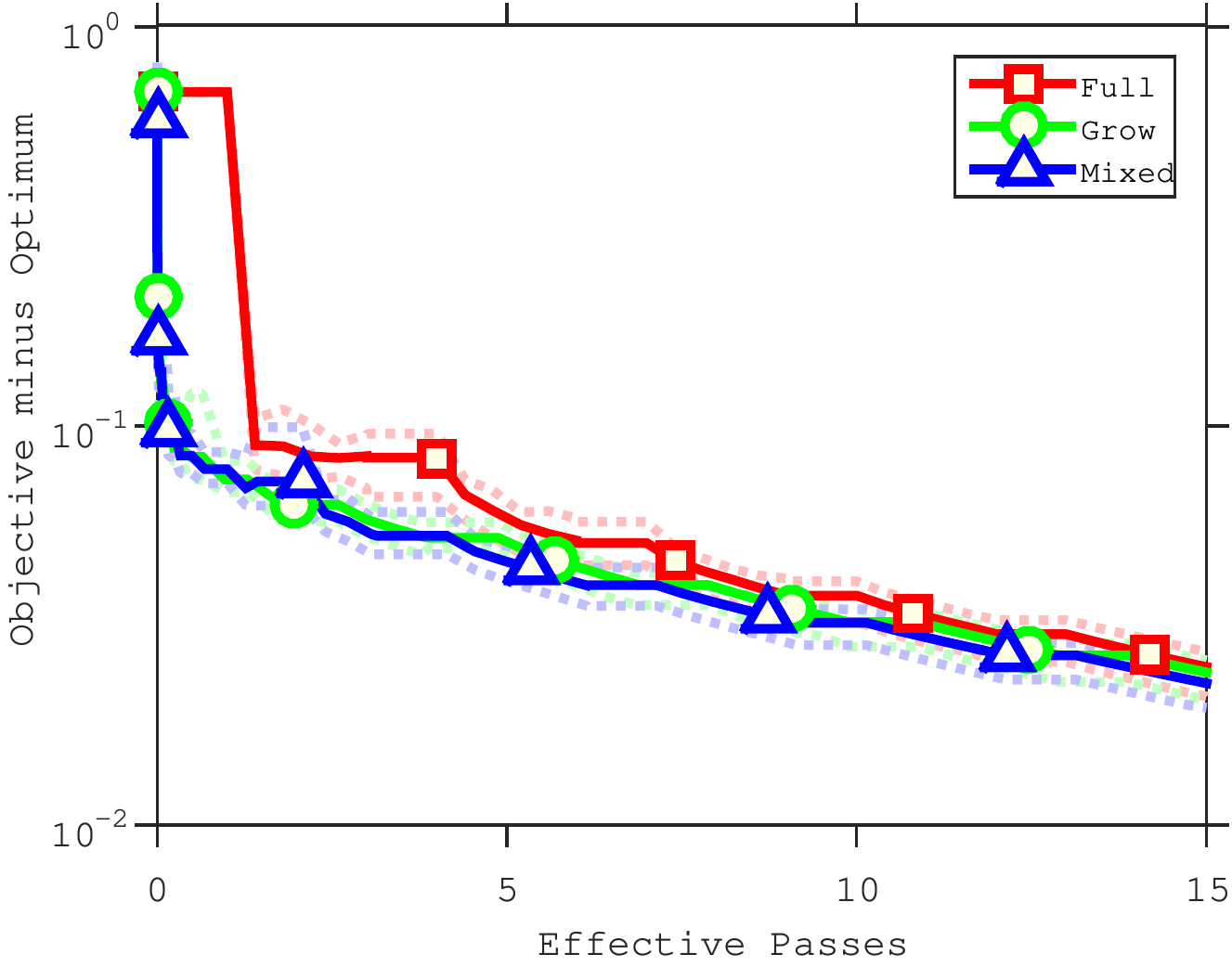}\\
\includegraphics[width=.32\textwidth]{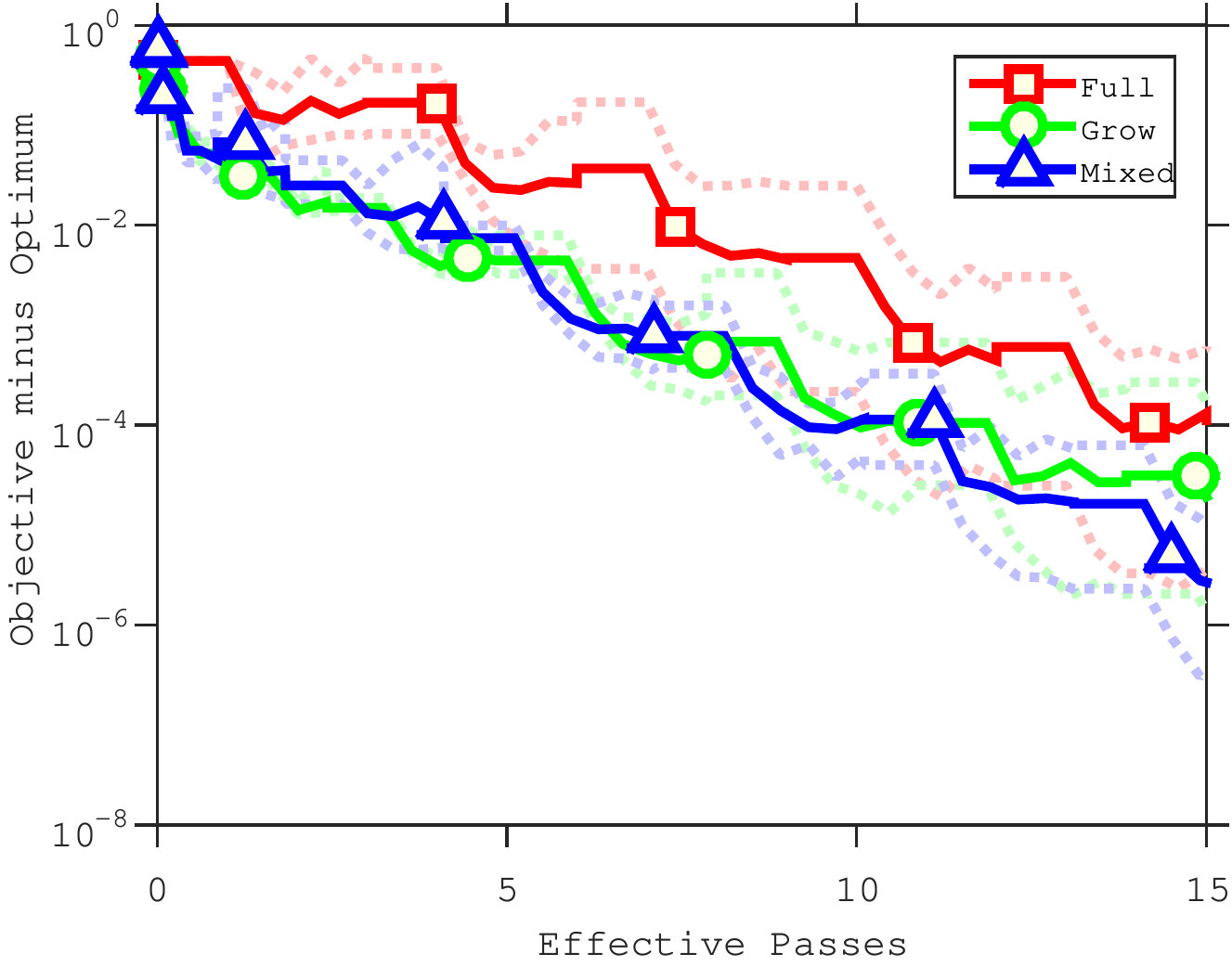}
\includegraphics[width=.32\textwidth]{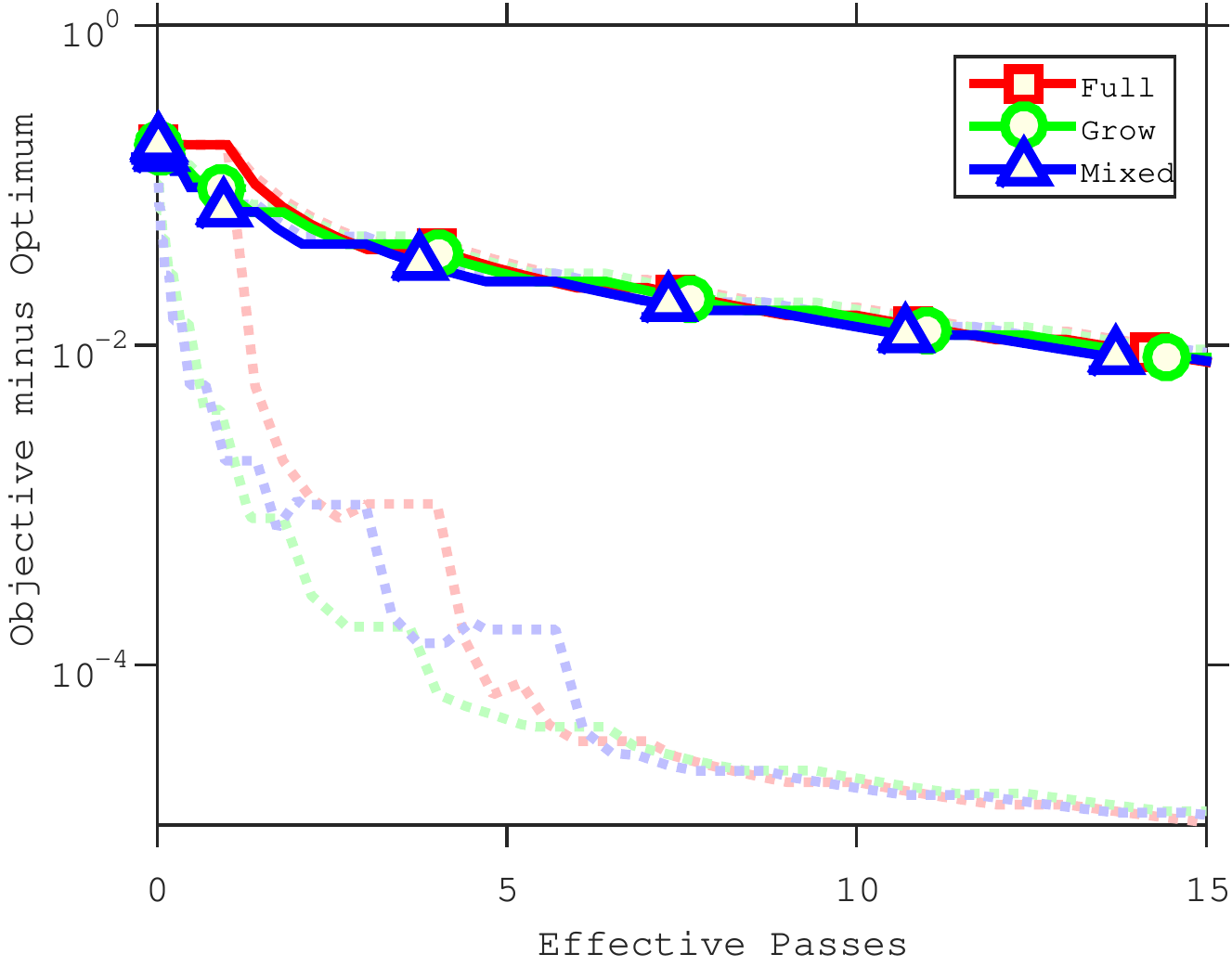}
\includegraphics[width=.32\textwidth]{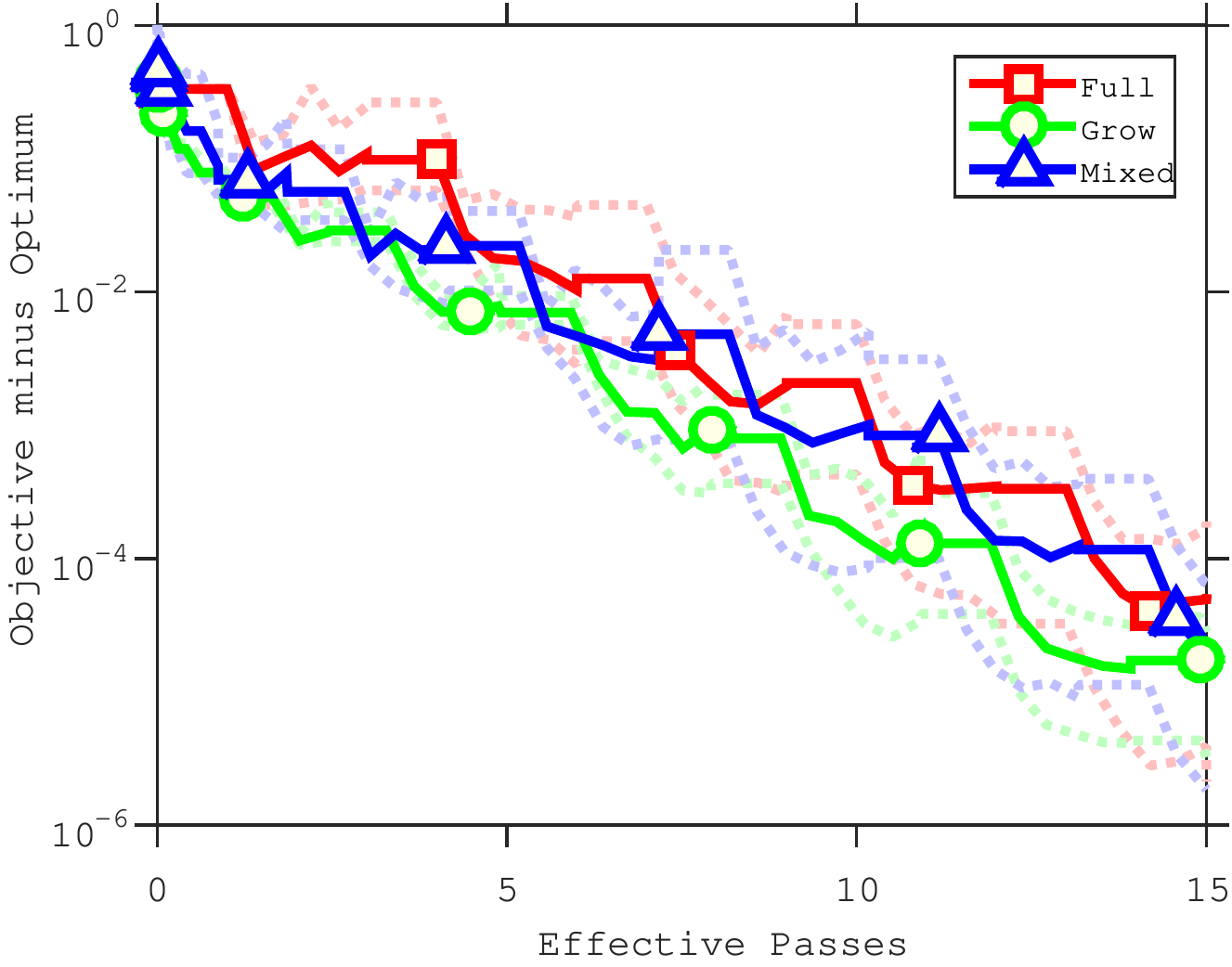}\\
\includegraphics[width=.32\textwidth]{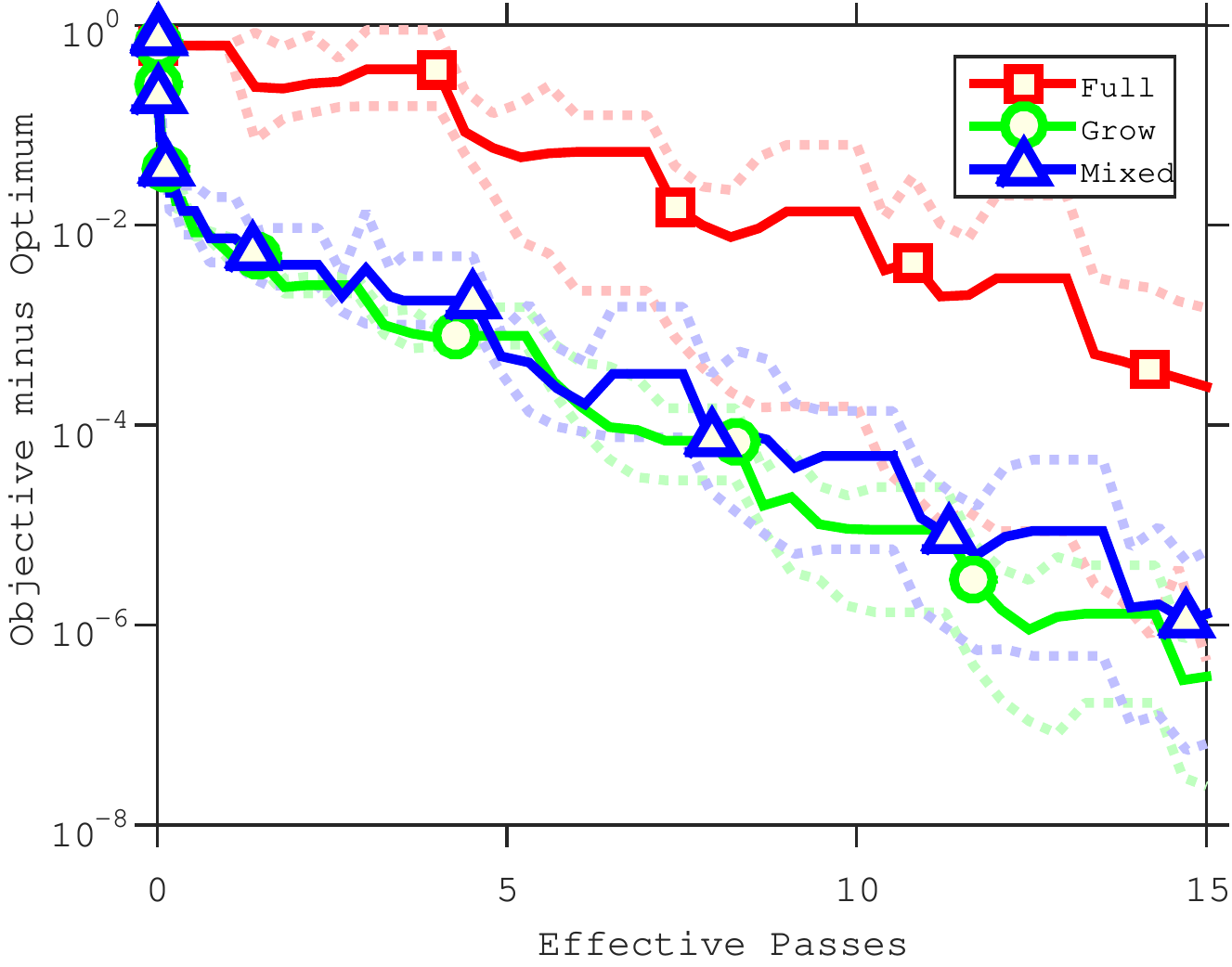}
\includegraphics[width=.32\textwidth]{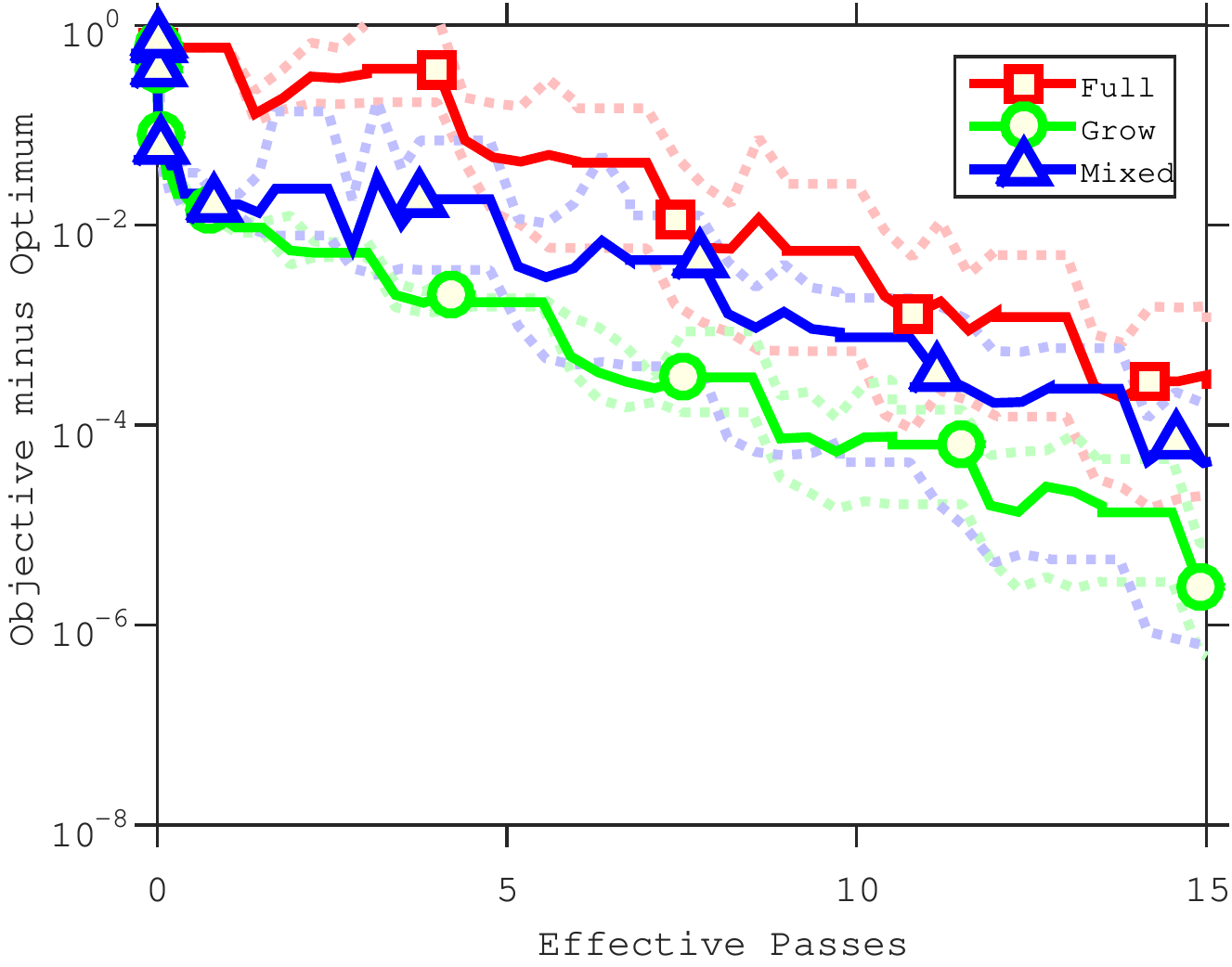}
\includegraphics[width=.32\textwidth]{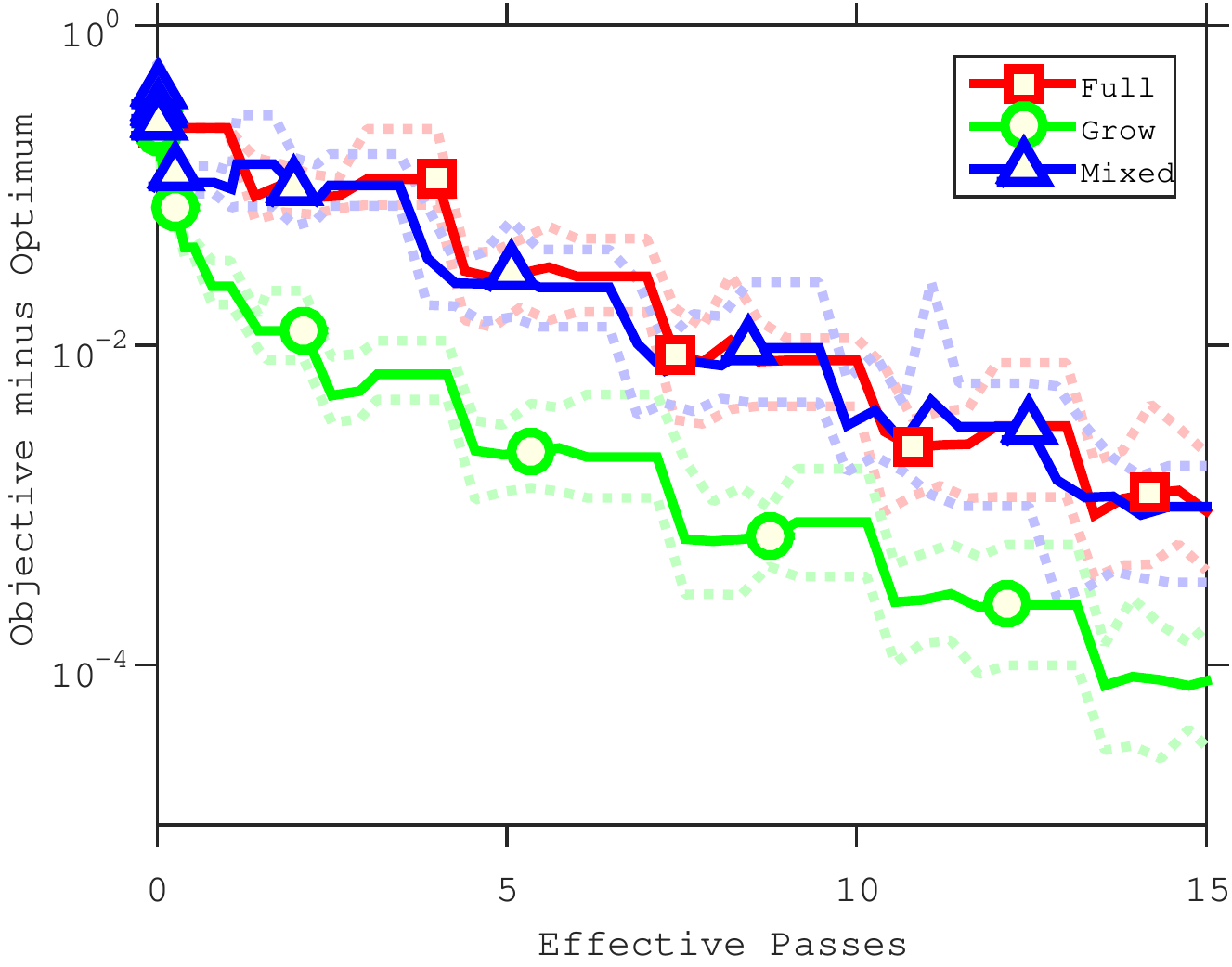}

\caption{Comparison of training objective of logistic regression for different datasets. The top row gives results on the \emph{quantum} (left), \emph{protein} (center) and \emph{sido} (right) datasets. The middle row gives results on the \emph{rcv11} (left), \emph{covertype} (center) and \emph{news} (right) datasets.  The bottom row gives results on the \emph{spam} (left), \emph{rcv1Full} (center), and \emph{alpha} (right) datasets.}
%\label{fig:1}
\end{figure*}

\begin{figure*}
\includegraphics[width=.32\textwidth]{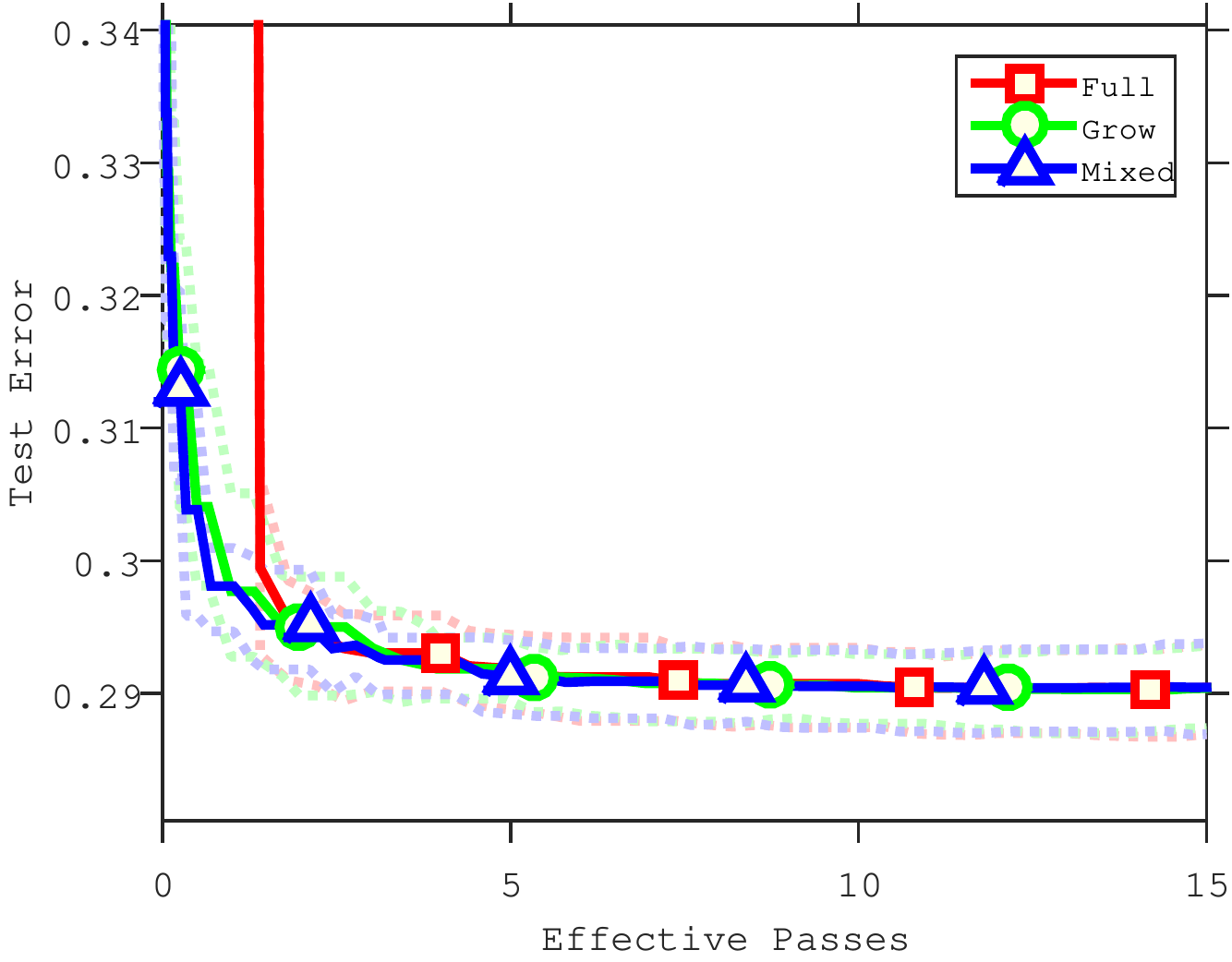}
\includegraphics[width=.32\textwidth]{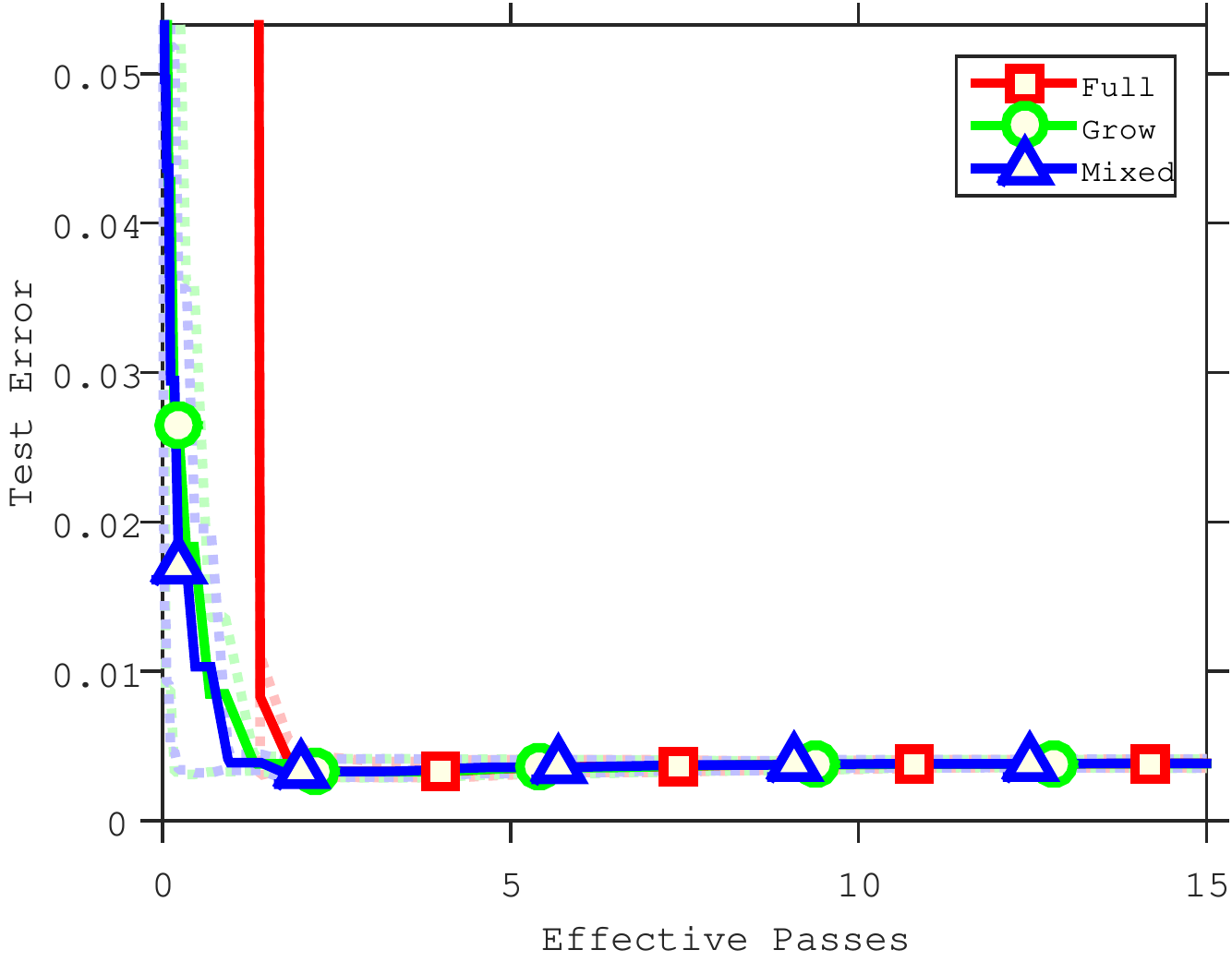}
\includegraphics[width=.32\textwidth]{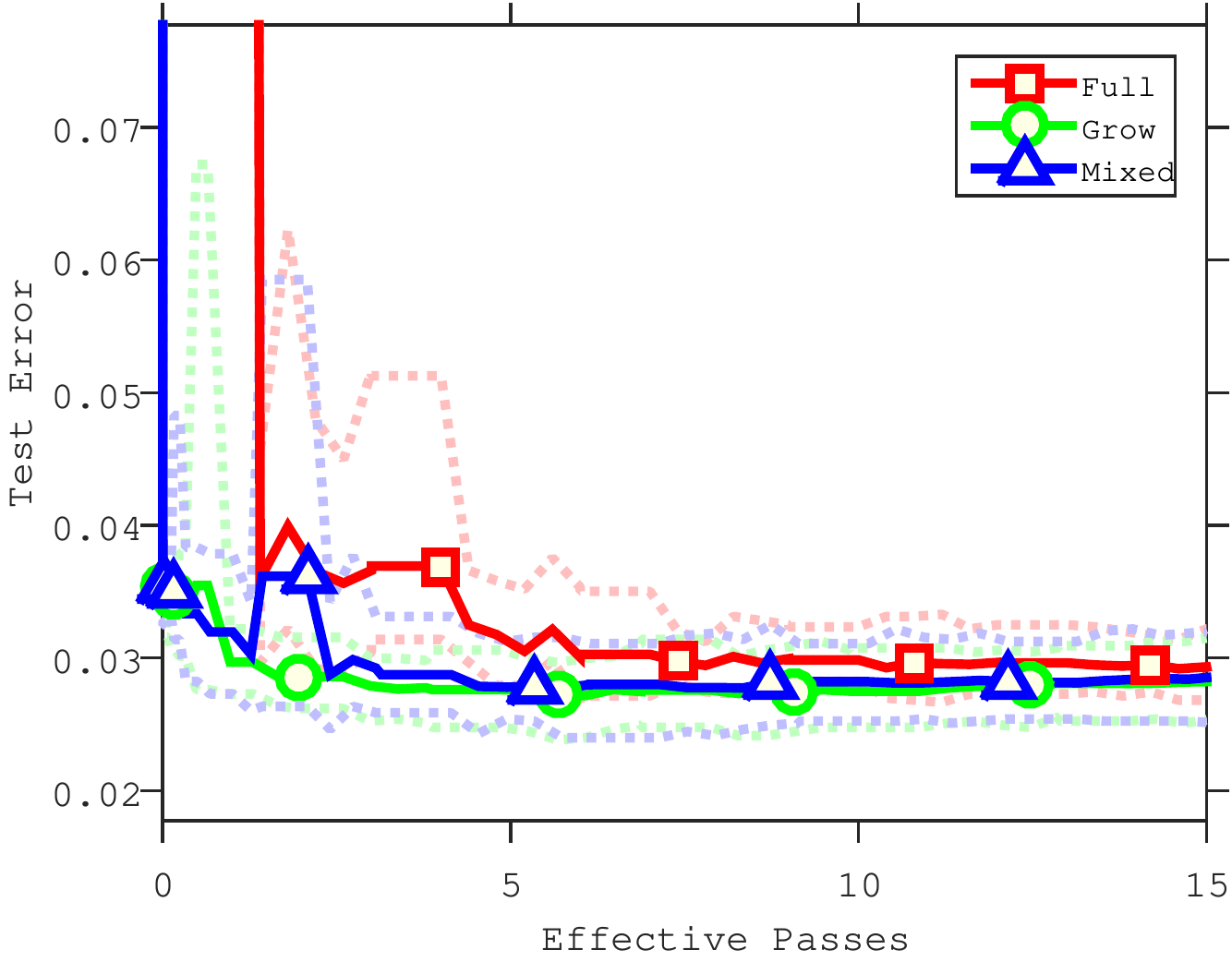}\\
\includegraphics[width=.32\textwidth]{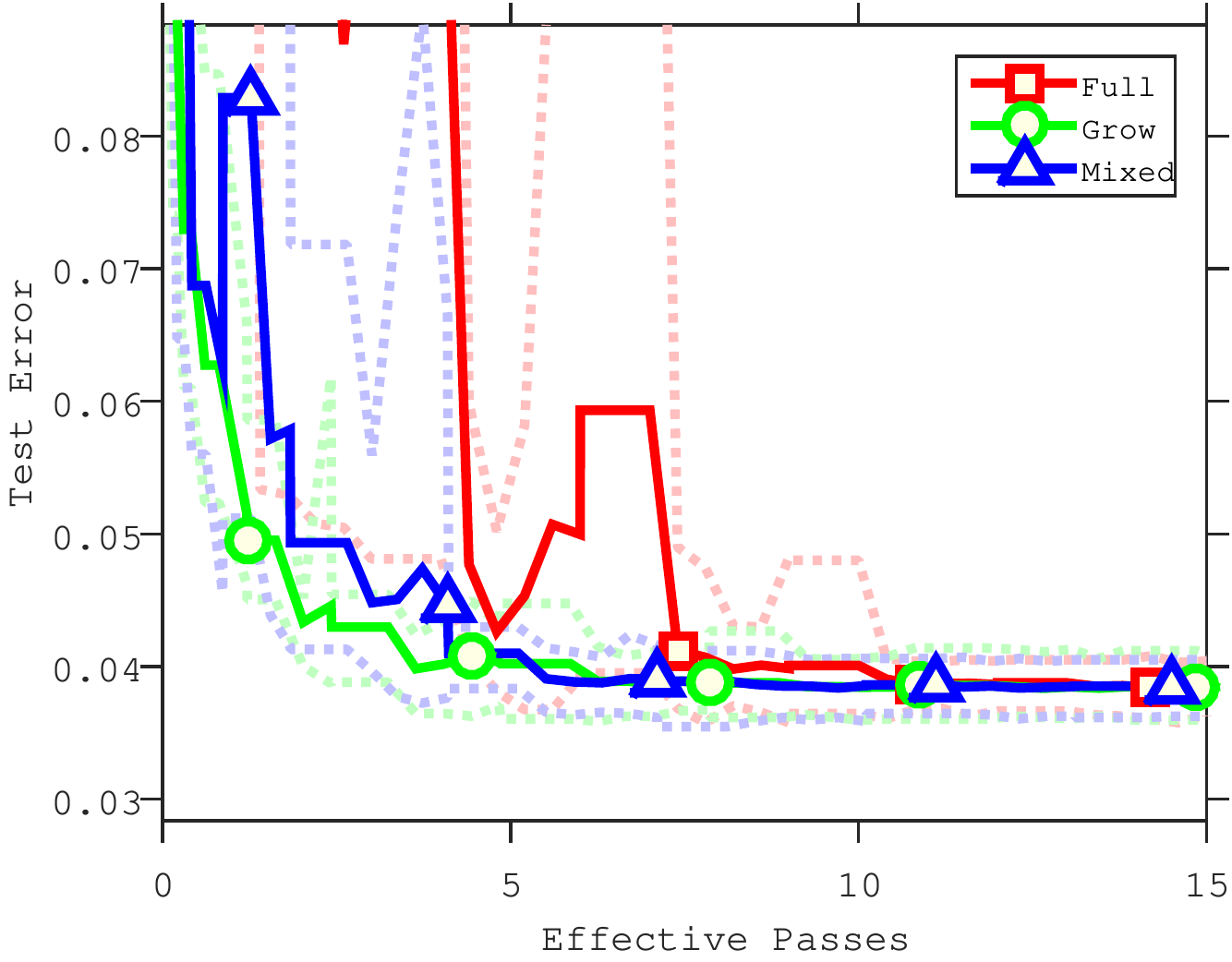}
\includegraphics[width=.32\textwidth]{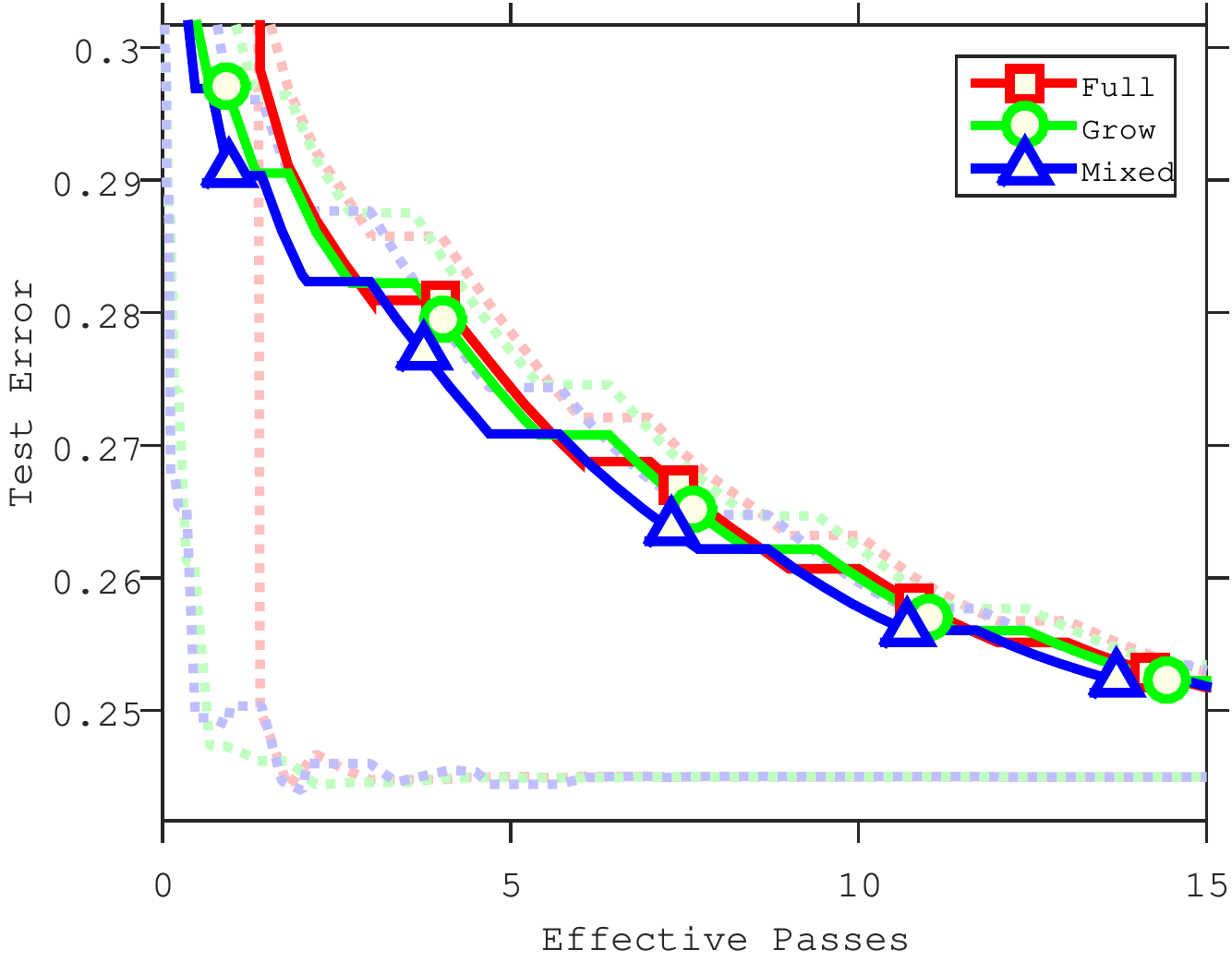}
\includegraphics[width=.32\textwidth]{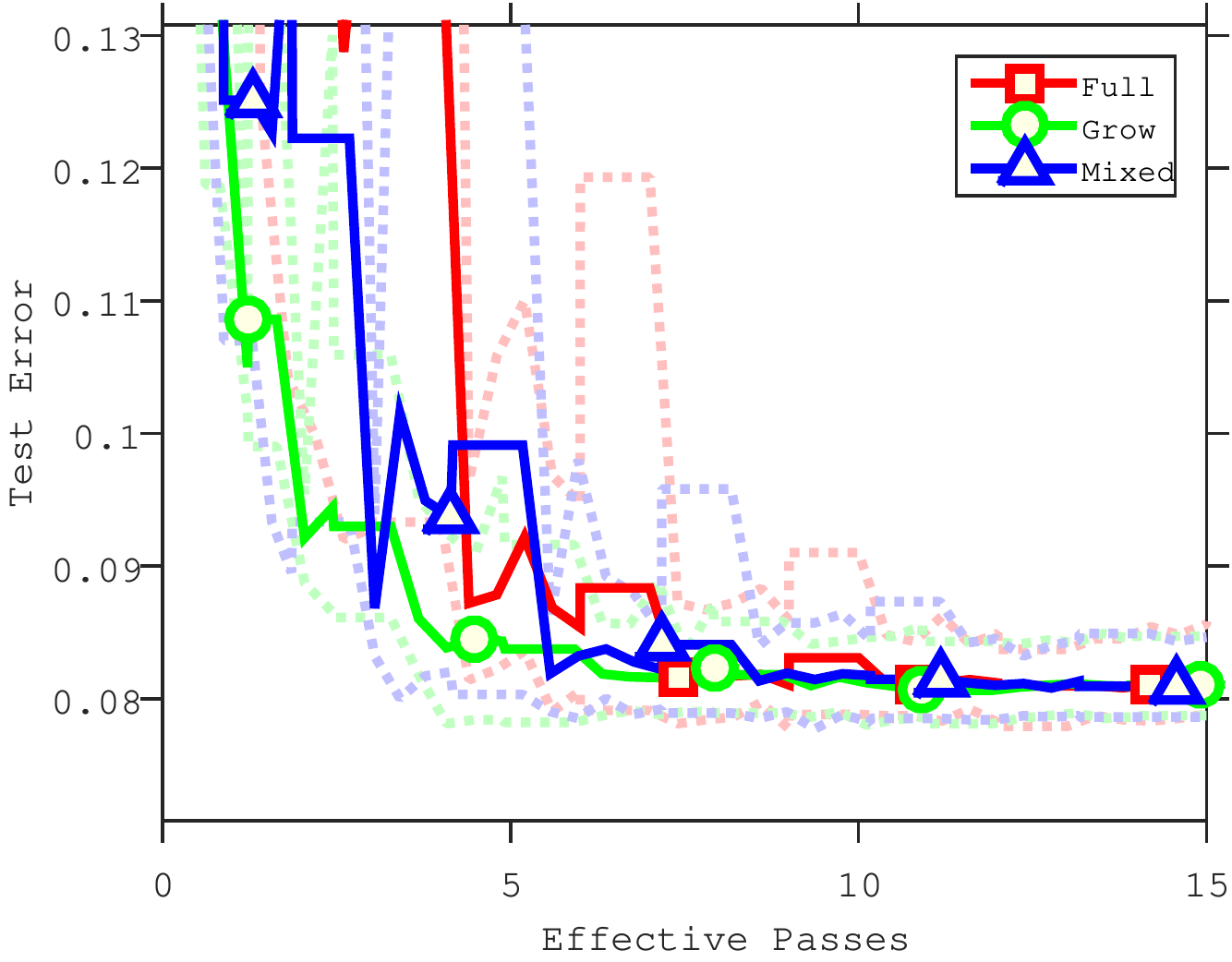}\\
\includegraphics[width=.32\textwidth]{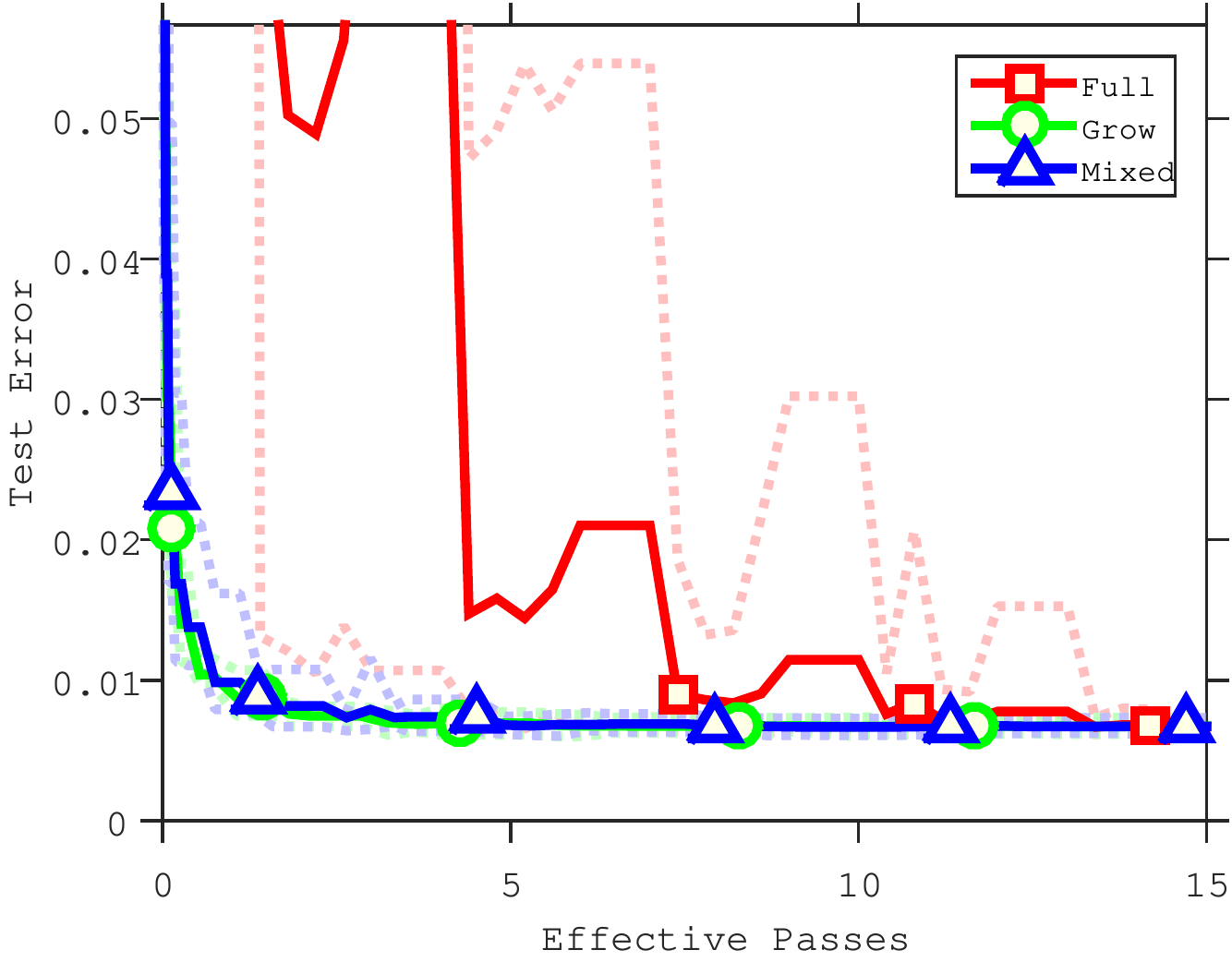}
\includegraphics[width=.32\textwidth]{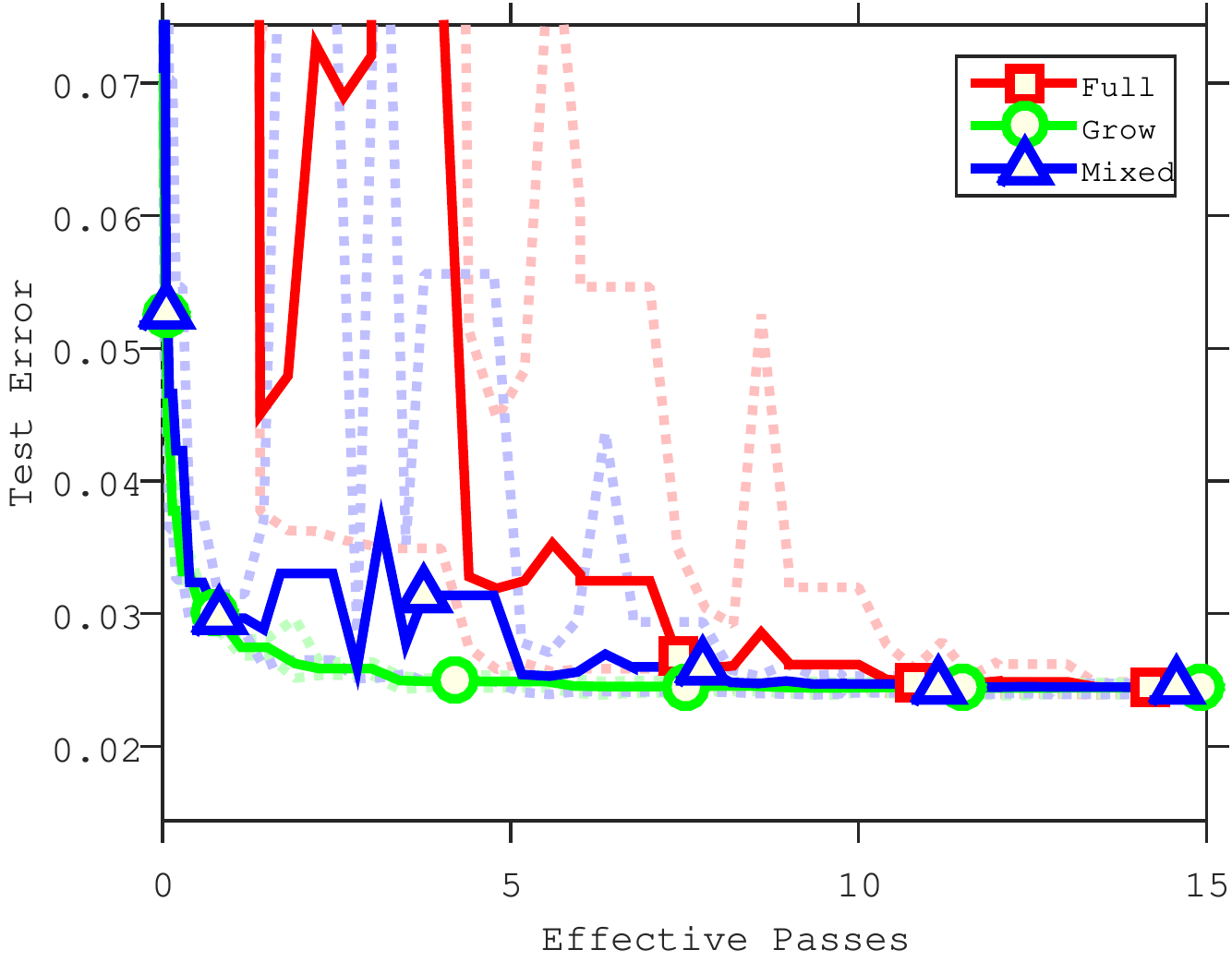}
\includegraphics[width=.32\textwidth]{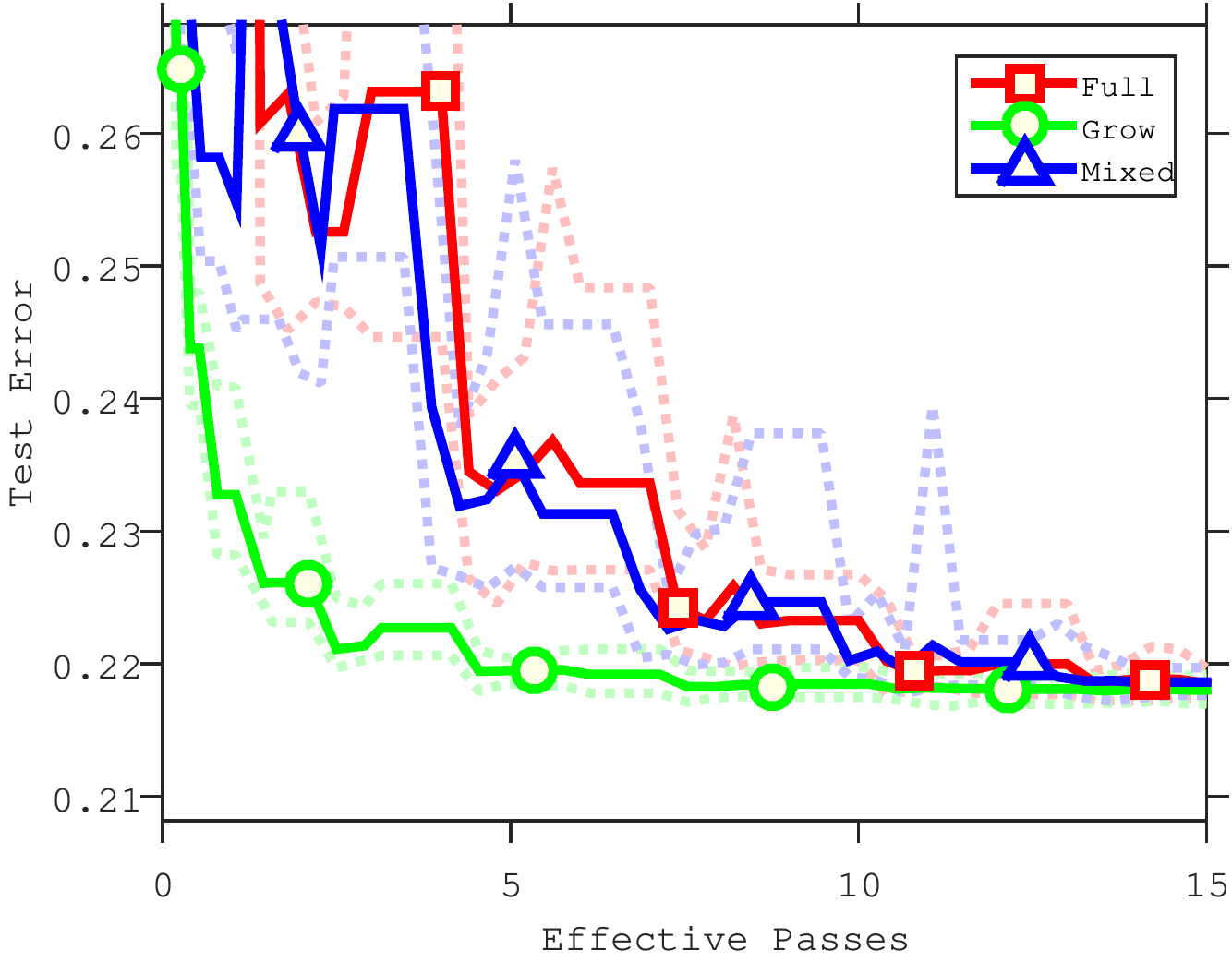}

\caption{Comparison of test error of logistic regression for different datasets. The top row gives results on the \emph{quantum} (left), \emph{protein} (center) and \emph{sido} (right) datasets. The middle row gives results on the \emph{rcv11} (left), \emph{covertype} (center) and \emph{news} (right) datasets.  The bottom row gives results on the \emph{spam} (left), \emph{rcv1Full} (center), and \emph{alpha} (right) datasets.}
%\label{fig:2}
\end{figure*}
%\end{figure*}

\begin{figure*}
\includegraphics[width=.32\textwidth]{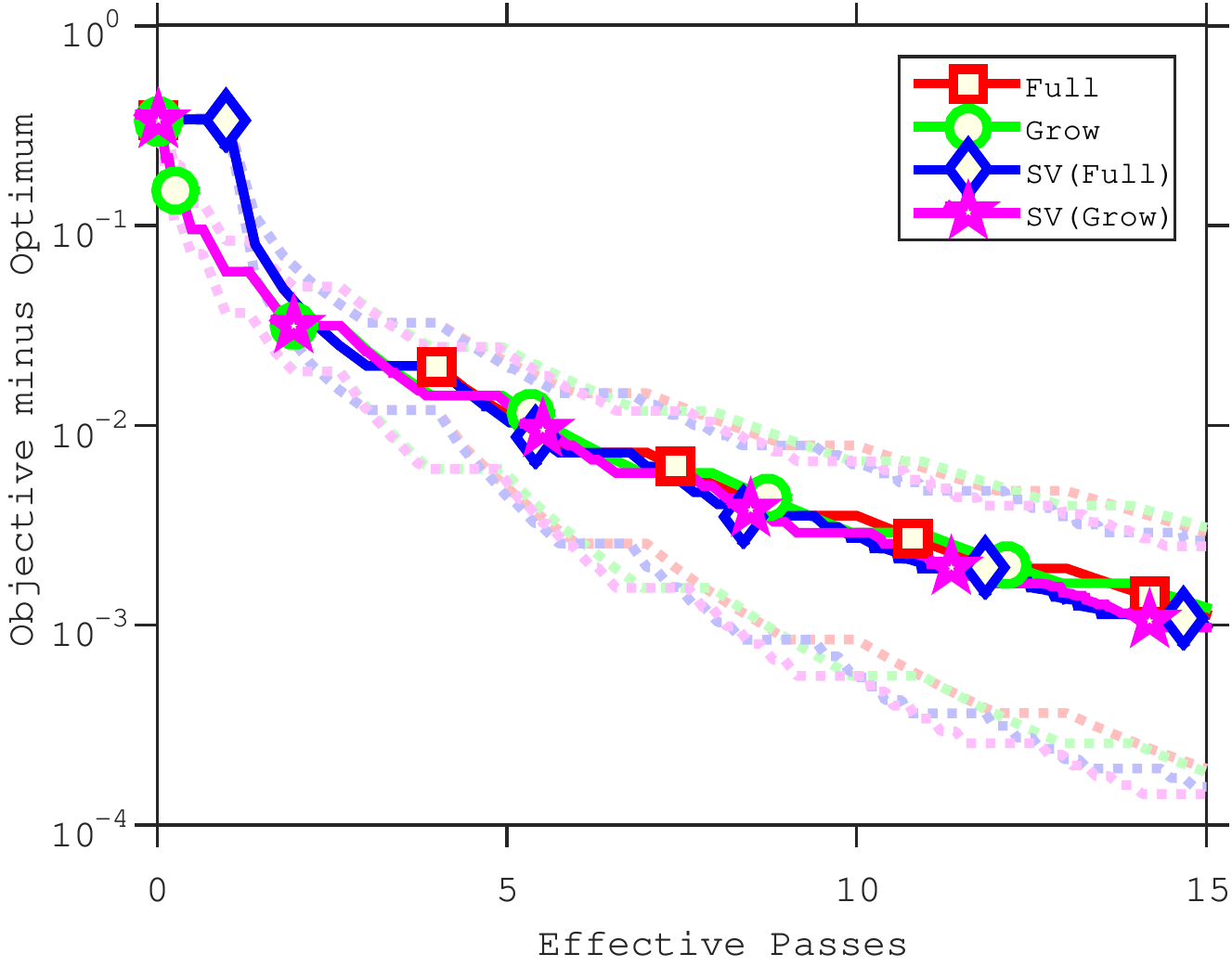}
\includegraphics[width=.32\textwidth]{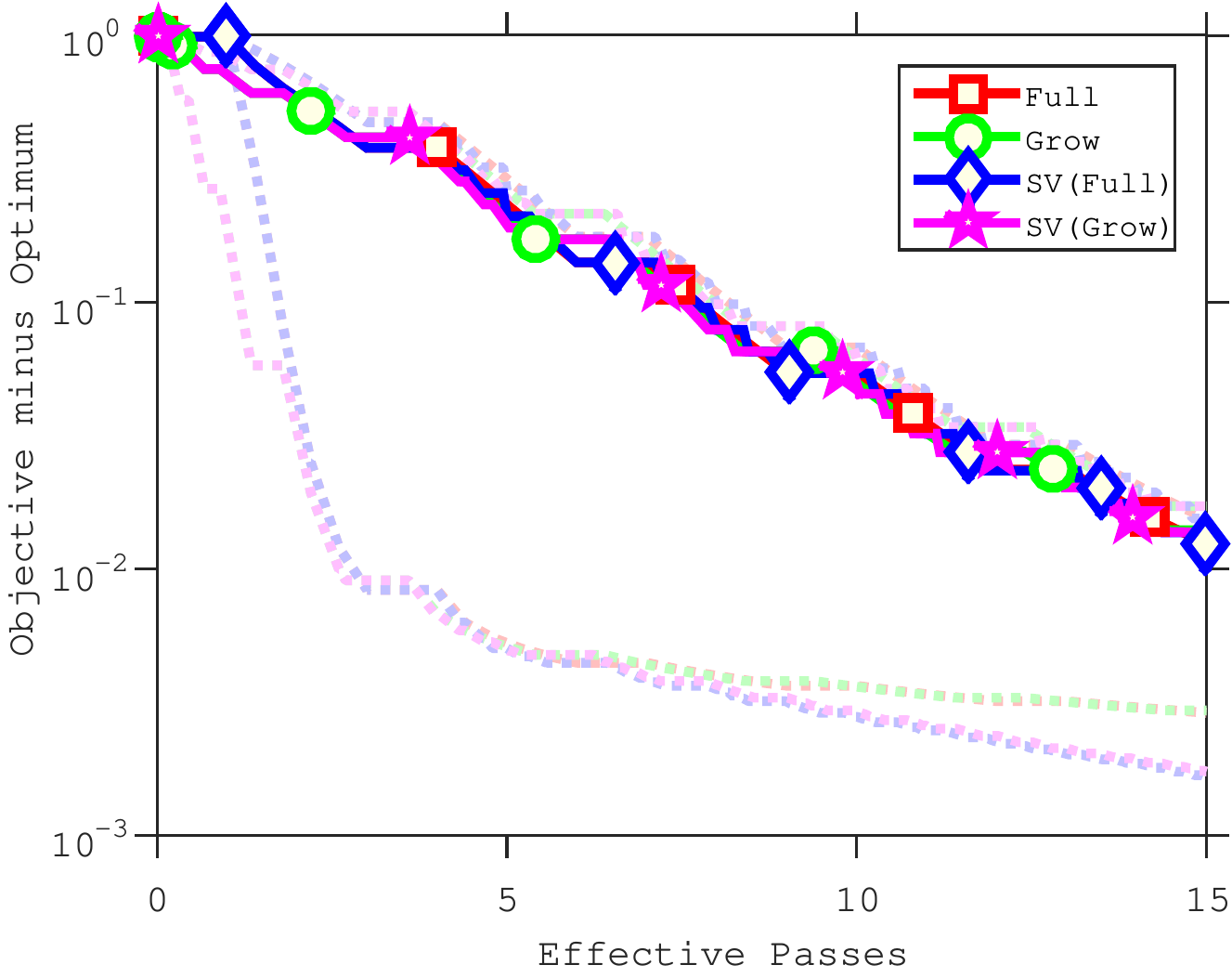}
\includegraphics[width=.32\textwidth]{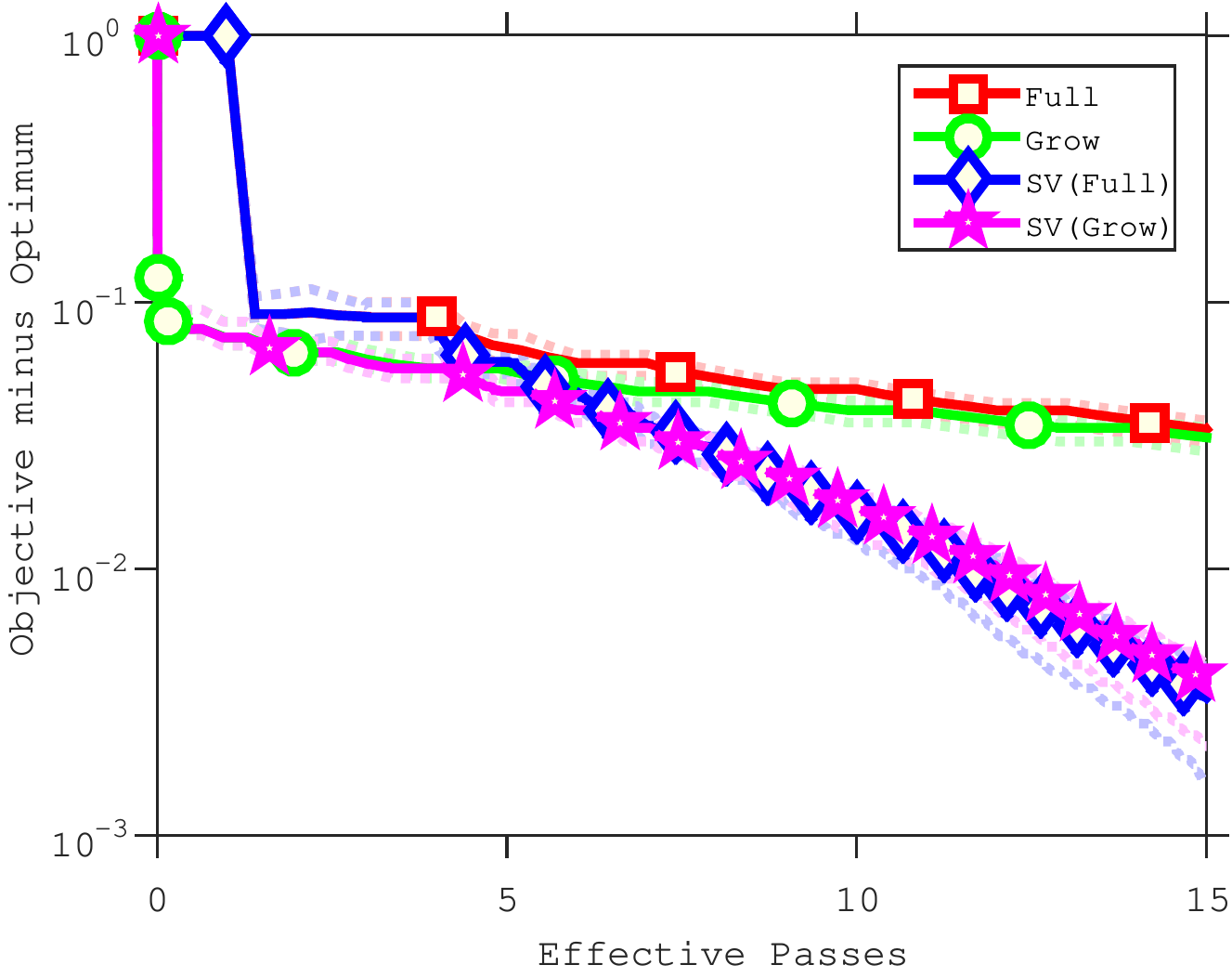}\\
\includegraphics[width=.32\textwidth]{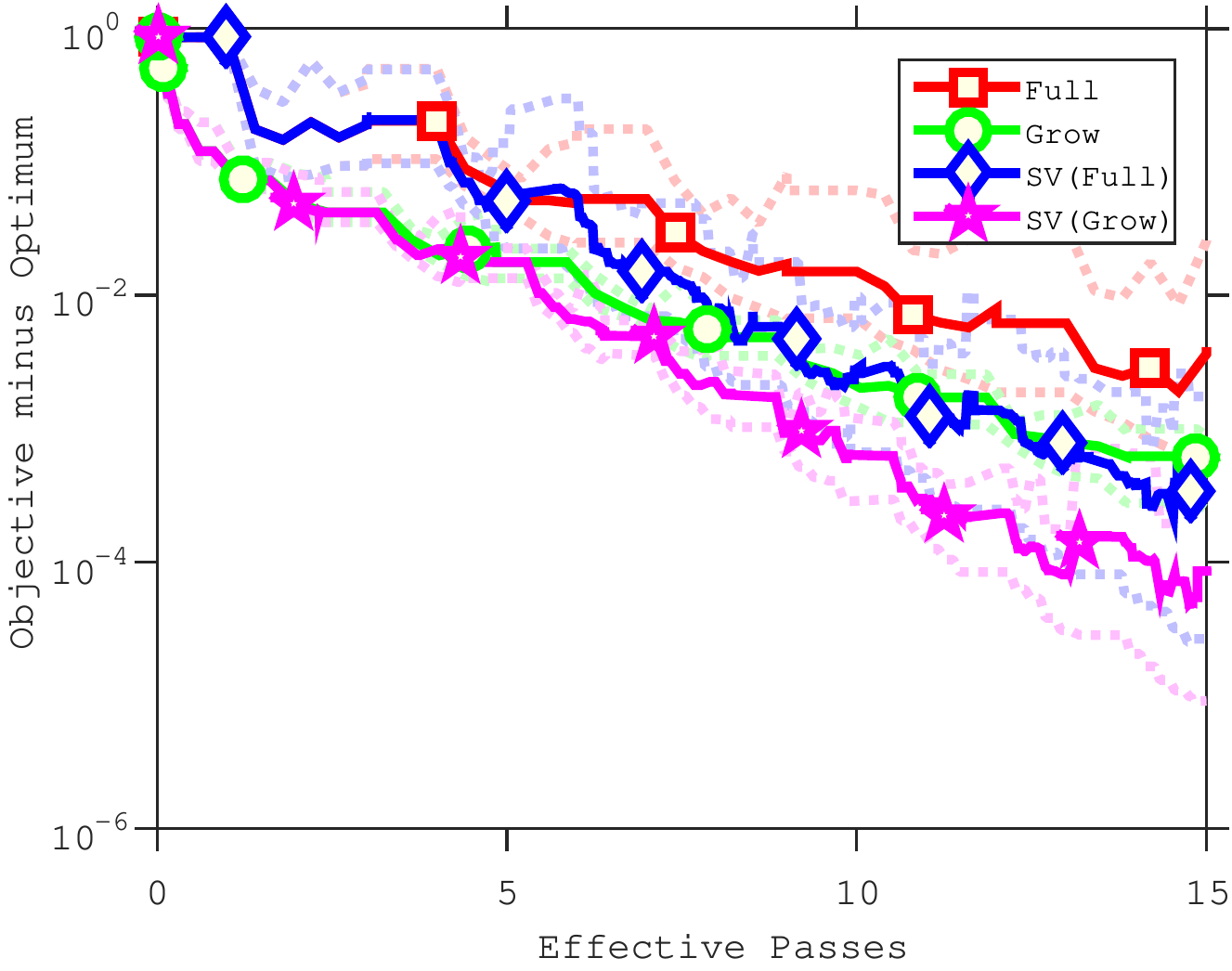}
\includegraphics[width=.32\textwidth]{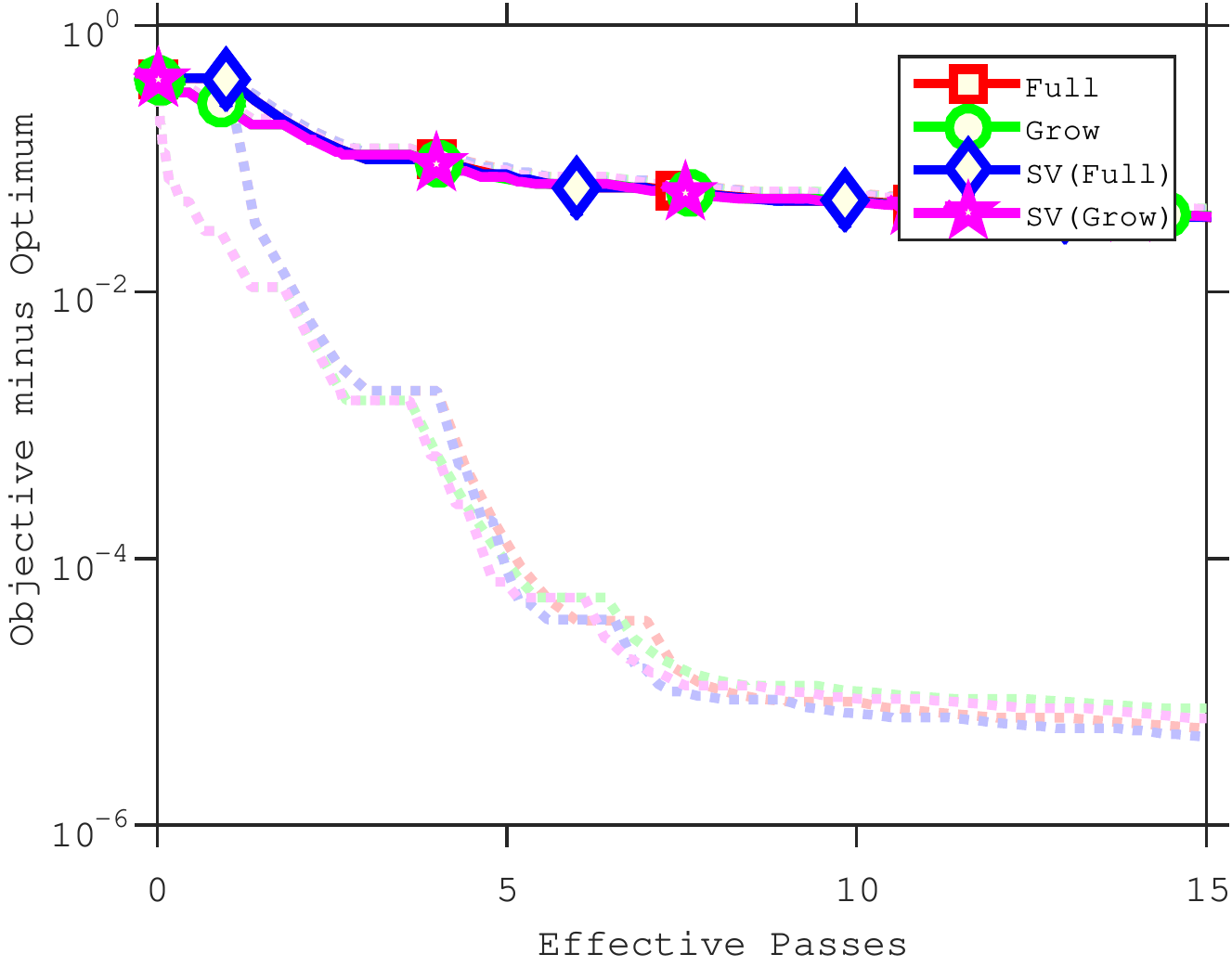}
\includegraphics[width=.32\textwidth]{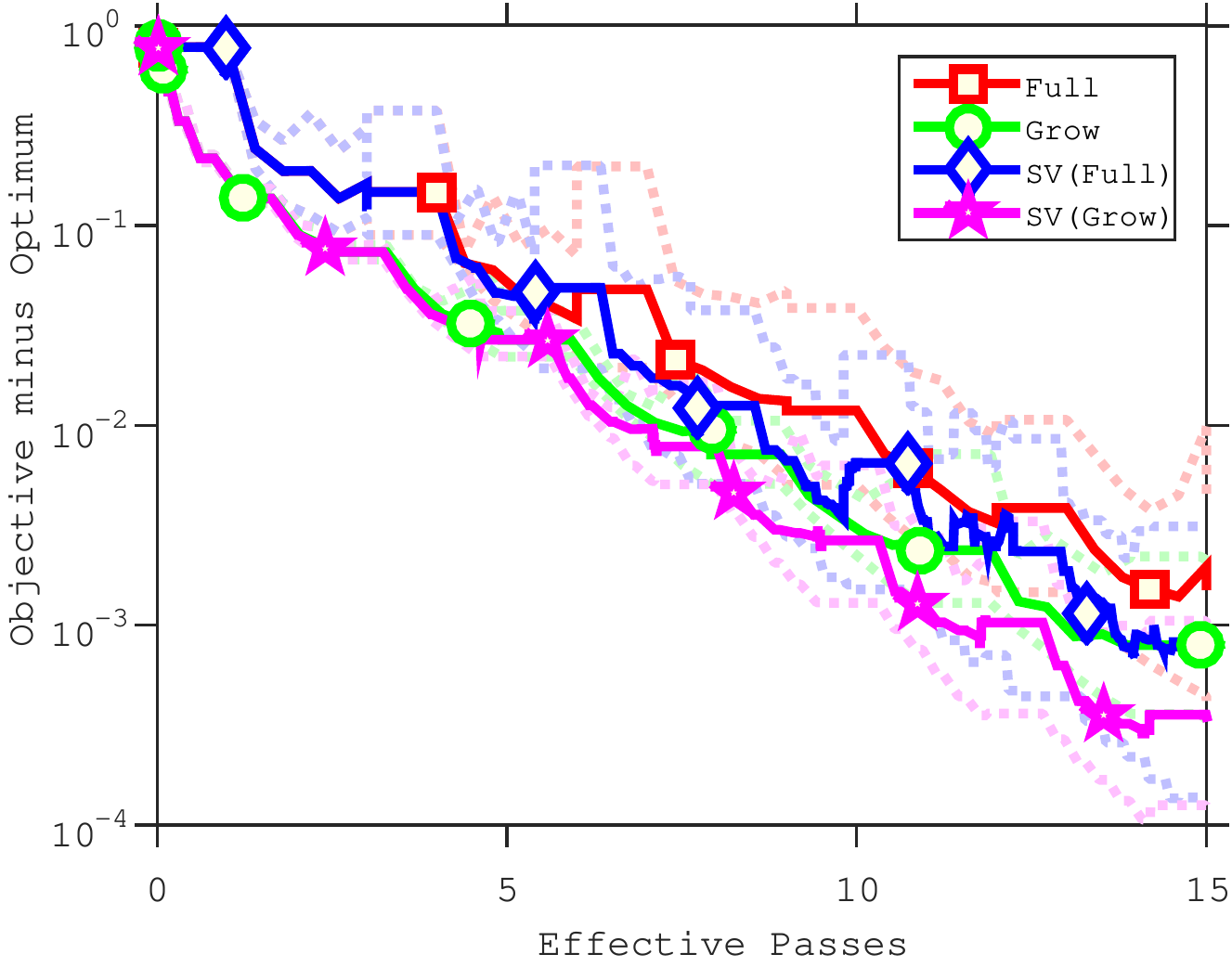}\\
\includegraphics[width=.32\textwidth]{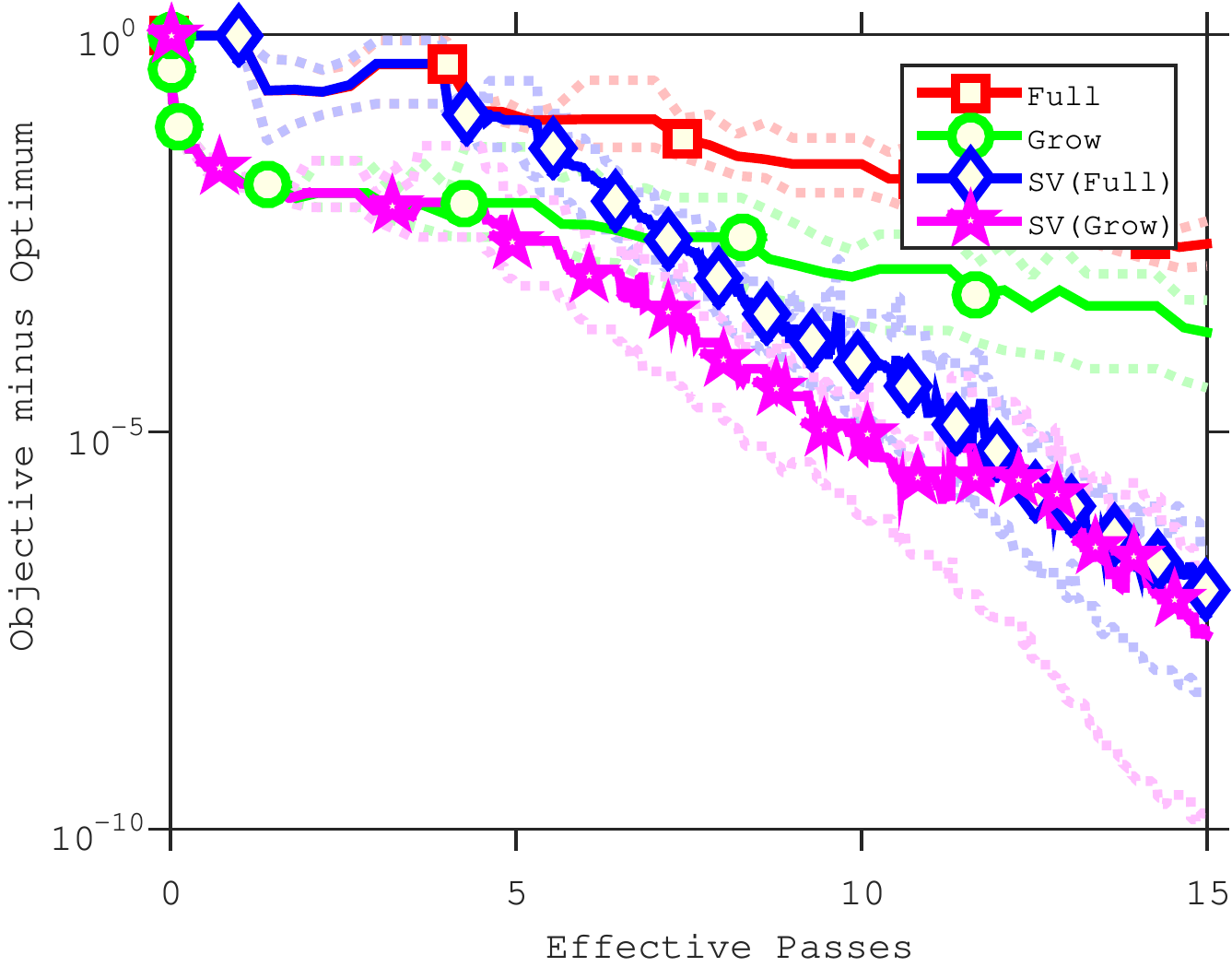}
\includegraphics[width=.32\textwidth]{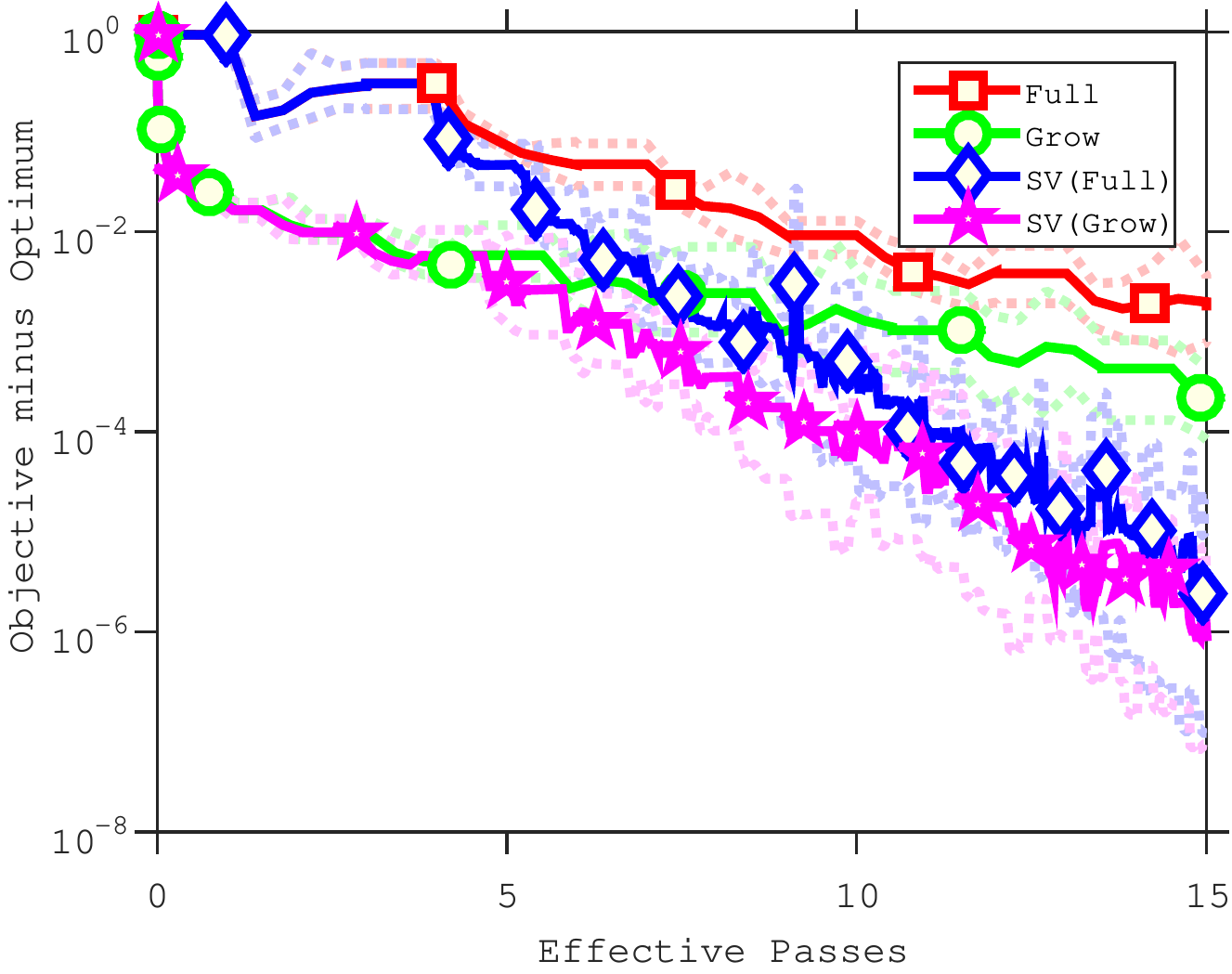}
\includegraphics[width=.32\textwidth]{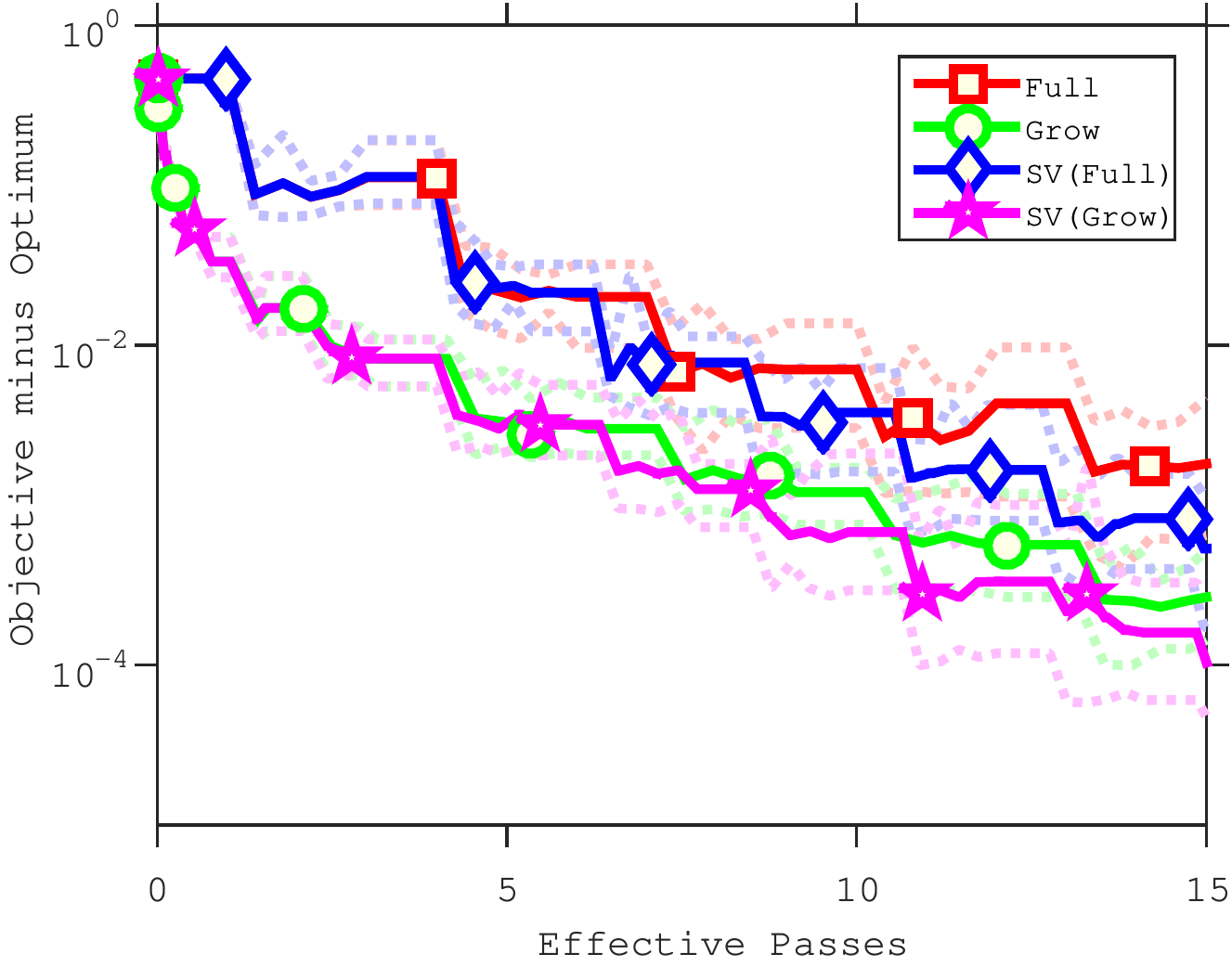}

\caption{Comparison of training objective of SVM for different datasets. The top row gives results on the \emph{quantum} (left), \emph{protein} (center) and \emph{sido} (right) datasets. The middle row gives results on the \emph{rcv11} (left), \emph{covertype} (center) and \emph{news} (right) datasets.  The bottom row gives results on the \emph{spam} (left), \emph{rcv1Full} (center), and \emph{alpha} (right) datasets.}
%\label{fig:3}
\end{figure*}

\begin{figure*}
\includegraphics[width=.32\textwidth]{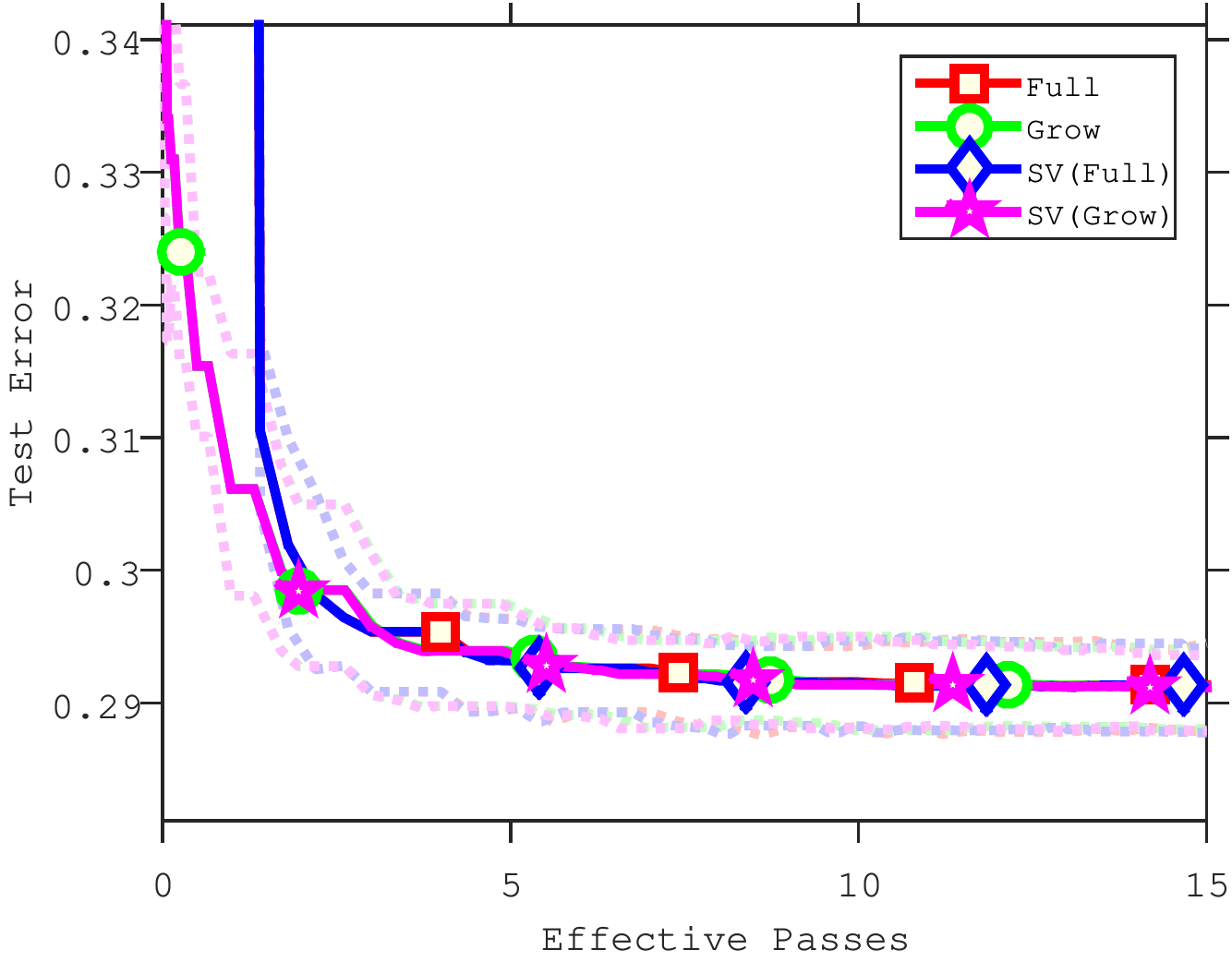}
\includegraphics[width=.32\textwidth]{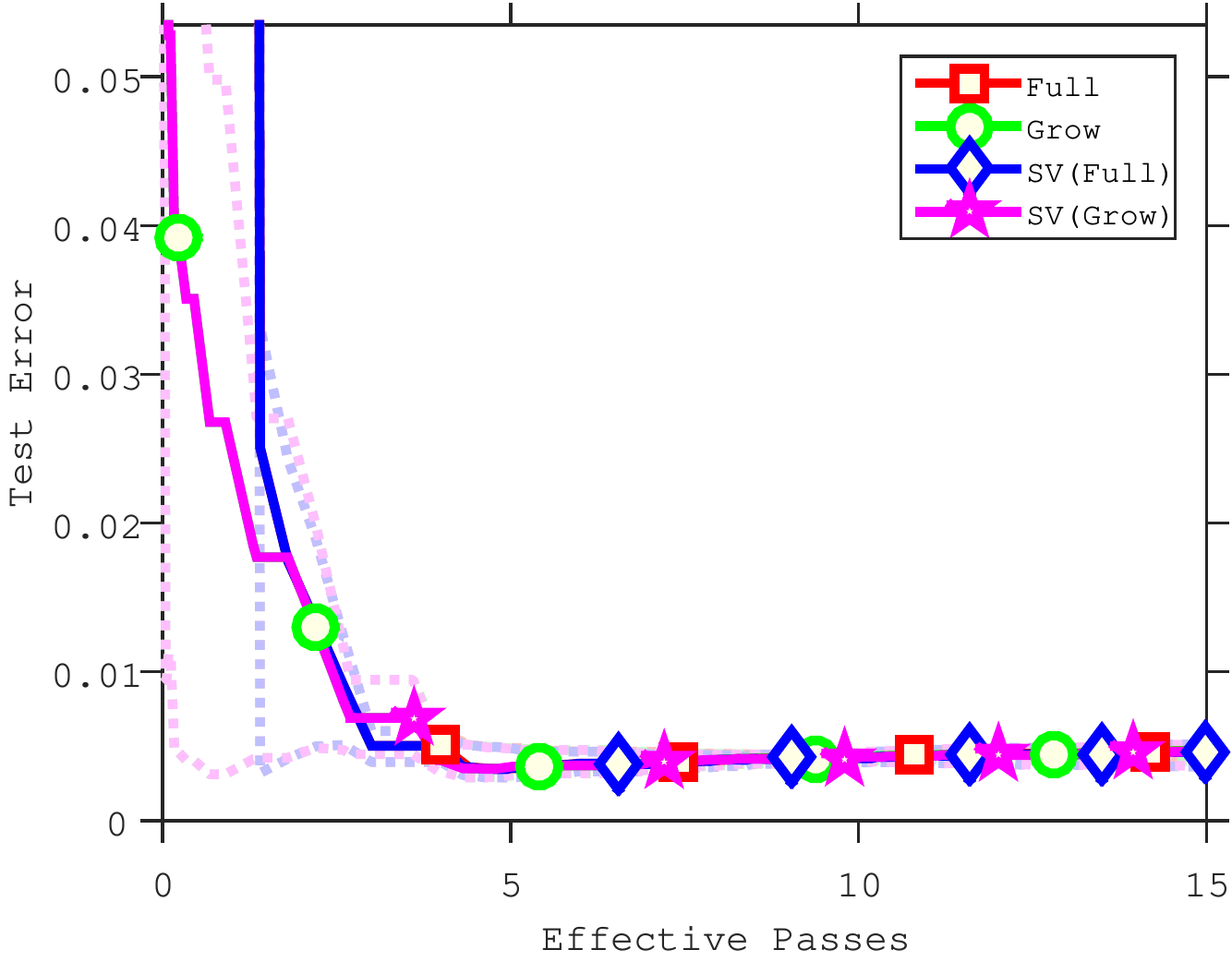}
\includegraphics[width=.32\textwidth]{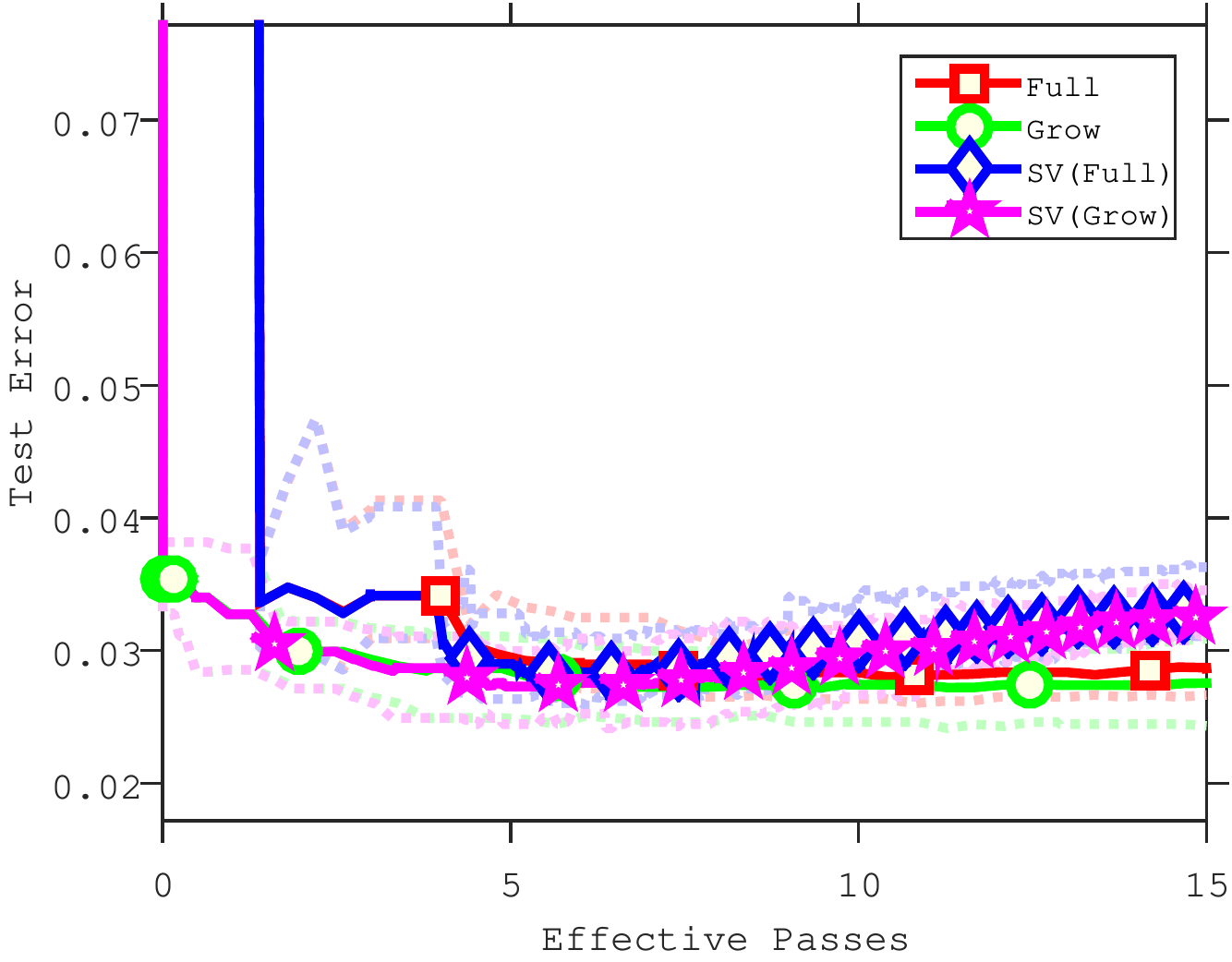}\\
\includegraphics[width=.32\textwidth]{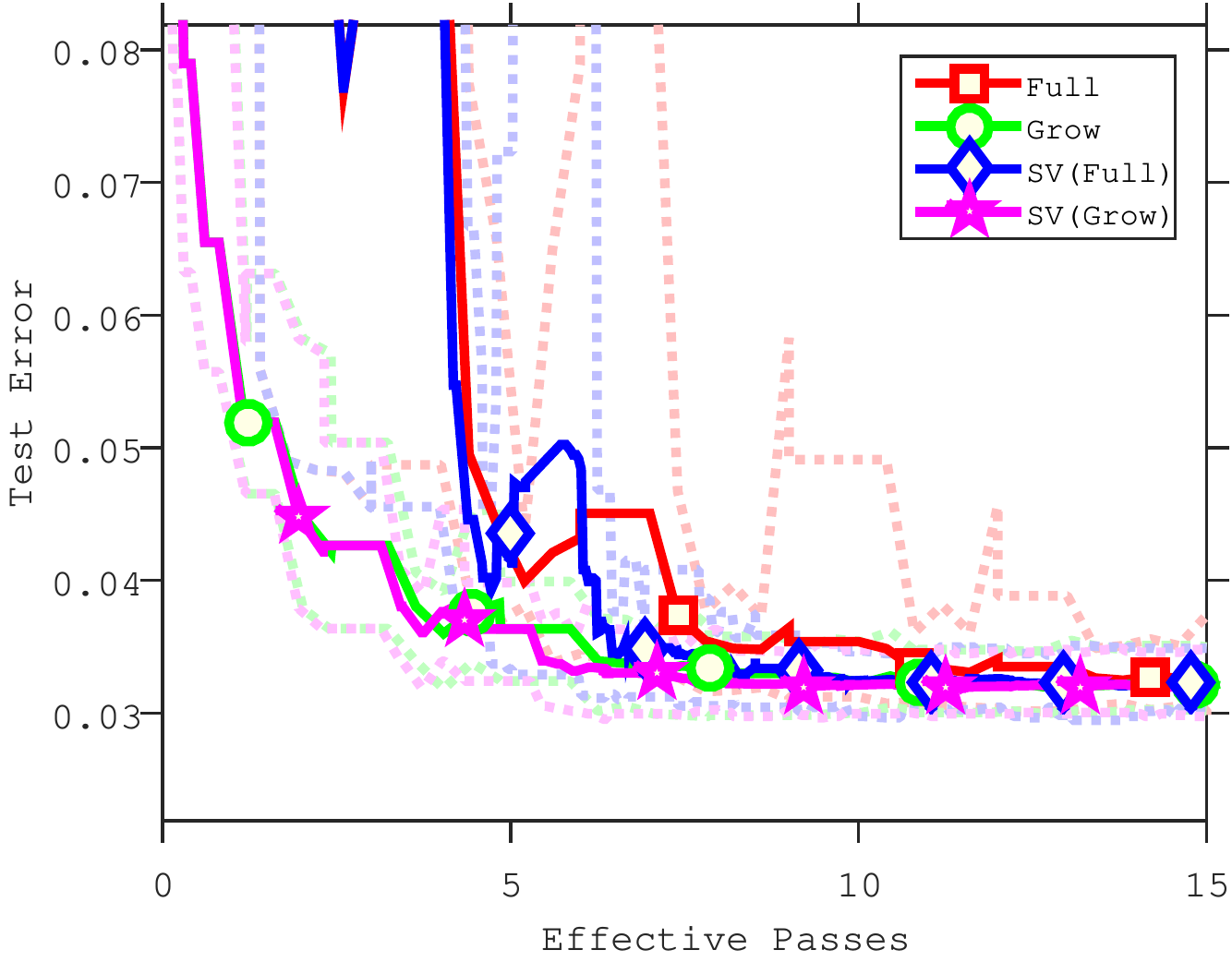}
\includegraphics[width=.32\textwidth]{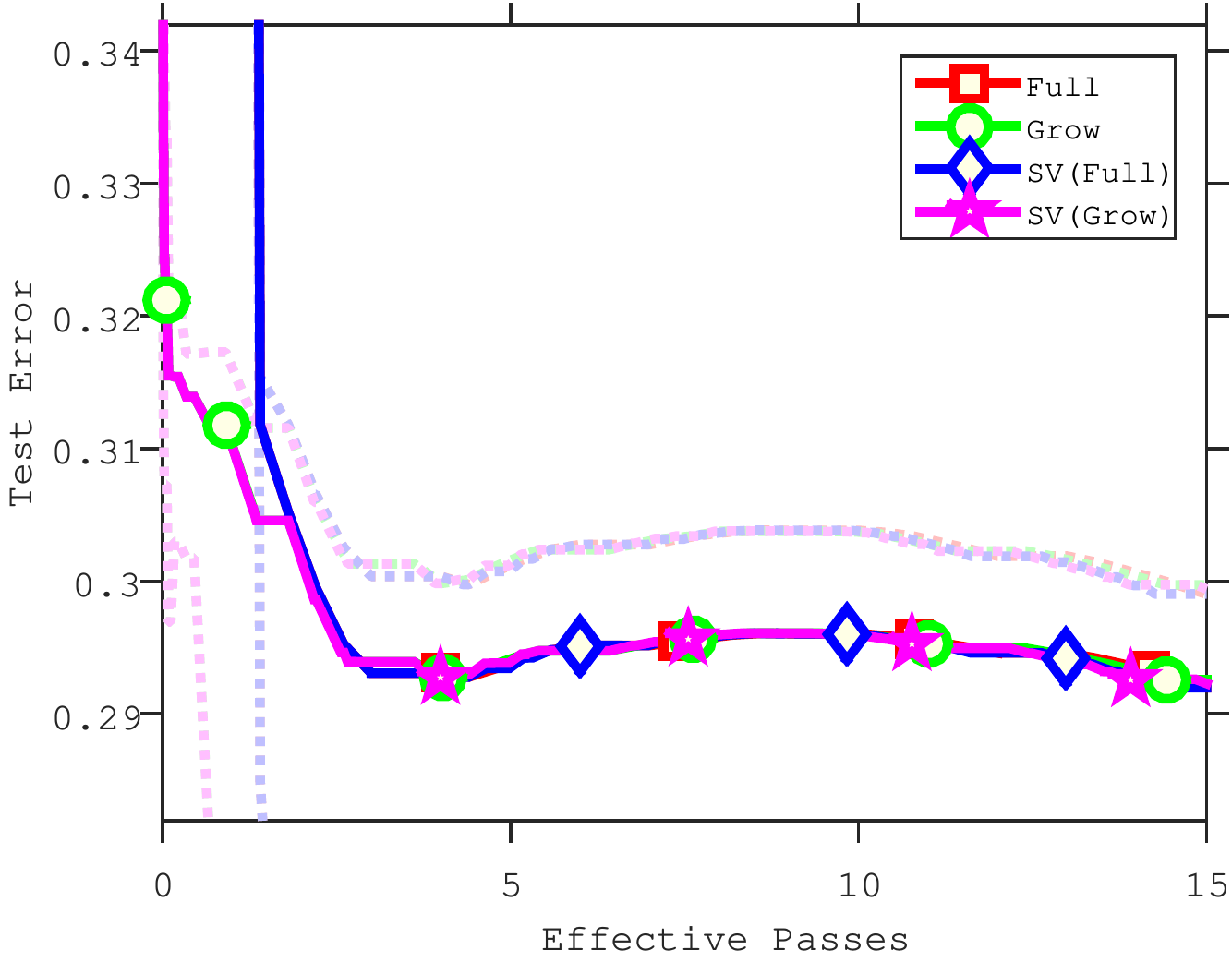}
\includegraphics[width=.32\textwidth]{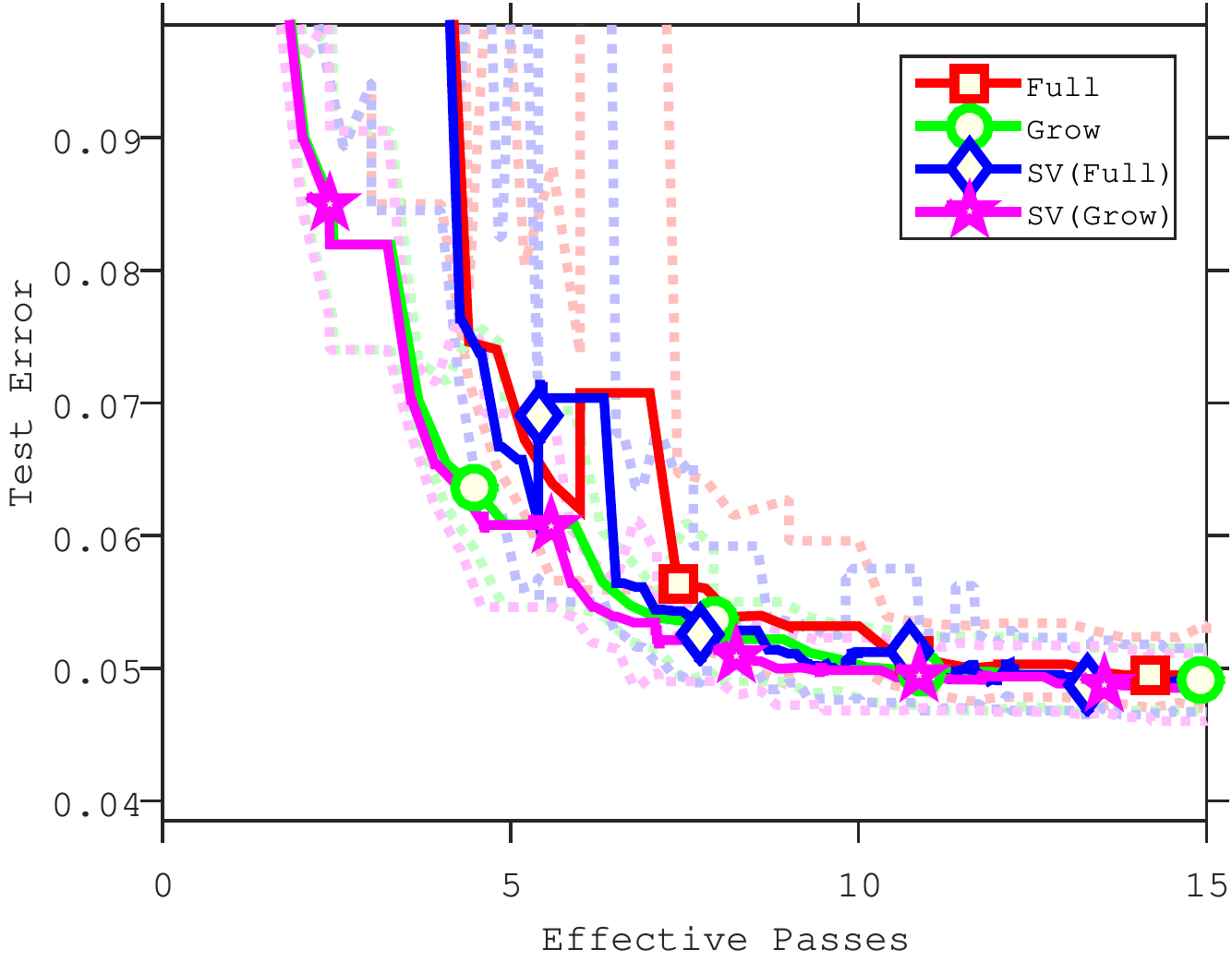}\\
\includegraphics[width=.32\textwidth]{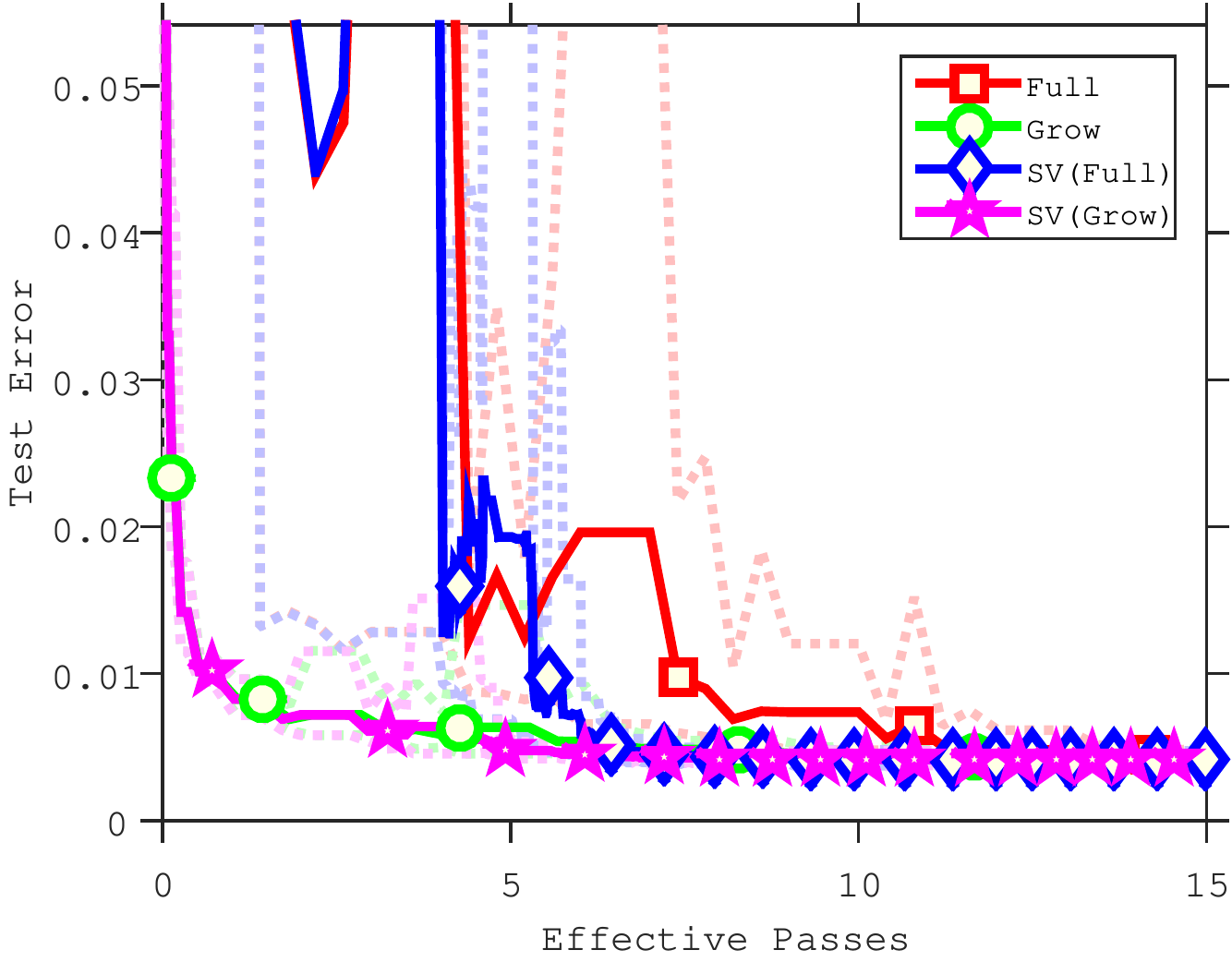}
\includegraphics[width=.32\textwidth]{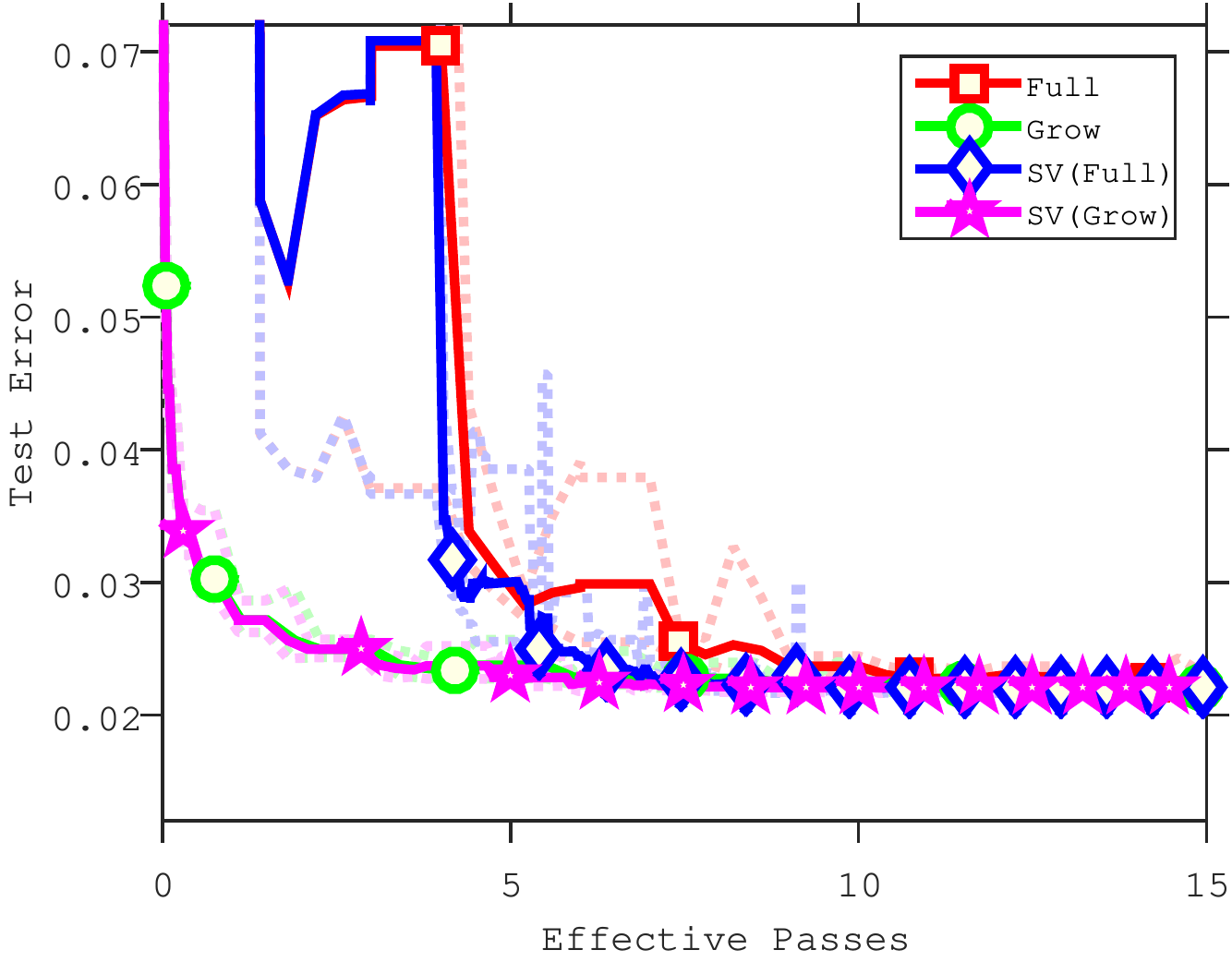}
\includegraphics[width=.32\textwidth]{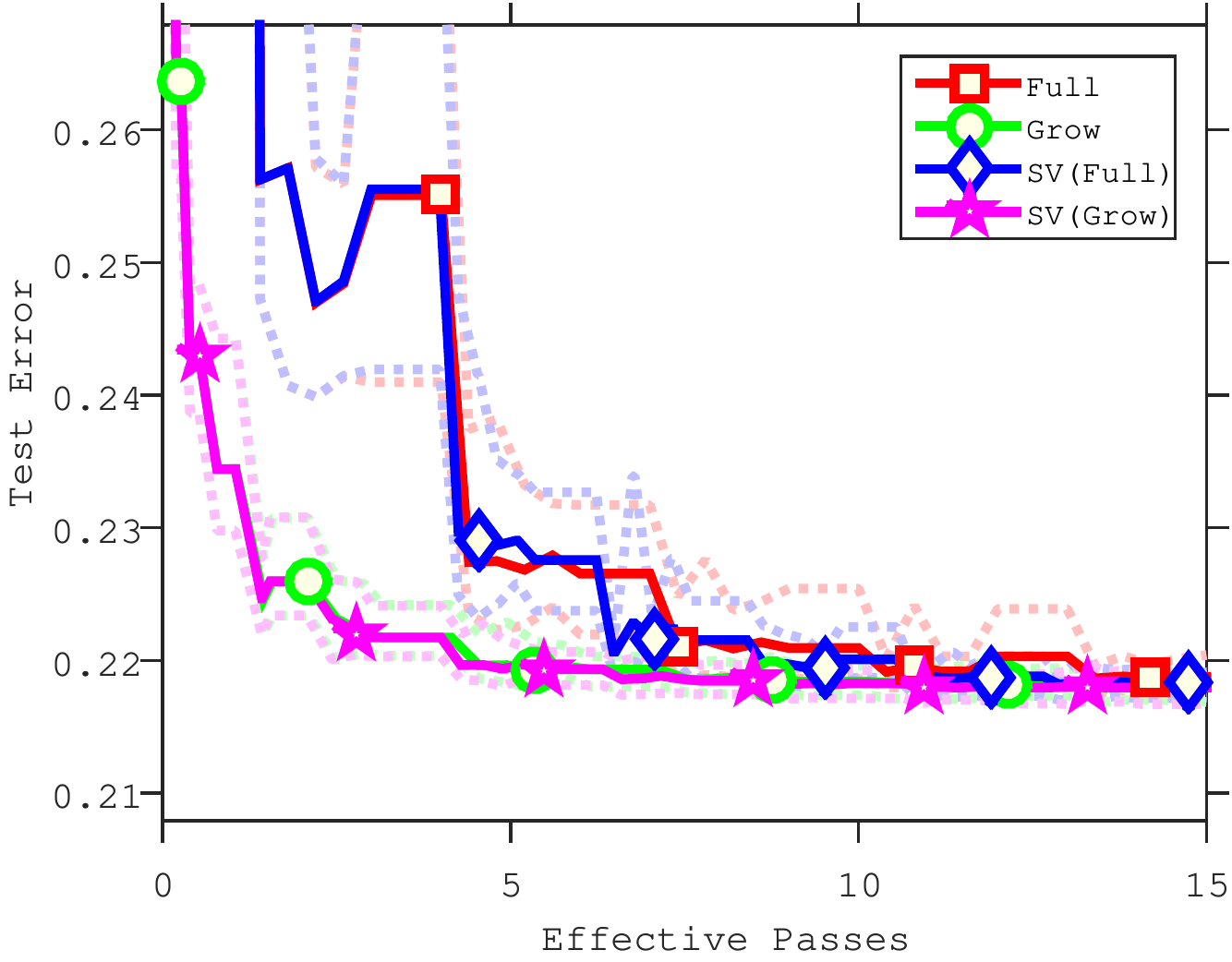}

\caption{Comparison of test error of SVM for different datasets. The top row gives results on the \emph{quantum} (left), \emph{protein} (center) and \emph{sido} (right) datasets. The middle row gives results on the \emph{rcv11} (left), \emph{covertype} (center) and \emph{news} (right) datasets.  The bottom row gives results on the \emph{spam} (left), \emph{rcv1Full} (center), and \emph{alpha} (right) datasets.}
%\label{fig:4}
\end{figure*}

\begin{figure*}
\includegraphics[width=.32\textwidth]{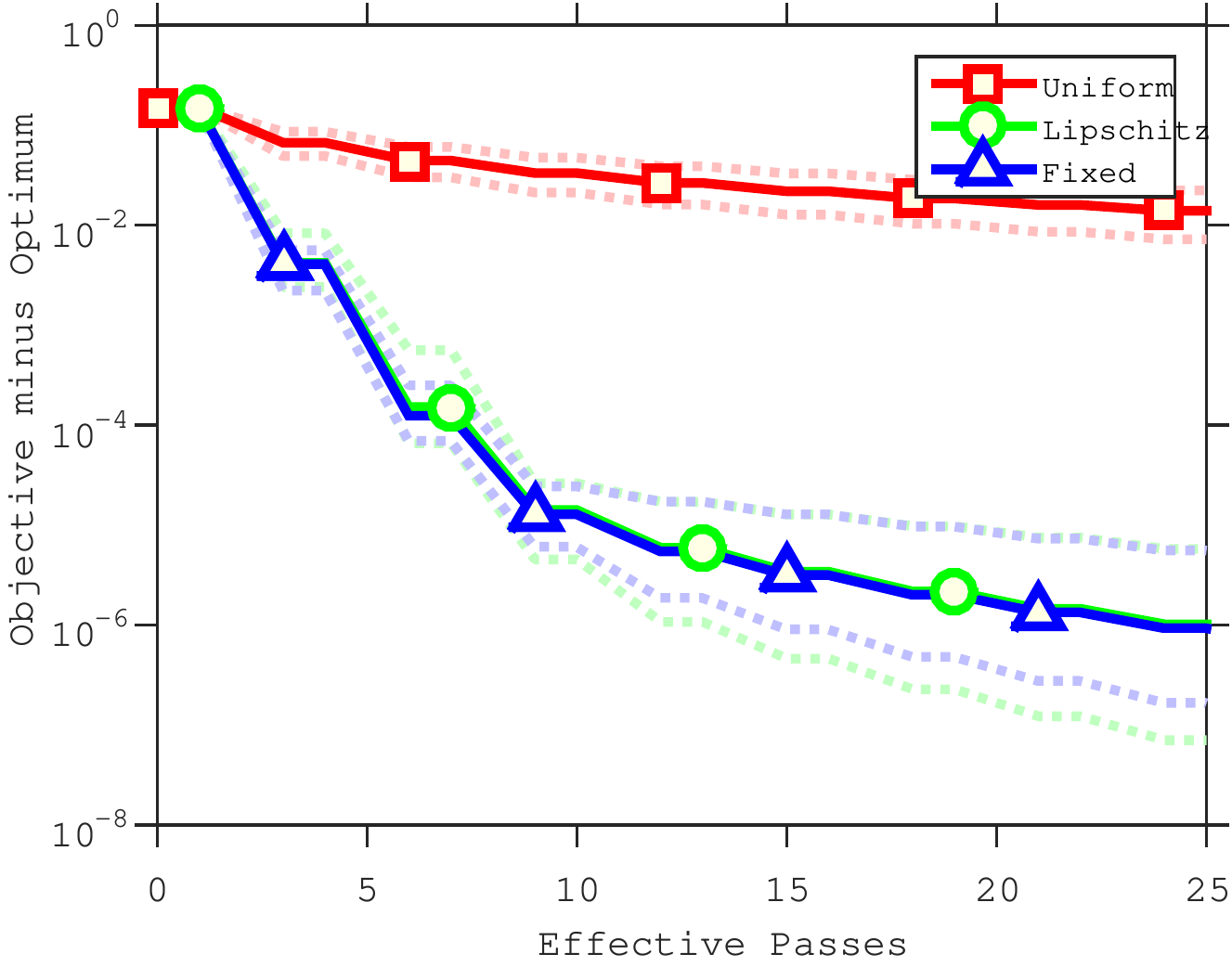}
\includegraphics[width=.32\textwidth]{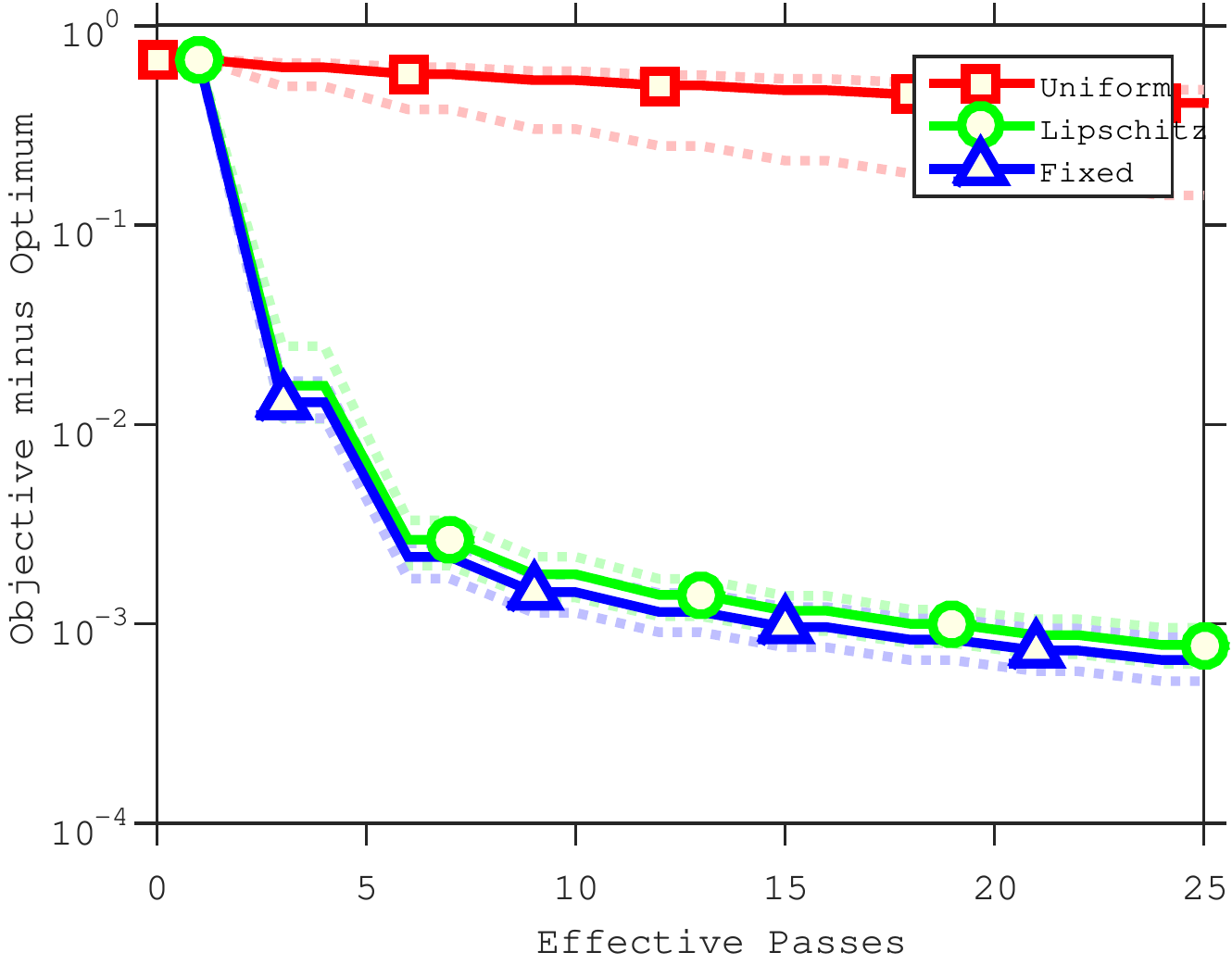}
\includegraphics[width=.32\textwidth]{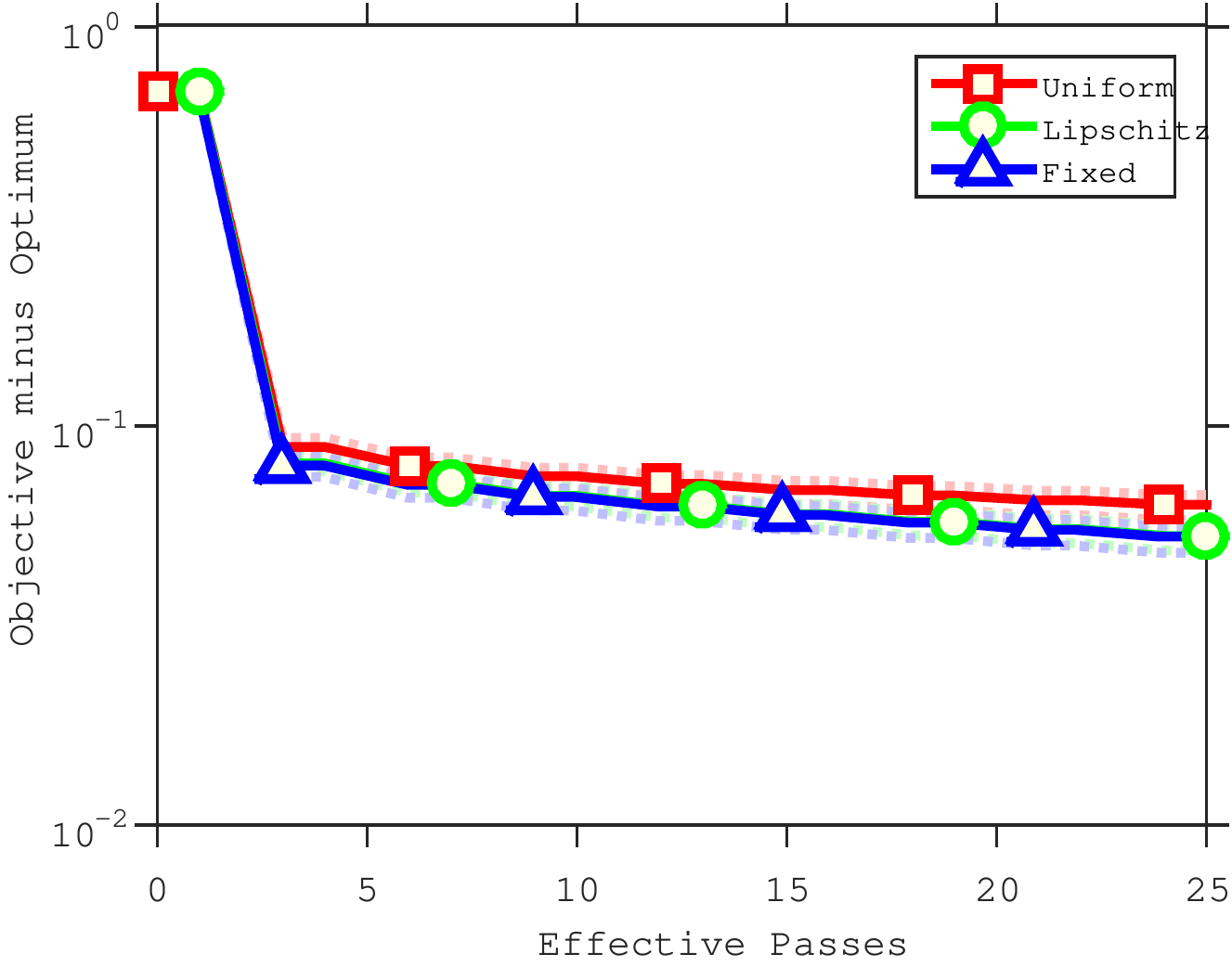}\\
\includegraphics[width=.32\textwidth]{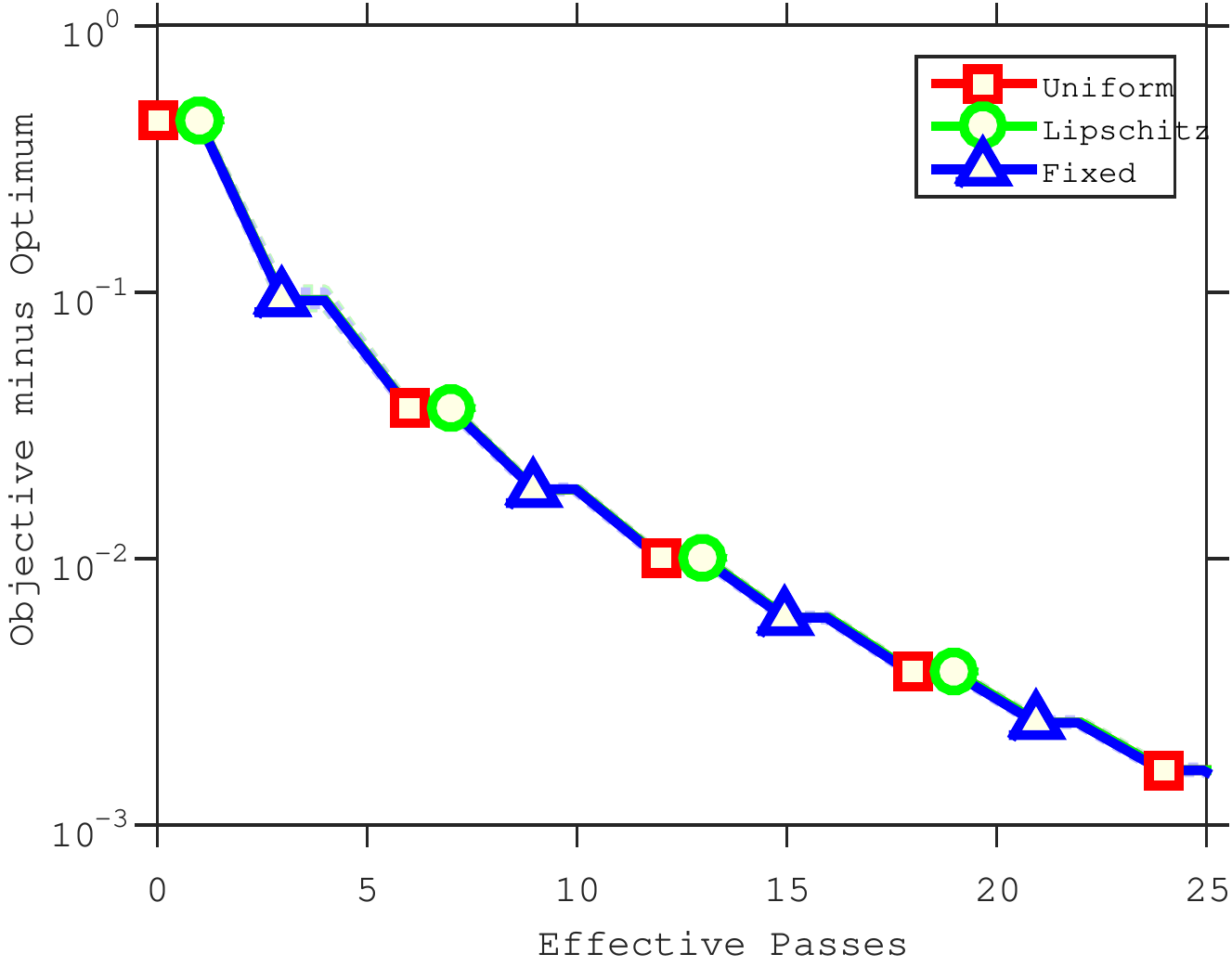}
\includegraphics[width=.32\textwidth]{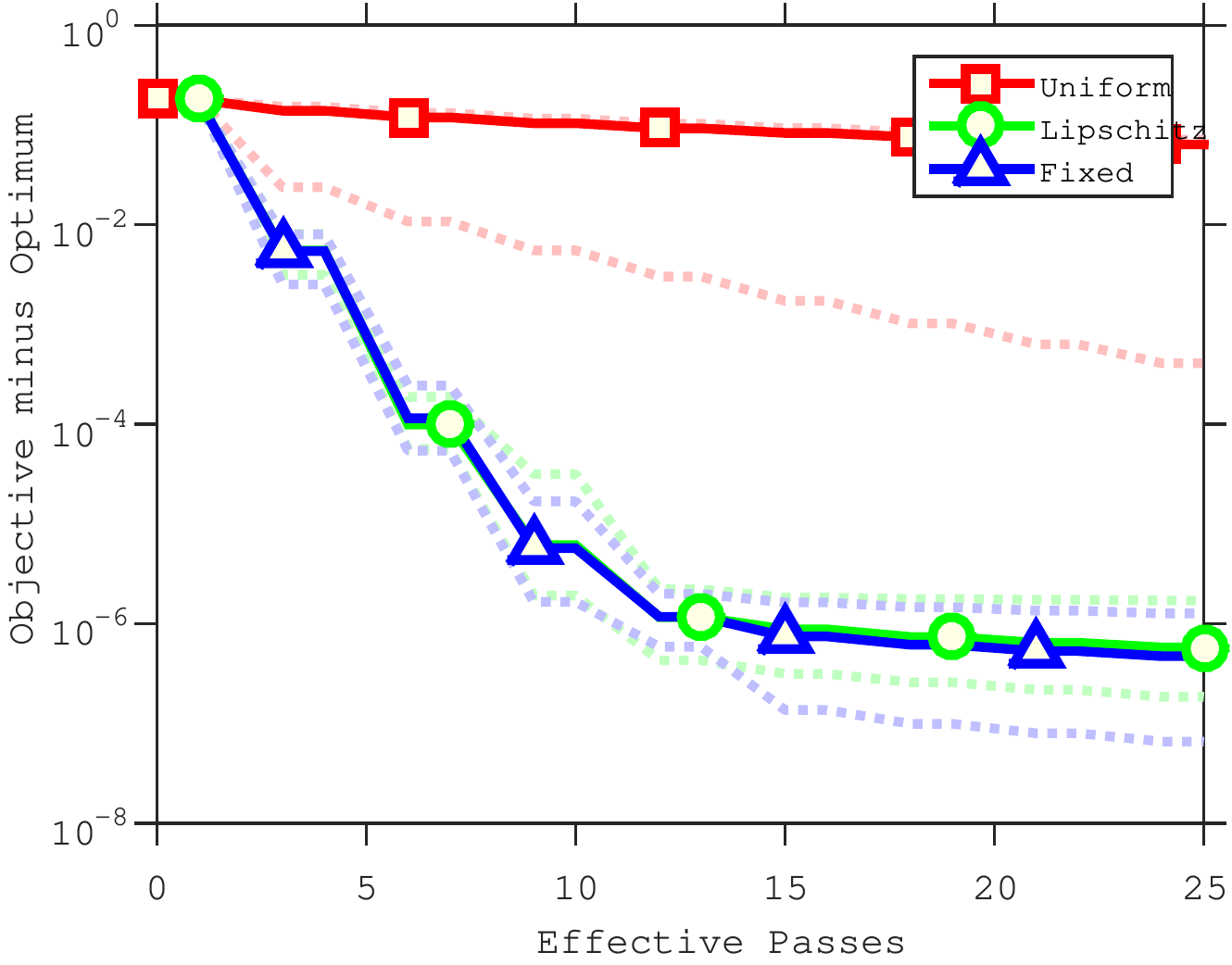}
\includegraphics[width=.32\textwidth]{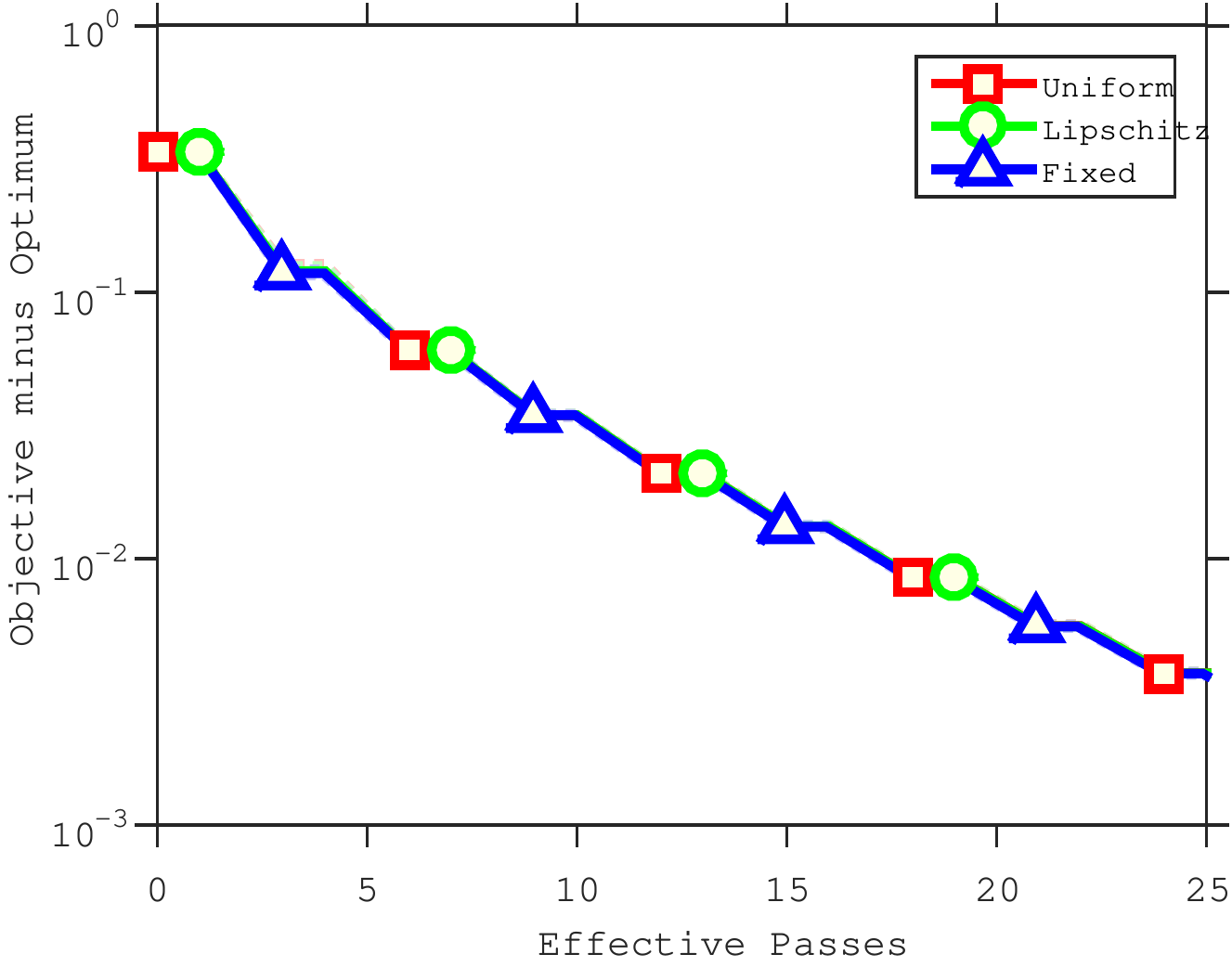}\\
\includegraphics[width=.32\textwidth]{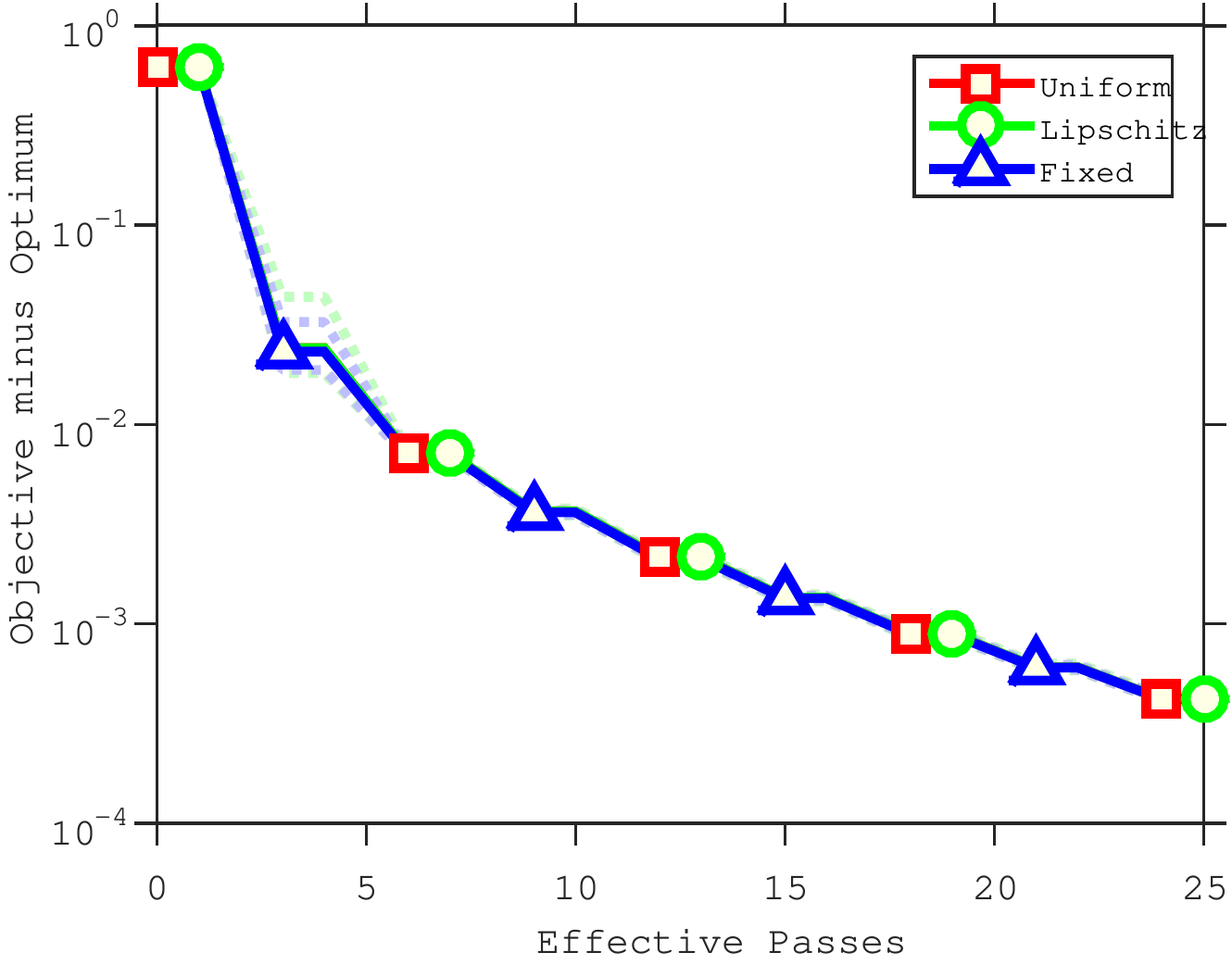}
\includegraphics[width=.32\textwidth]{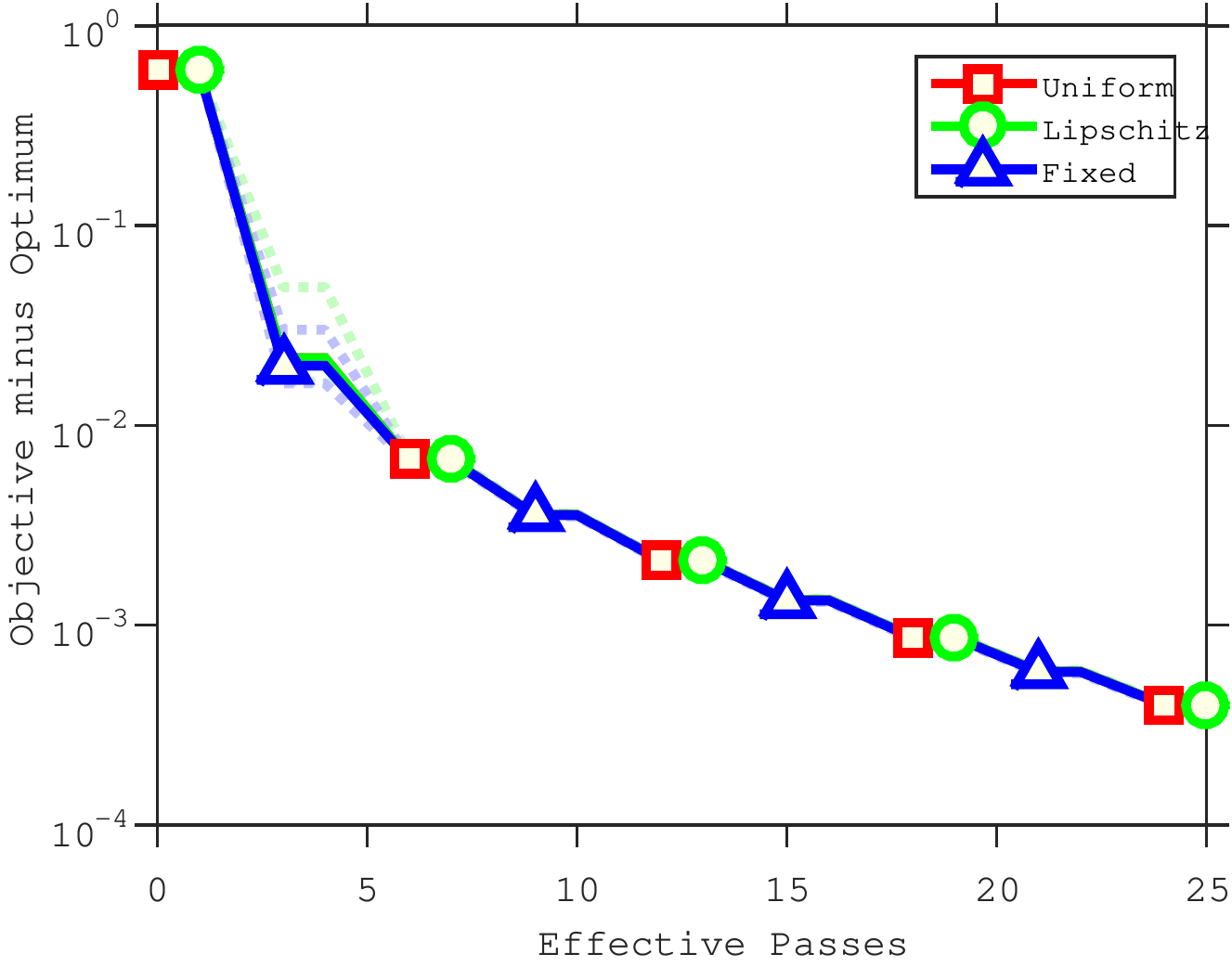}
\includegraphics[width=.32\textwidth]{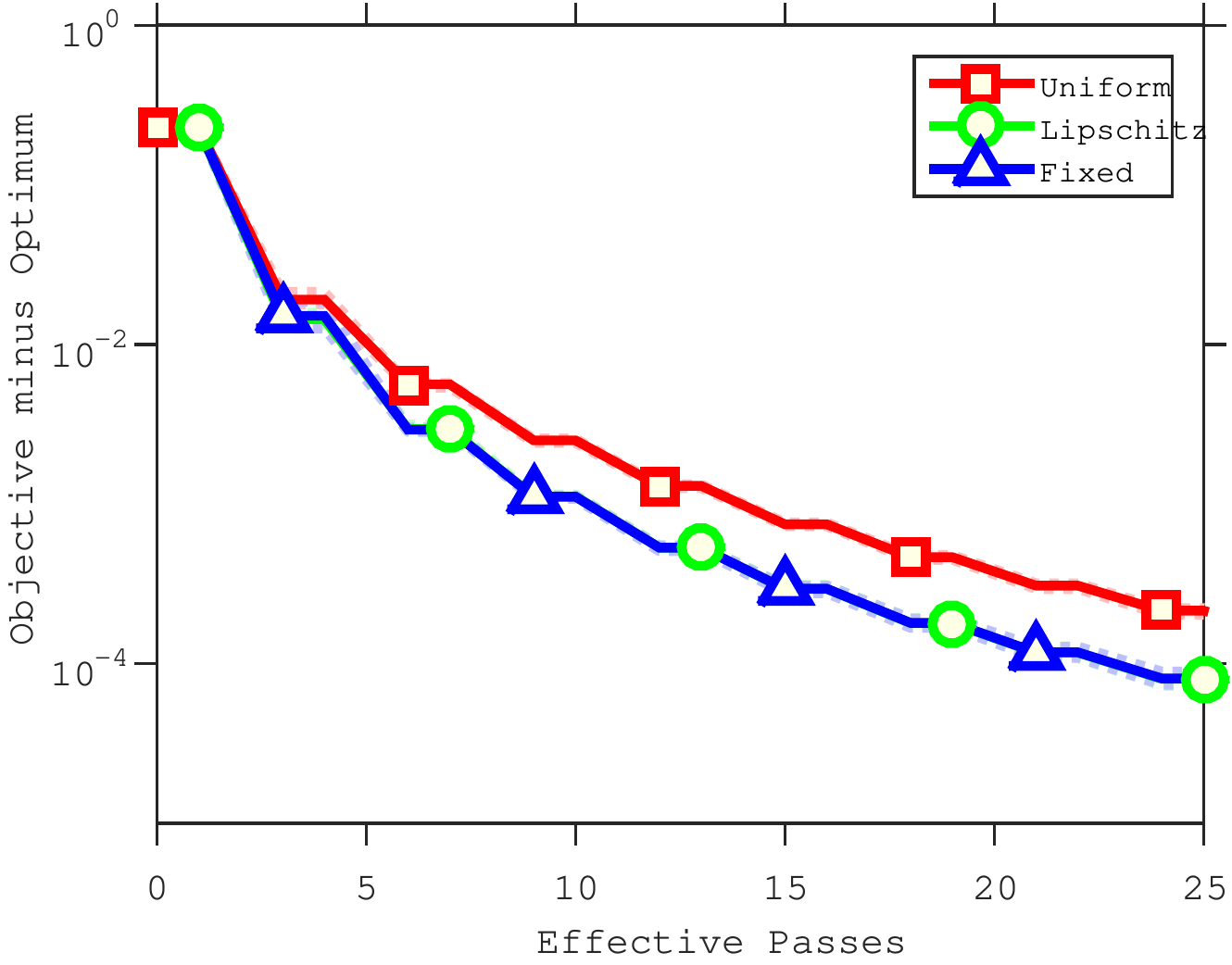}

\caption{Comparison of training objective of logistic regression with different mini-batch strategies. The top row gives results on the \emph{quantum} (left), \emph{protein} (center) and \emph{sido} (right) datasets. The middle row gives results on the \emph{rcv11} (left), \emph{covertype} (center) and \emph{news} (right) datasets.  The bottom row gives results on the \emph{spam} (left), \emph{rcv1Full} (center), and \emph{alpha} (right) datasets.}
%\label{fig:5}
\end{figure*}

\begin{figure*}
\includegraphics[width=.32\textwidth]{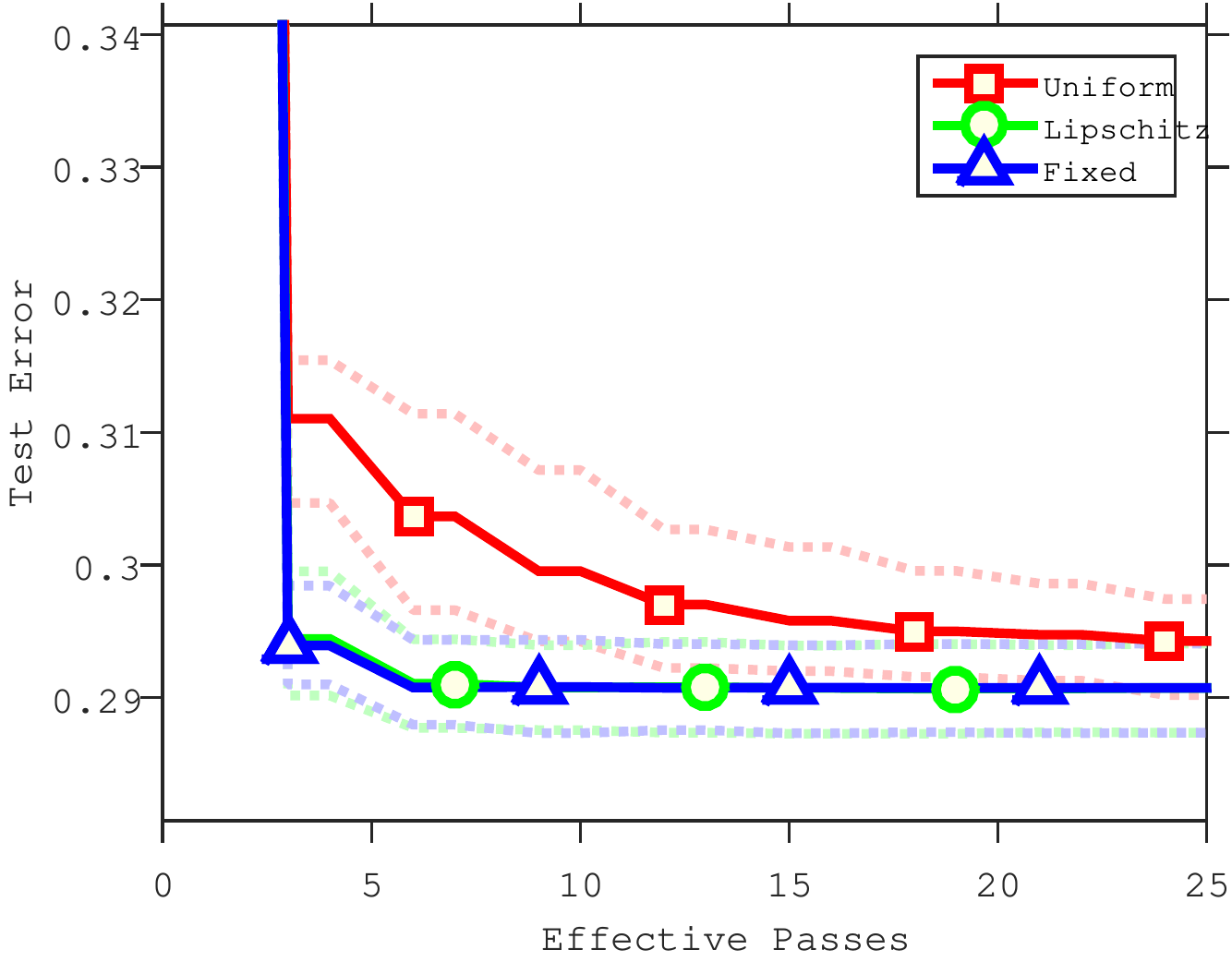}
\includegraphics[width=.32\textwidth]{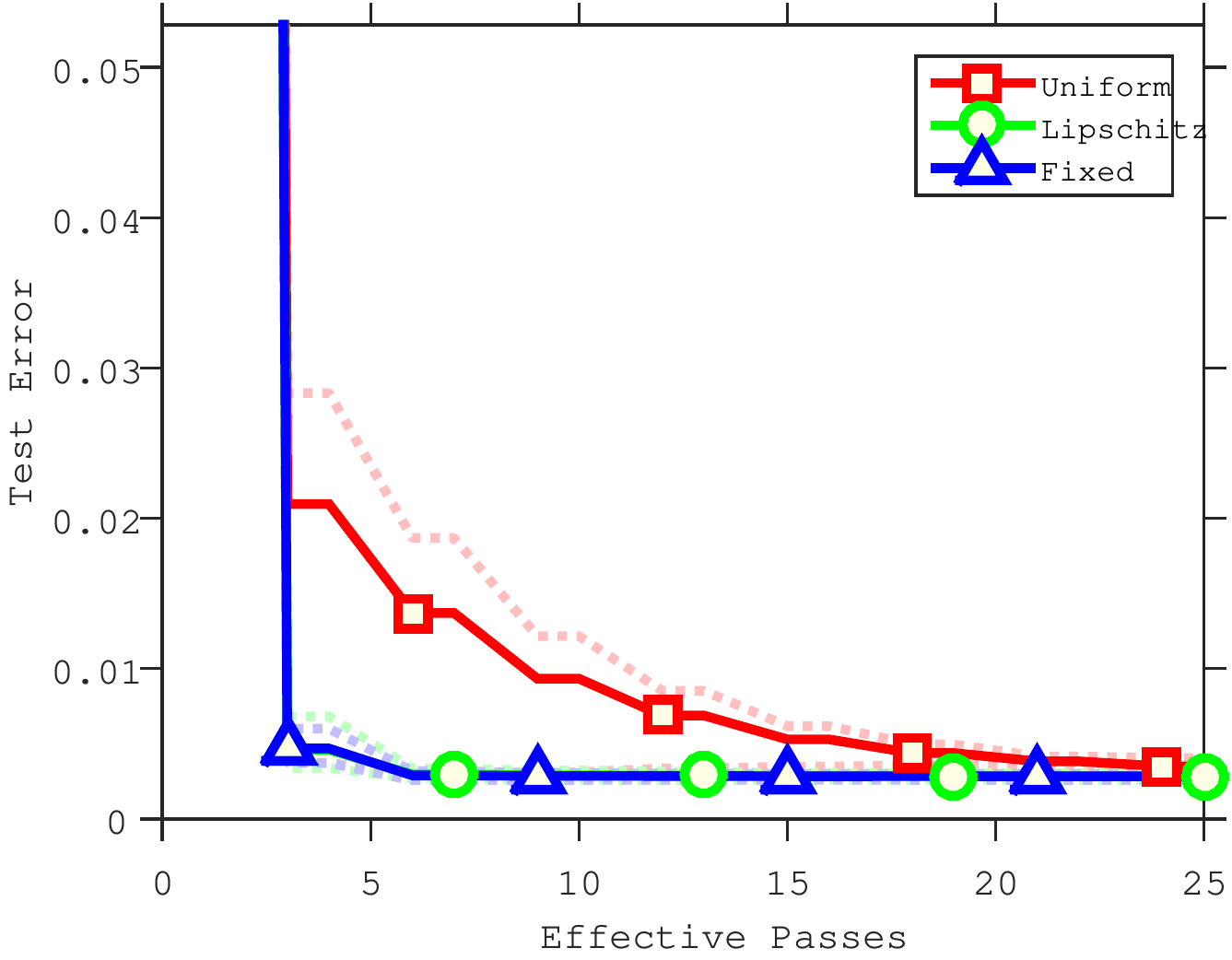}
\includegraphics[width=.32\textwidth]{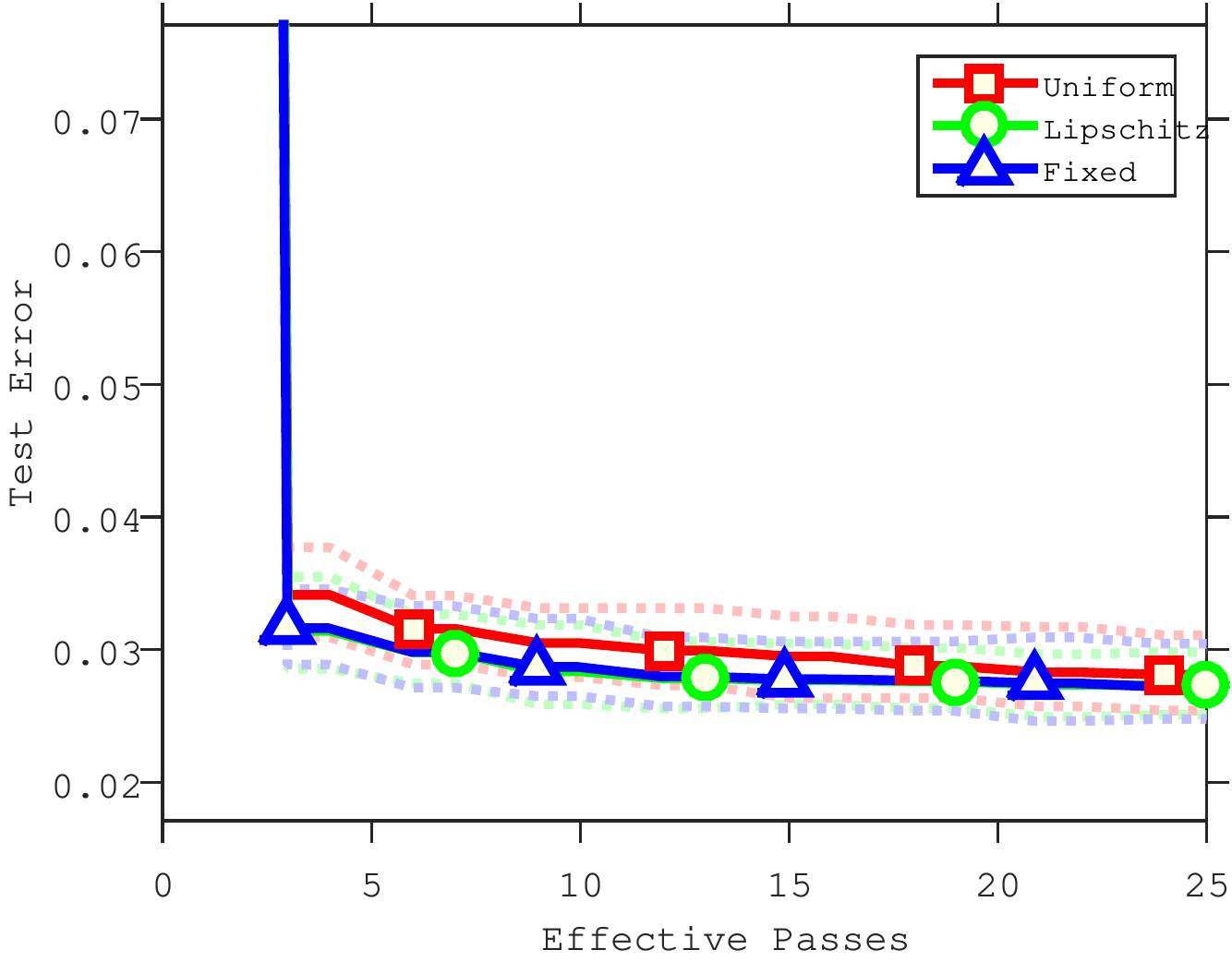}\\
\includegraphics[width=.32\textwidth]{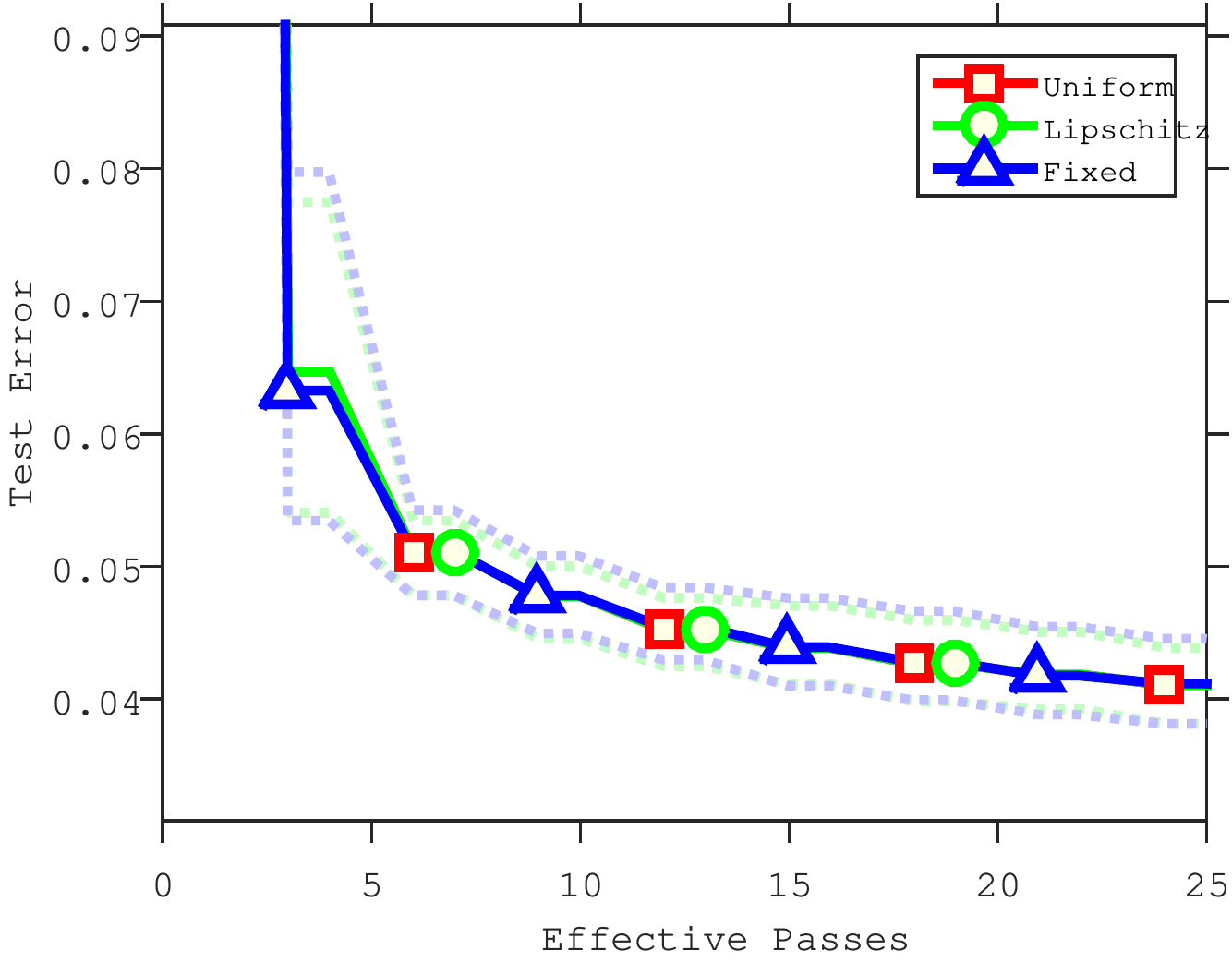}
\includegraphics[width=.32\textwidth]{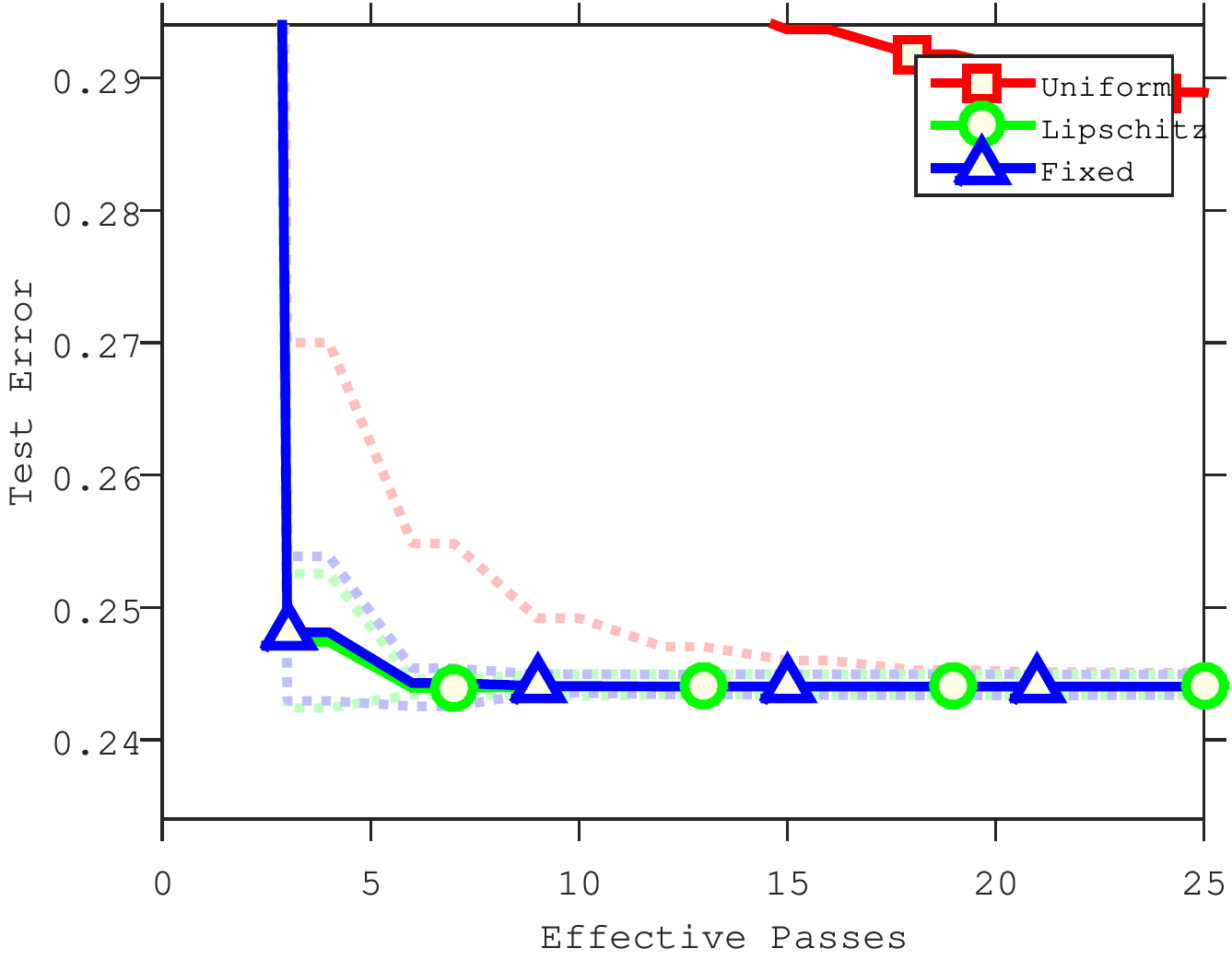}
\includegraphics[width=.32\textwidth]{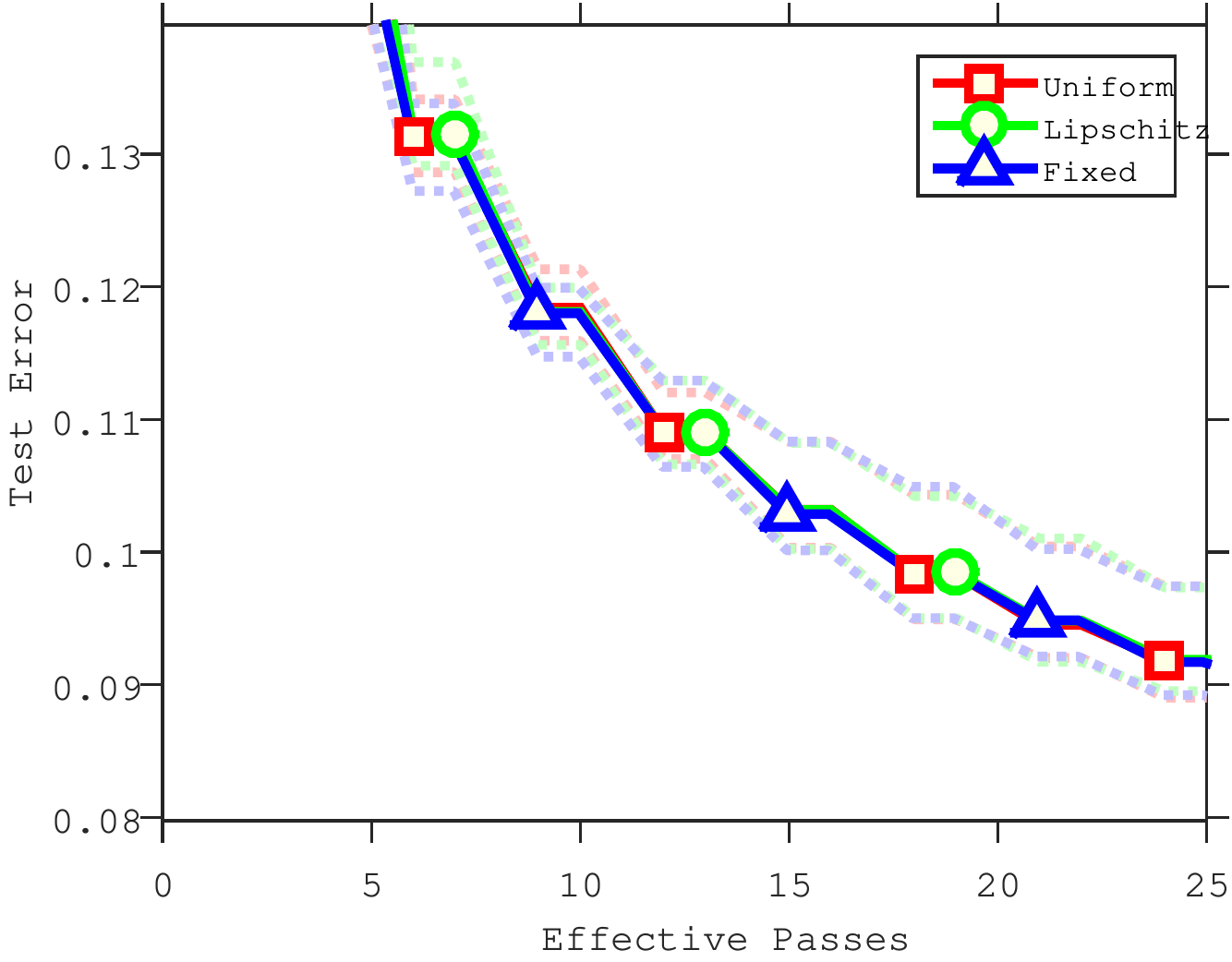}\\
\includegraphics[width=.32\textwidth]{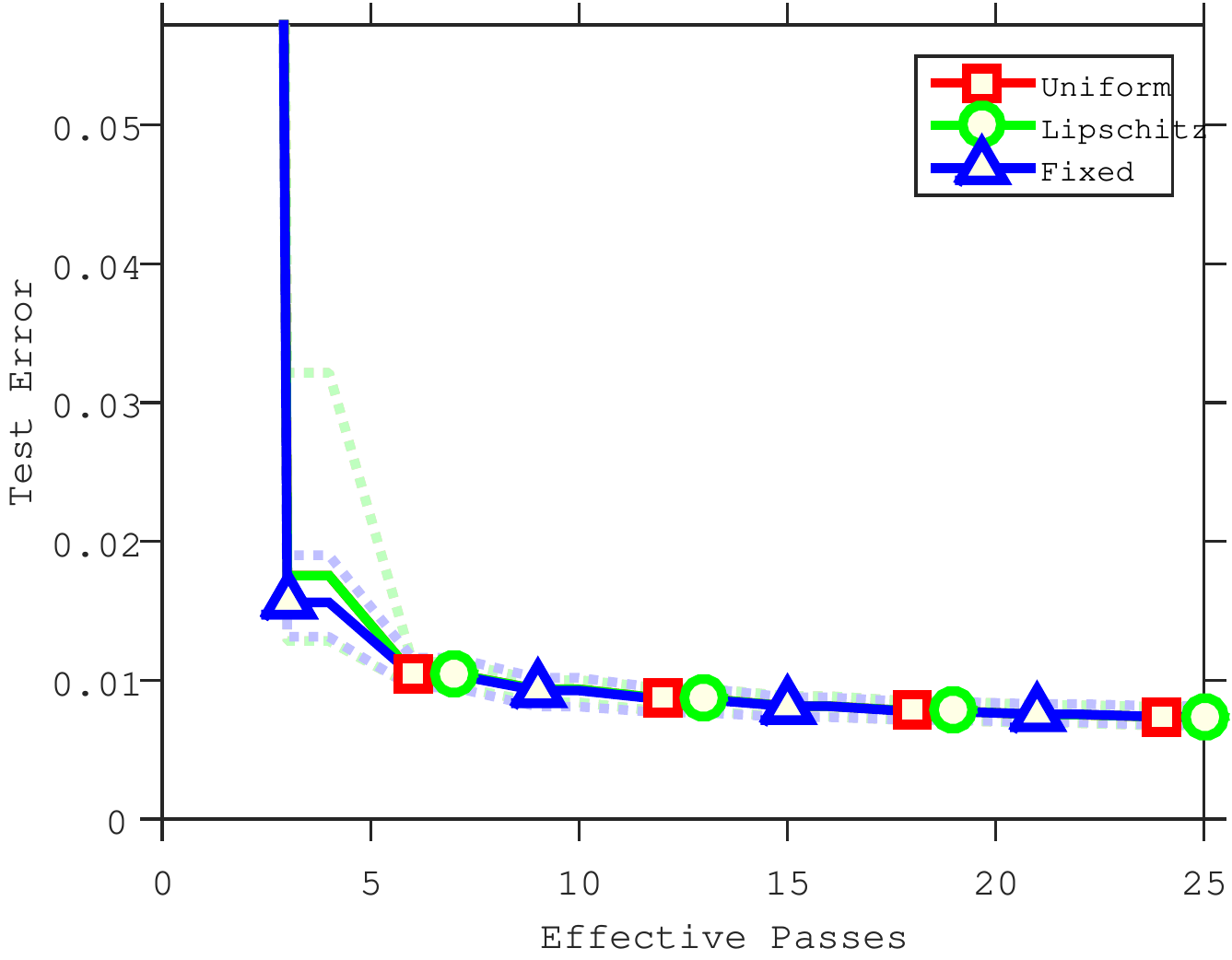}
\includegraphics[width=.32\textwidth]{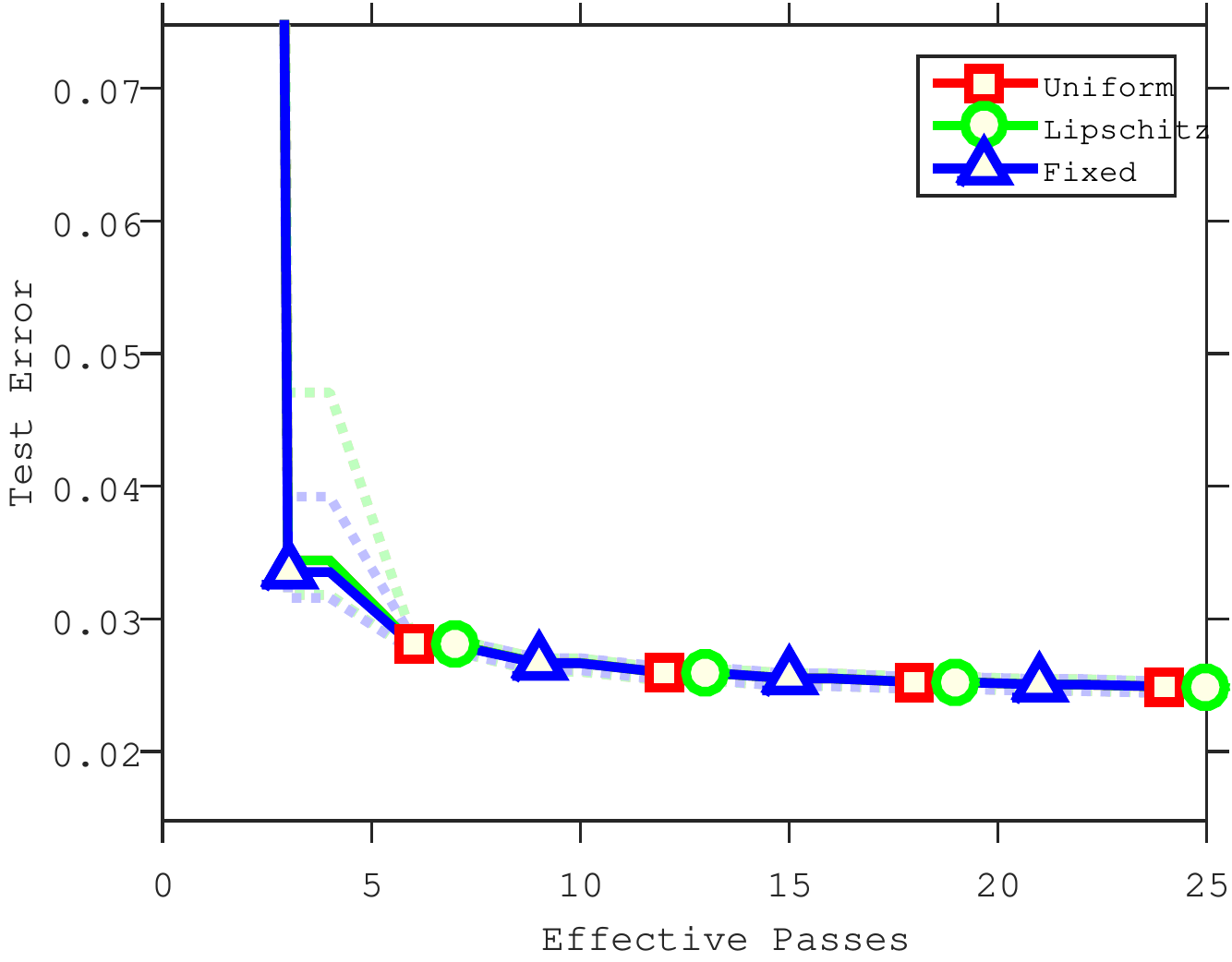}
\includegraphics[width=.32\textwidth]{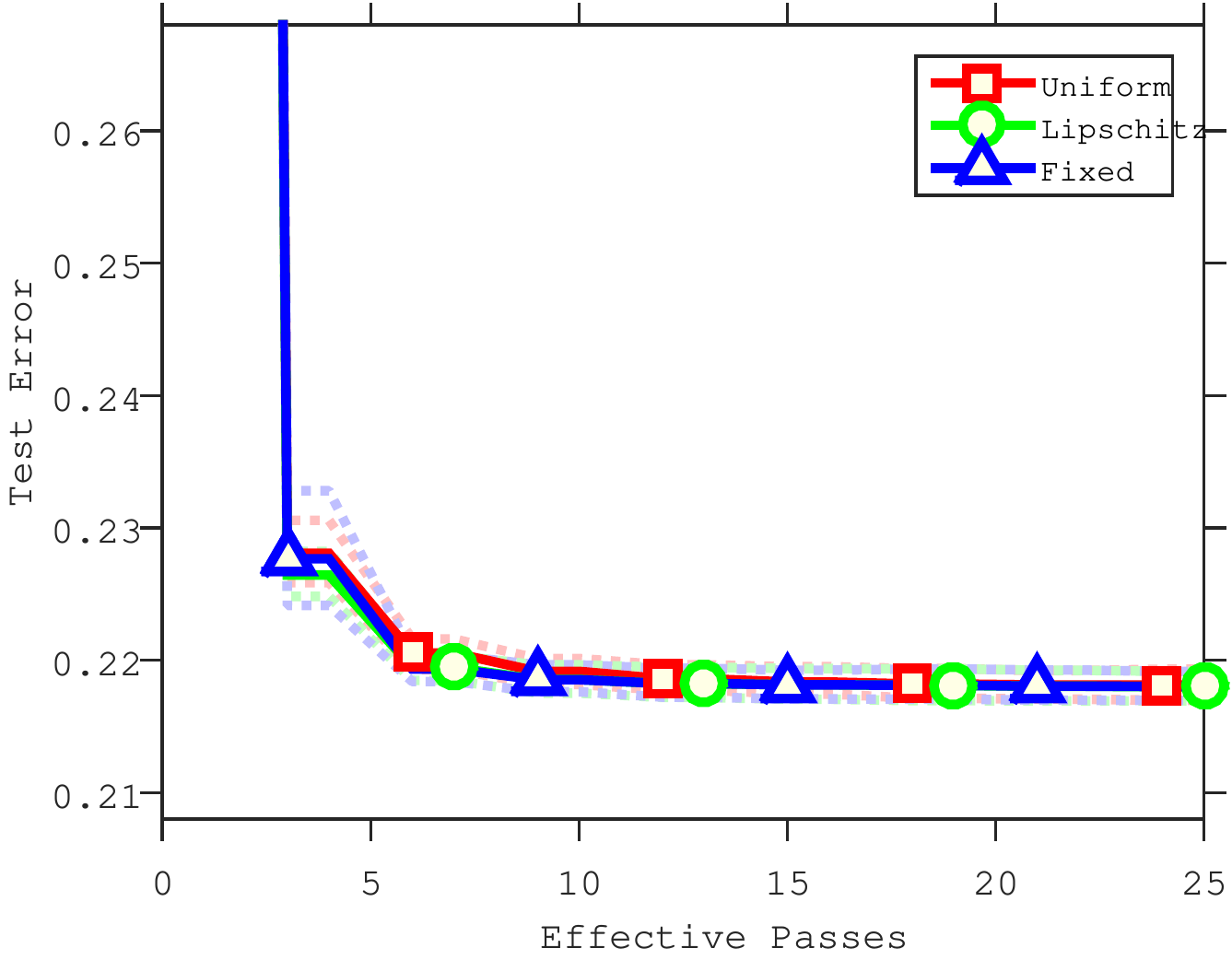}

\caption{Comparison of test error of  logistic regression with different mini-batch strategies.  The top row gives results on the \emph{quantum} (left), \emph{protein} (center) and \emph{sido} (right) datasets. The middle row gives results on the \emph{rcv11} (left), \emph{covertype} (center) and \emph{news} (right) datasets.  The bottom row gives results on the \emph{spam} (left), \emph{rcv1Full} (center), and \emph{alpha} (right) datasets.}
%\label{fig:6}
\end{figure*}

\begin{figure*}
\includegraphics[width=.32\textwidth]{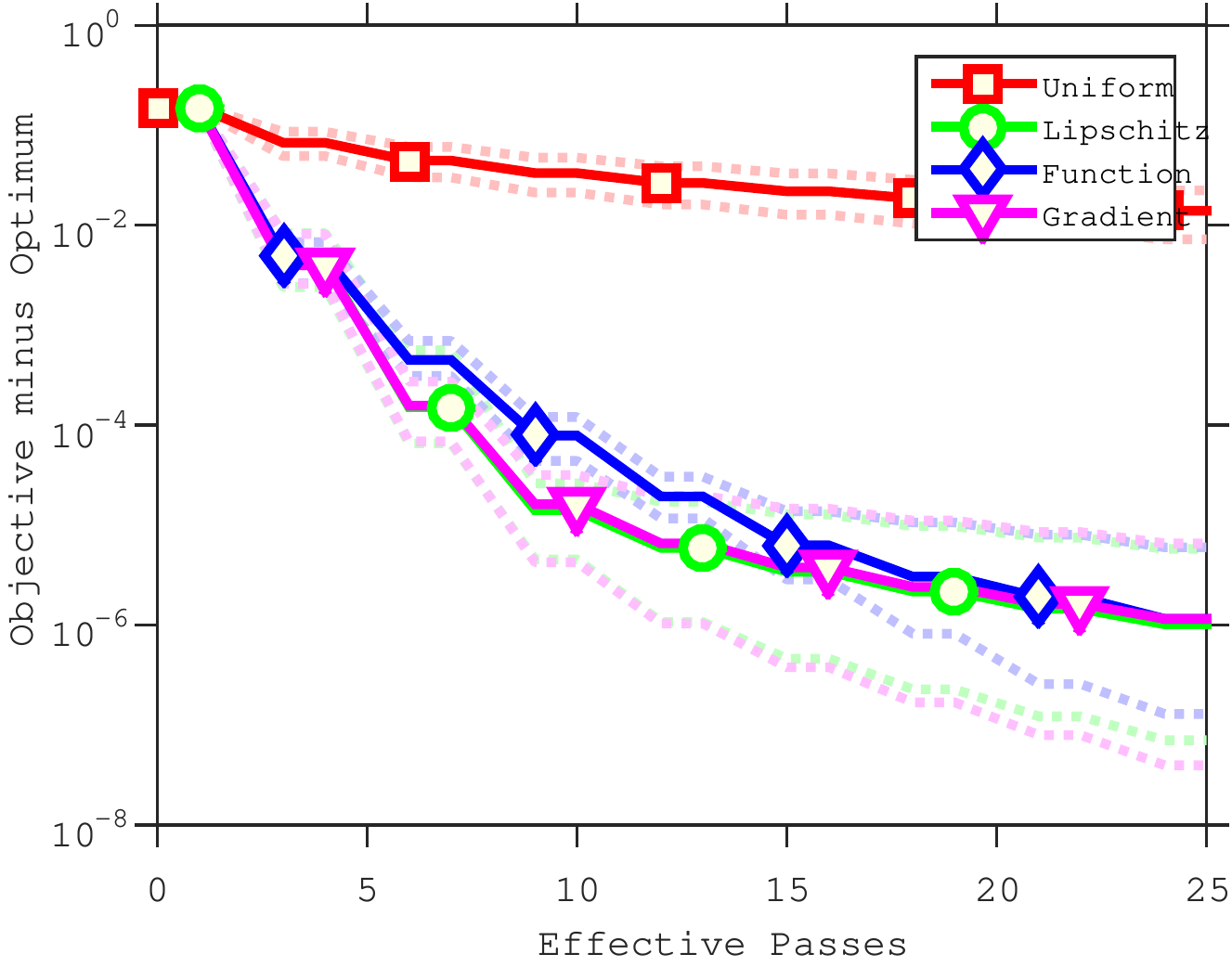}
\includegraphics[width=.32\textwidth]{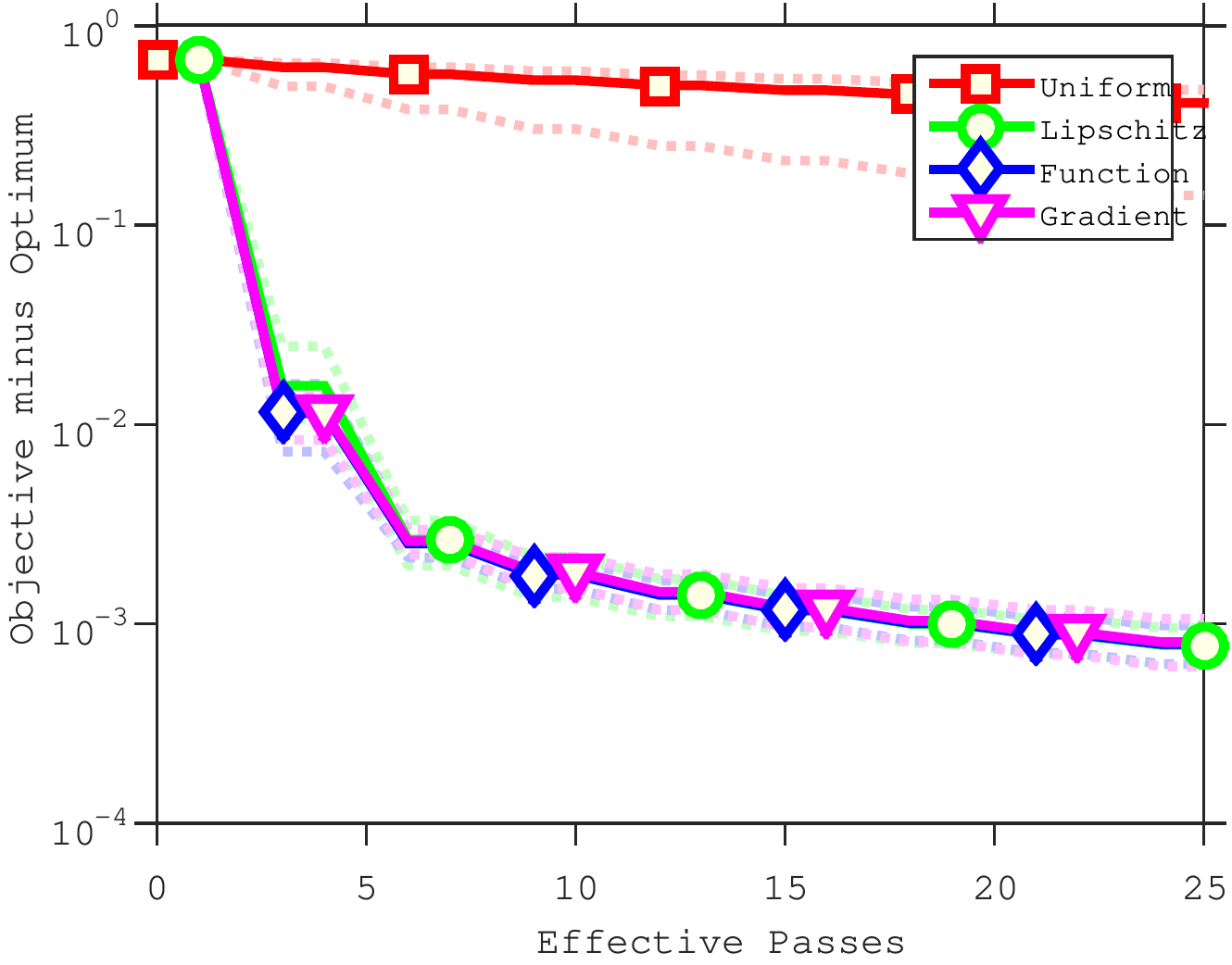}
\includegraphics[width=.32\textwidth]{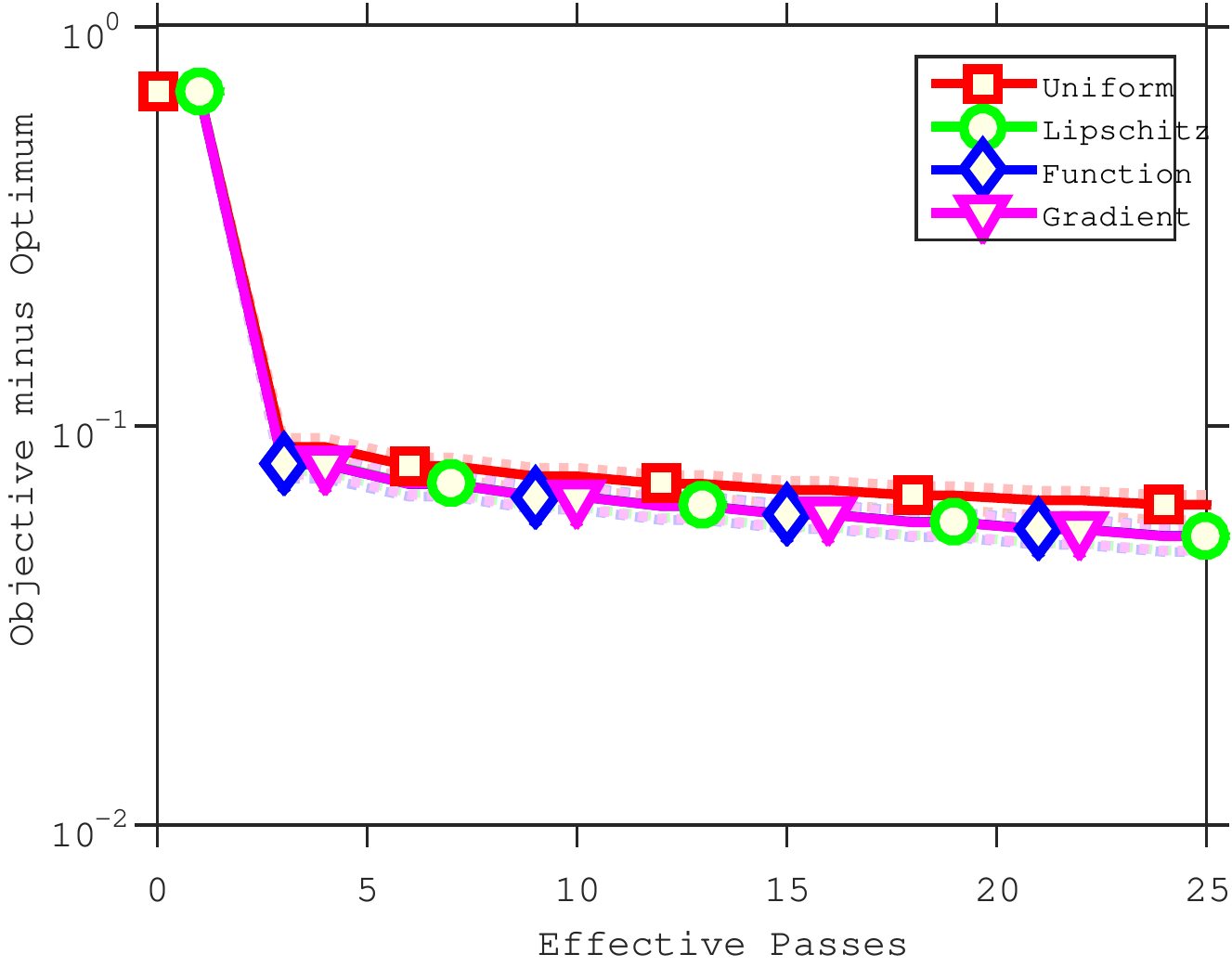}\\
\includegraphics[width=.32\textwidth]{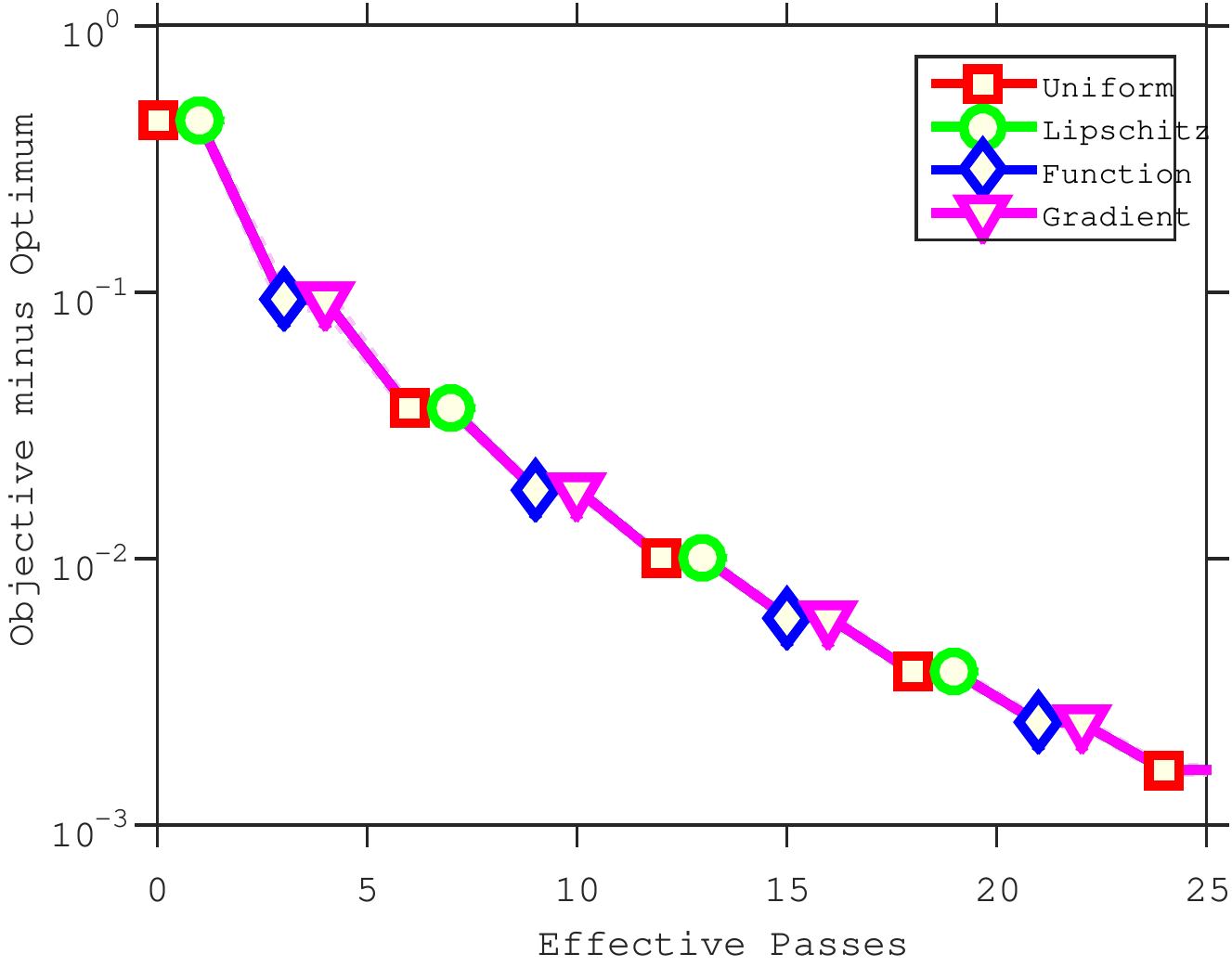}
\includegraphics[width=.32\textwidth]{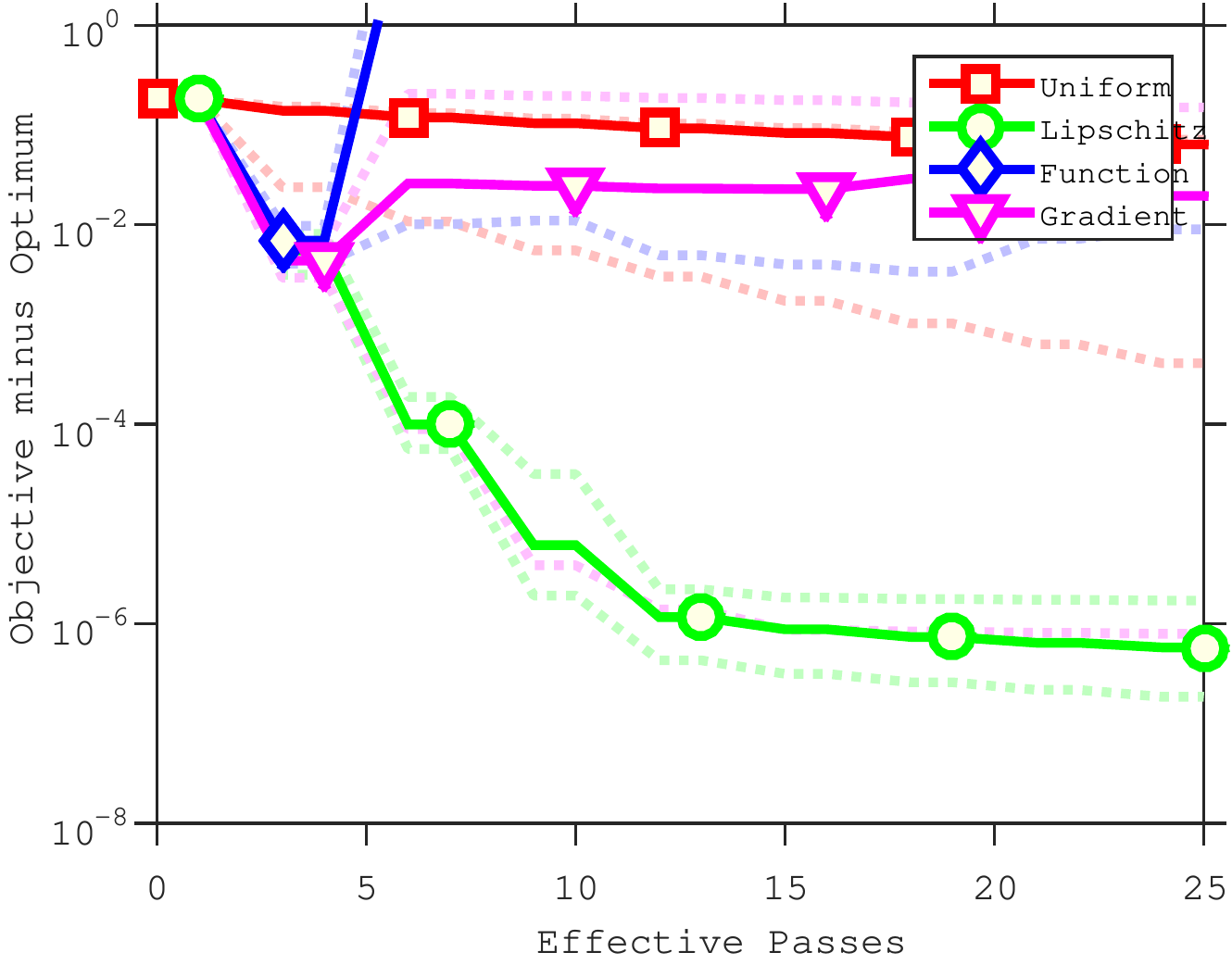}
\includegraphics[width=.32\textwidth]{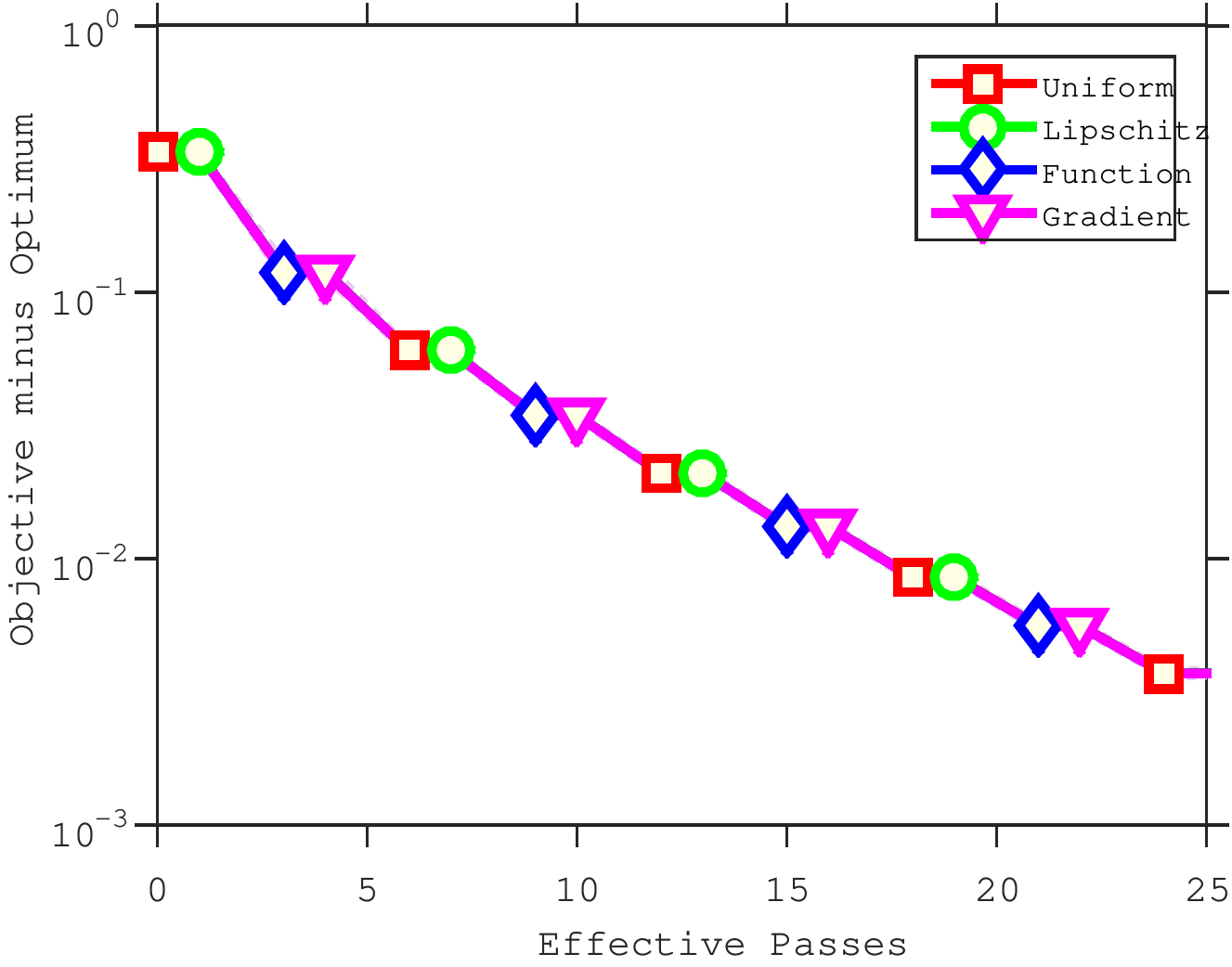}\\
\includegraphics[width=.32\textwidth]{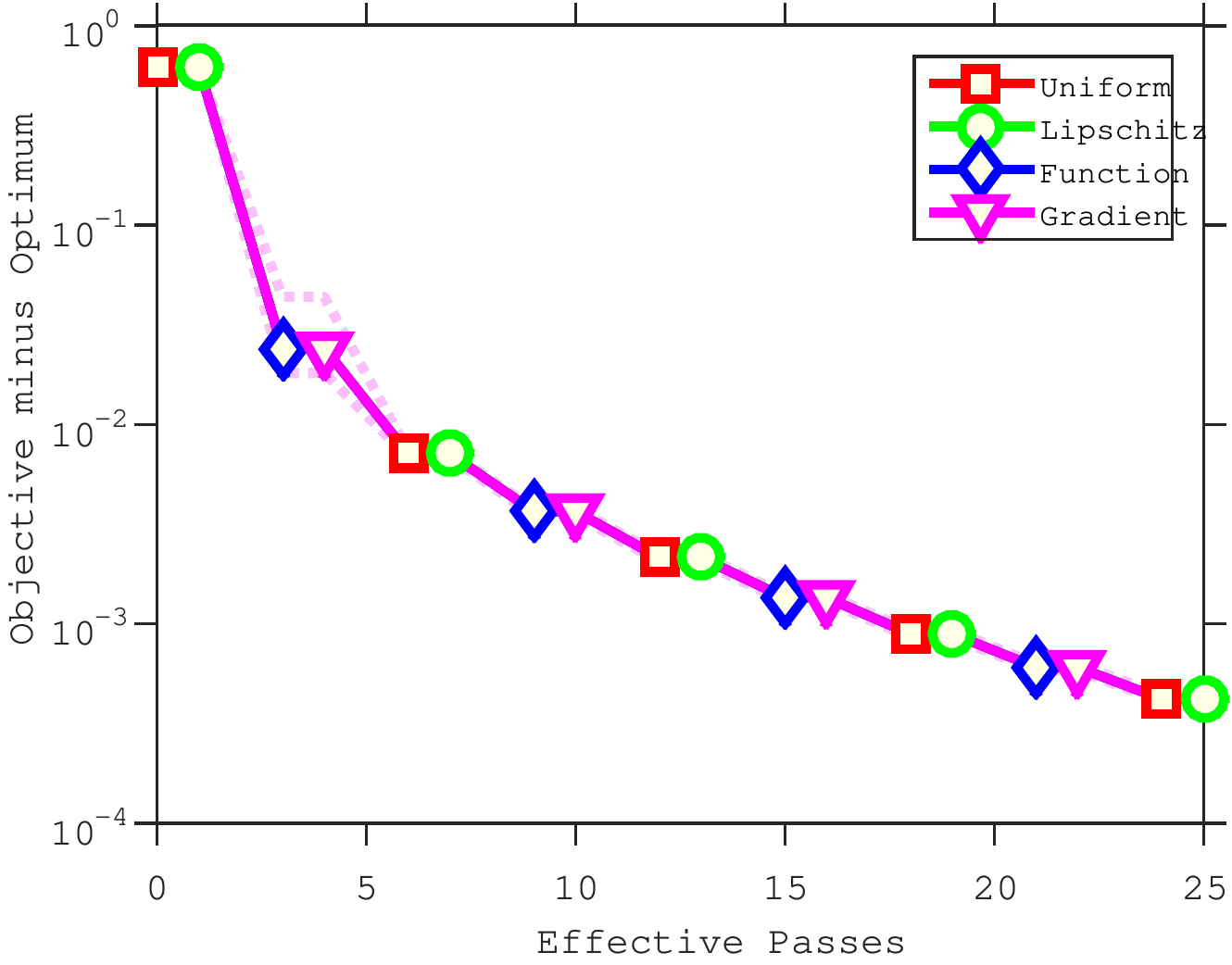}
\includegraphics[width=.32\textwidth]{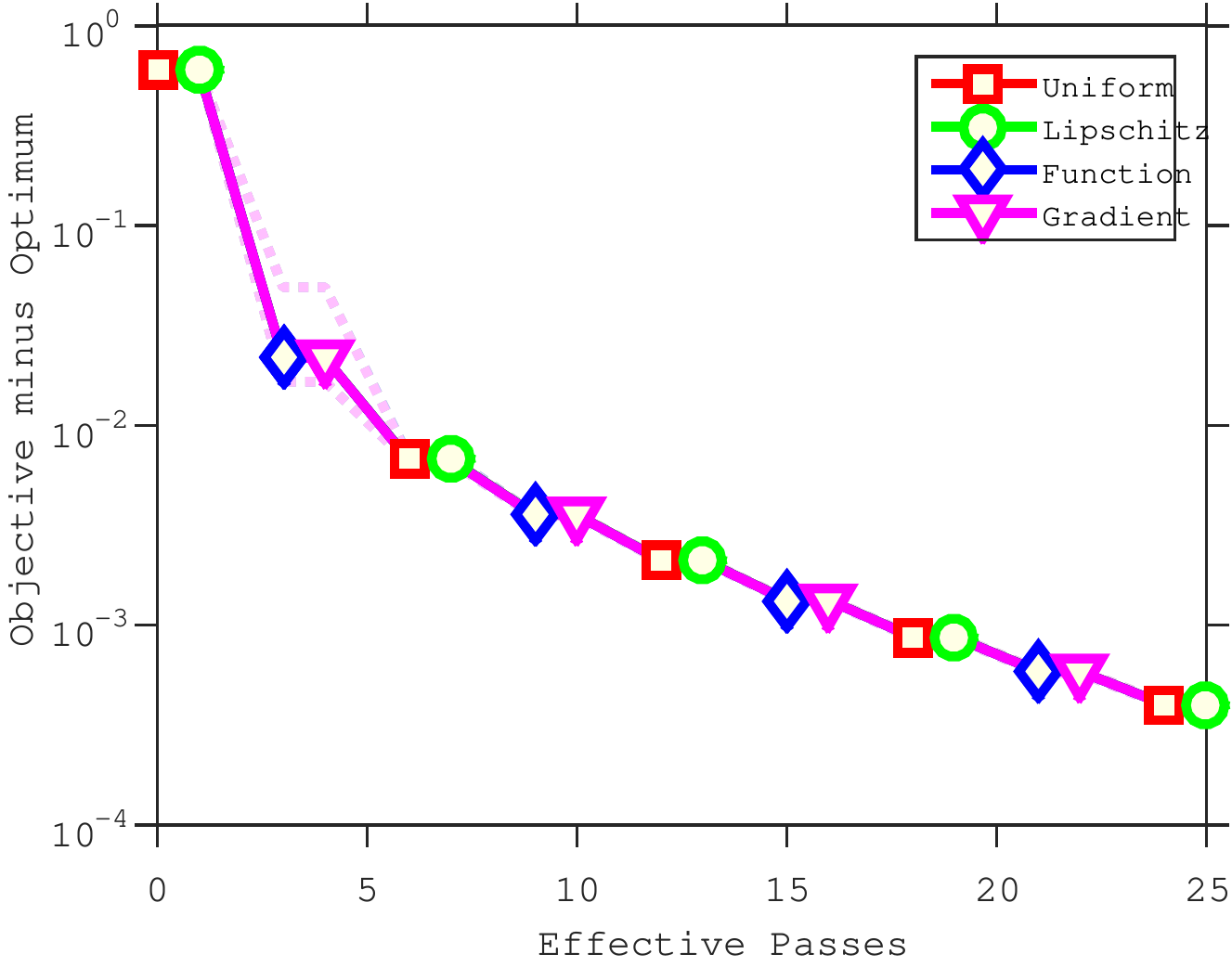}
\includegraphics[width=.32\textwidth]{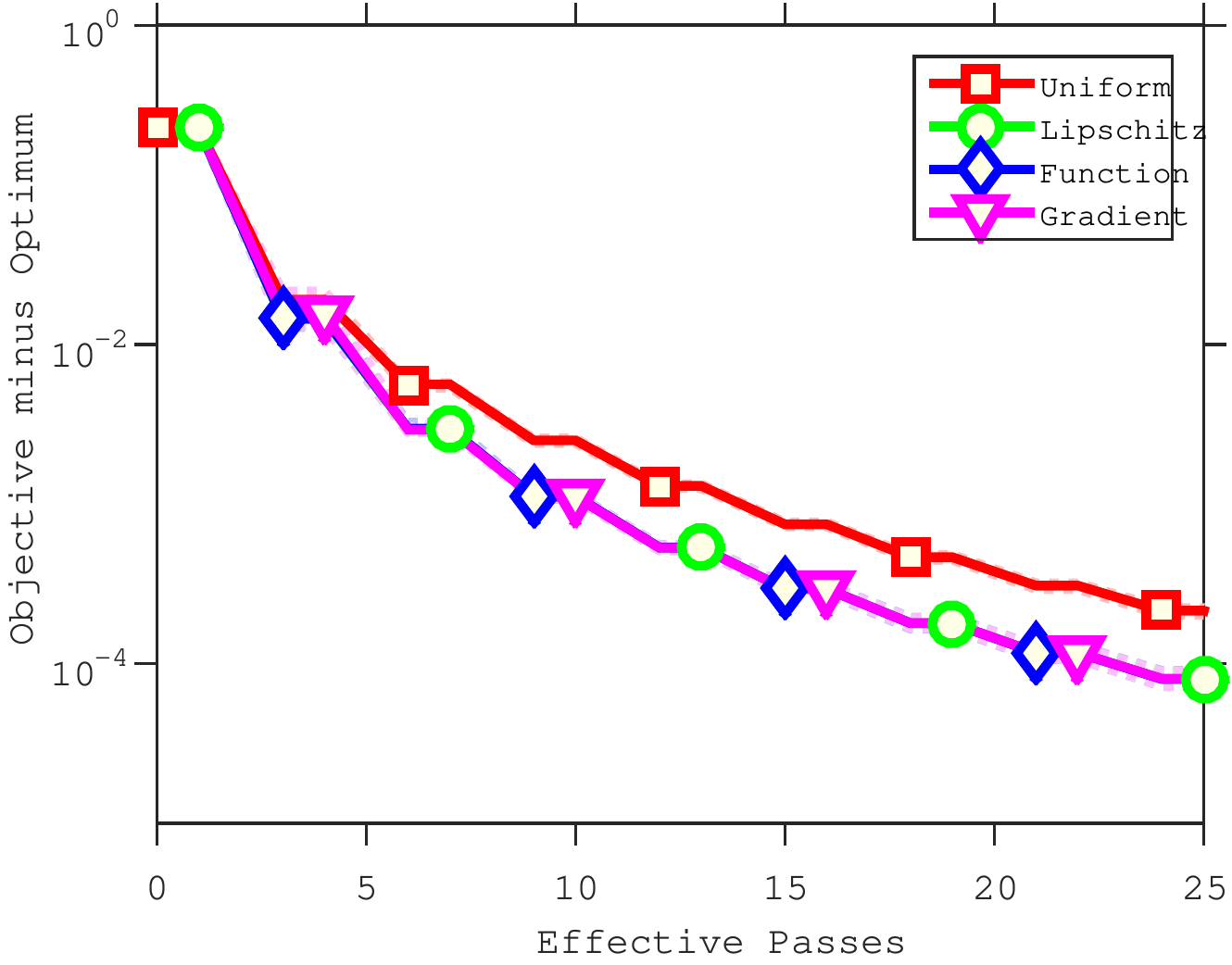}

\caption{Comparison of training objective of logistic regression with different mini-batch strategies. The top row gives results on the \emph{quantum} (left), \emph{protein} (center) and \emph{sido} (right) datasets. The middle row gives results on the \emph{rcv11} (left), \emph{covertype} (center) and \emph{news} (right) datasets.  The bottom row gives results on the \emph{spam} (left), \emph{rcv1Full} (center), and \emph{alpha} (right) datasets.}
%\label{fig:5}
\end{figure*}

\begin{figure*}
\includegraphics[width=.32\textwidth]{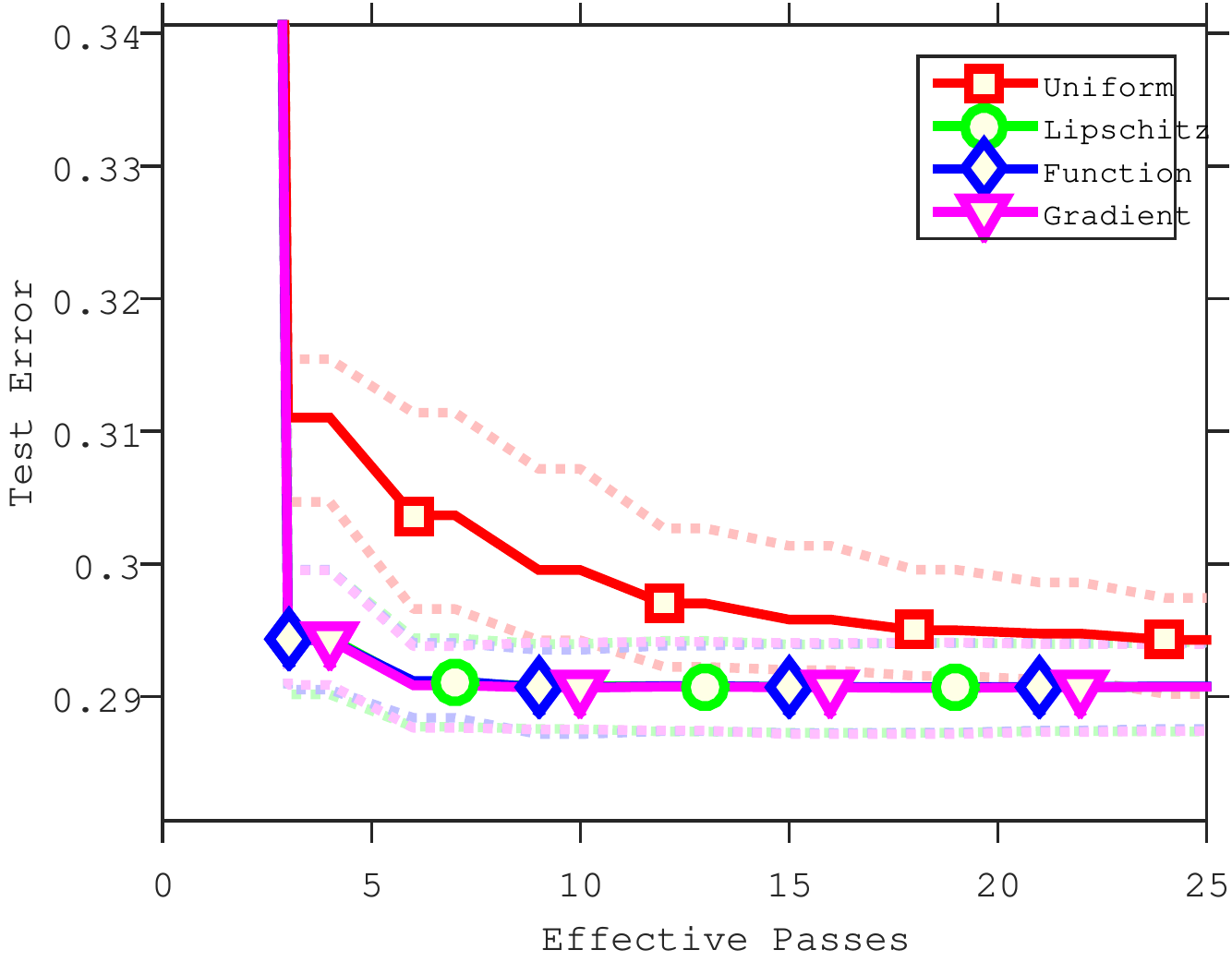}
\includegraphics[width=.32\textwidth]{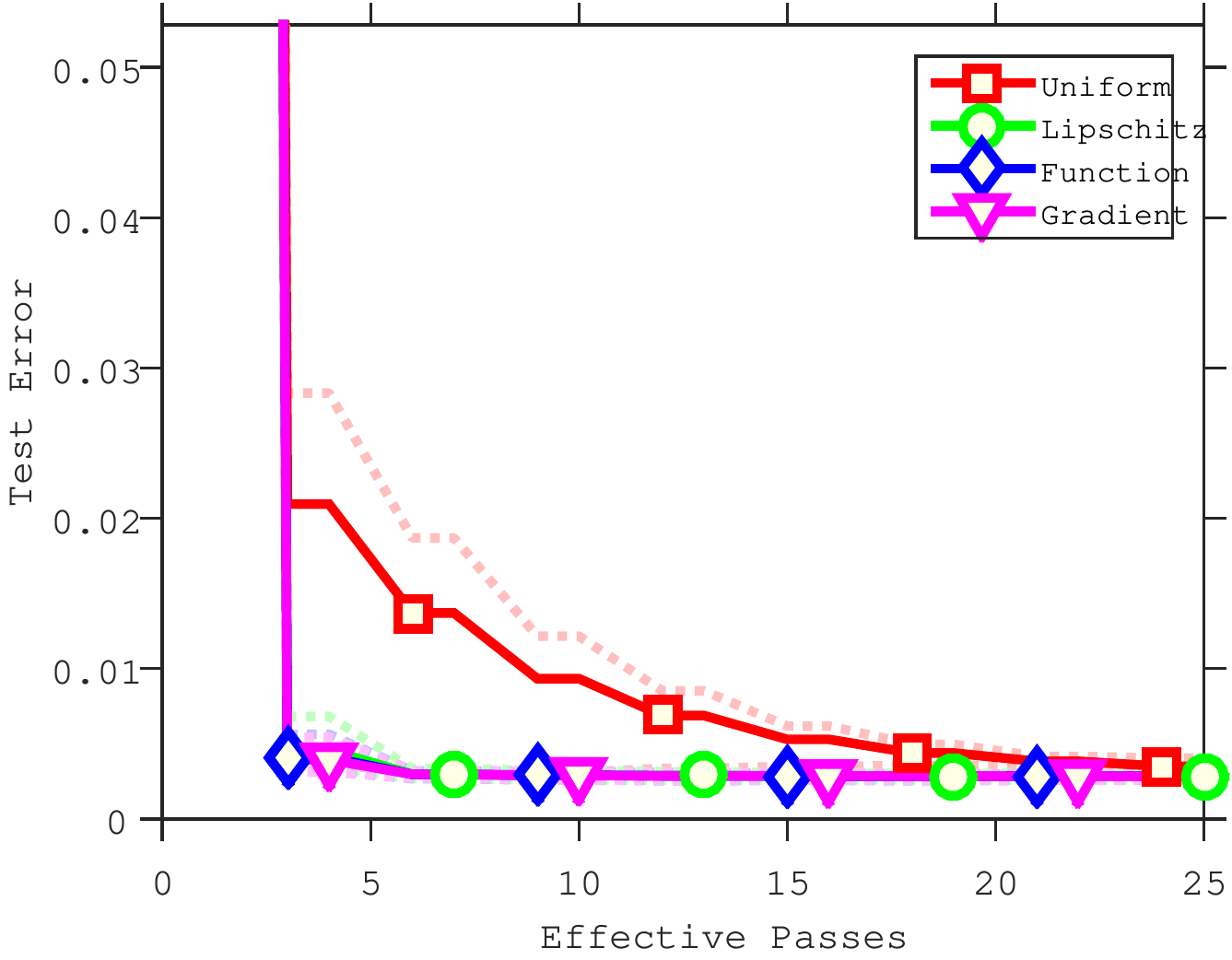}
\includegraphics[width=.32\textwidth]{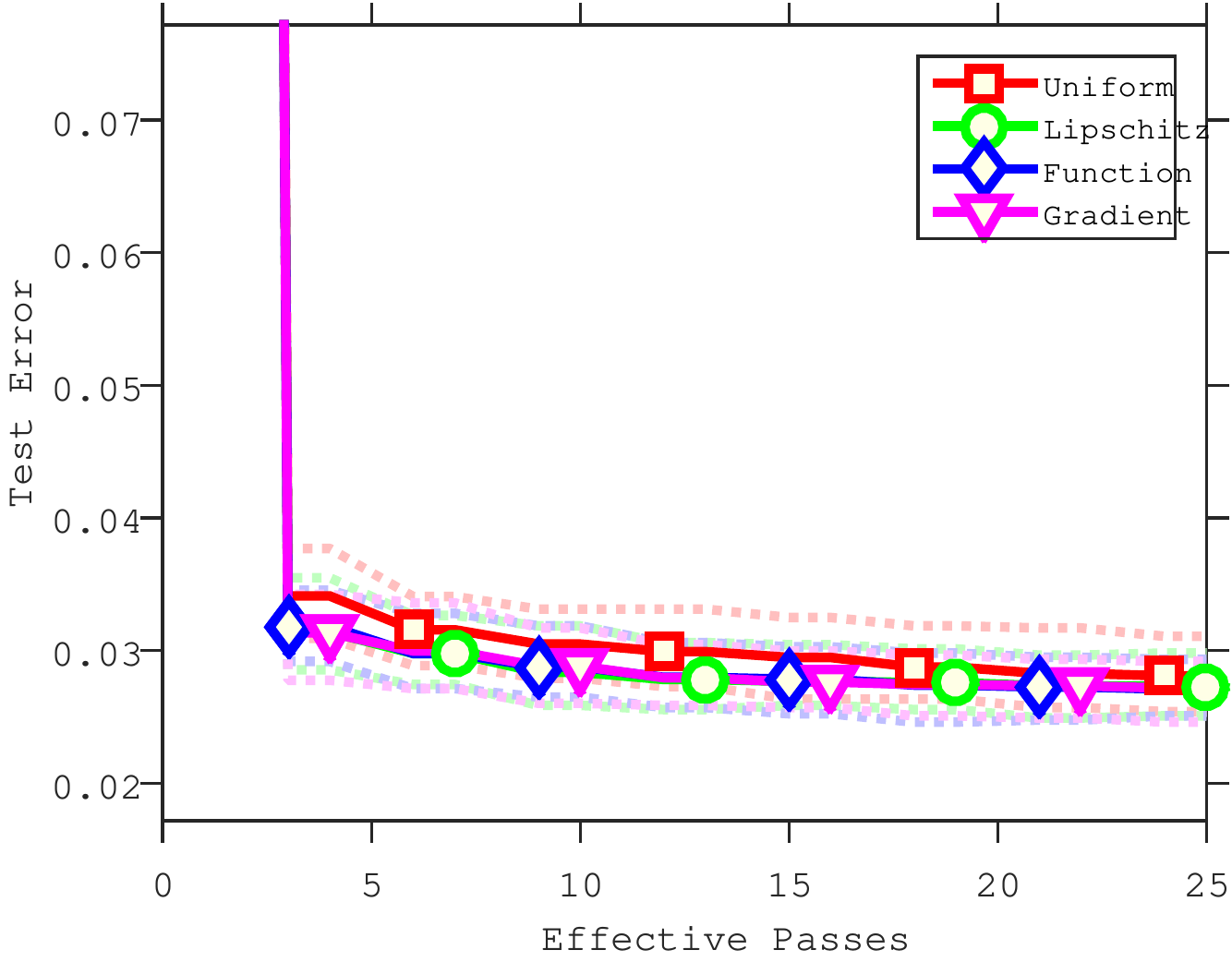}\\
\includegraphics[width=.32\textwidth]{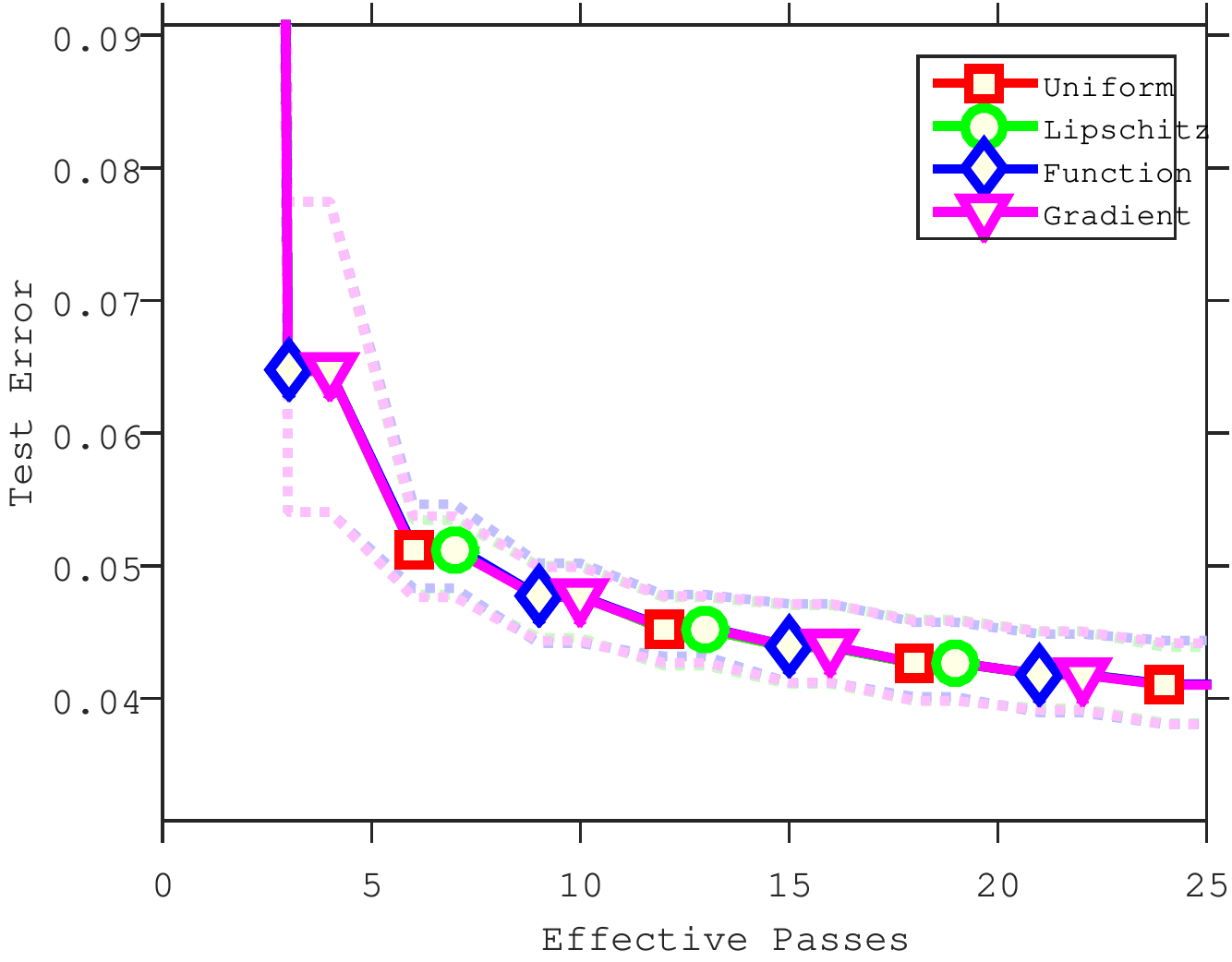}
\includegraphics[width=.32\textwidth]{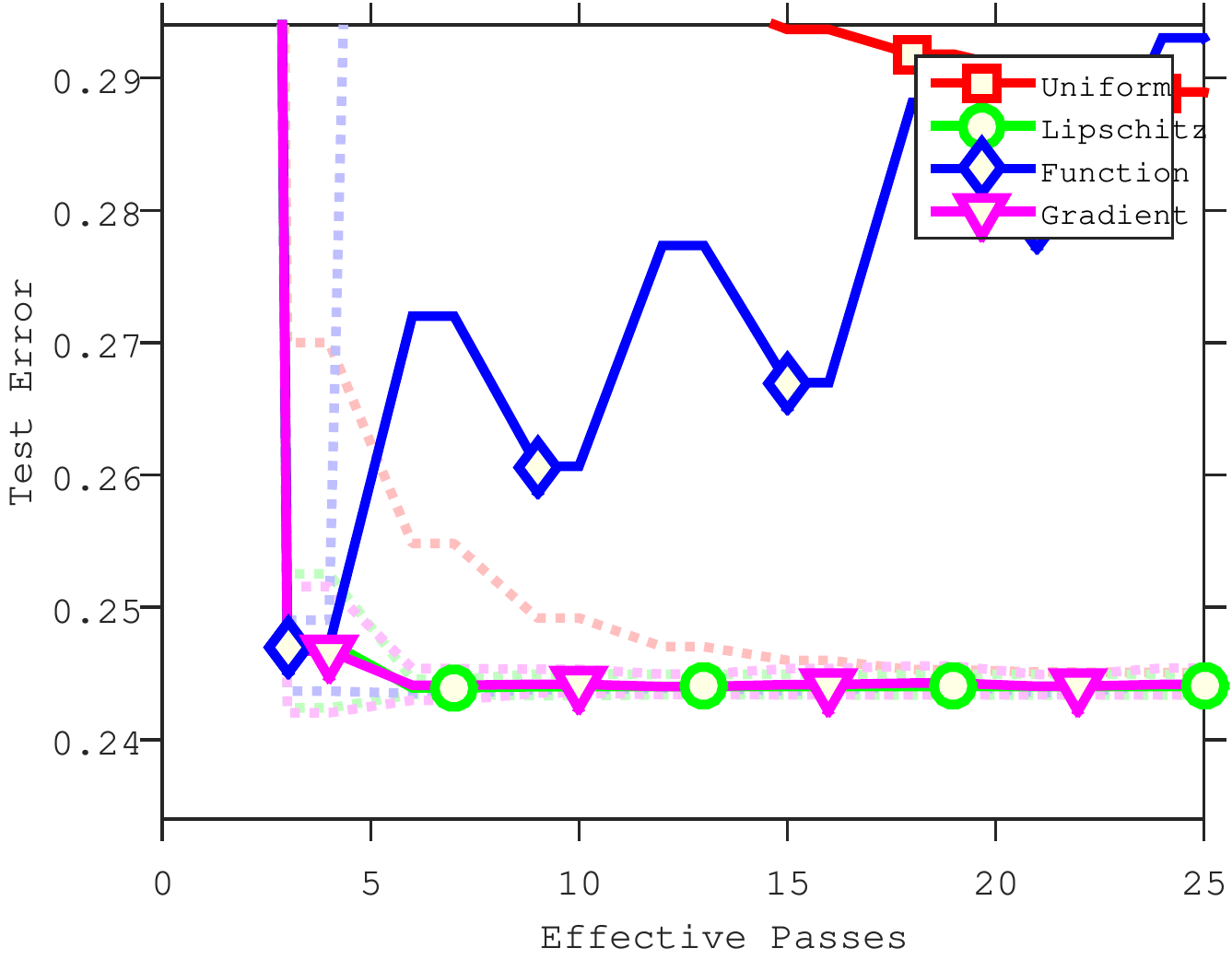}
\includegraphics[width=.32\textwidth]{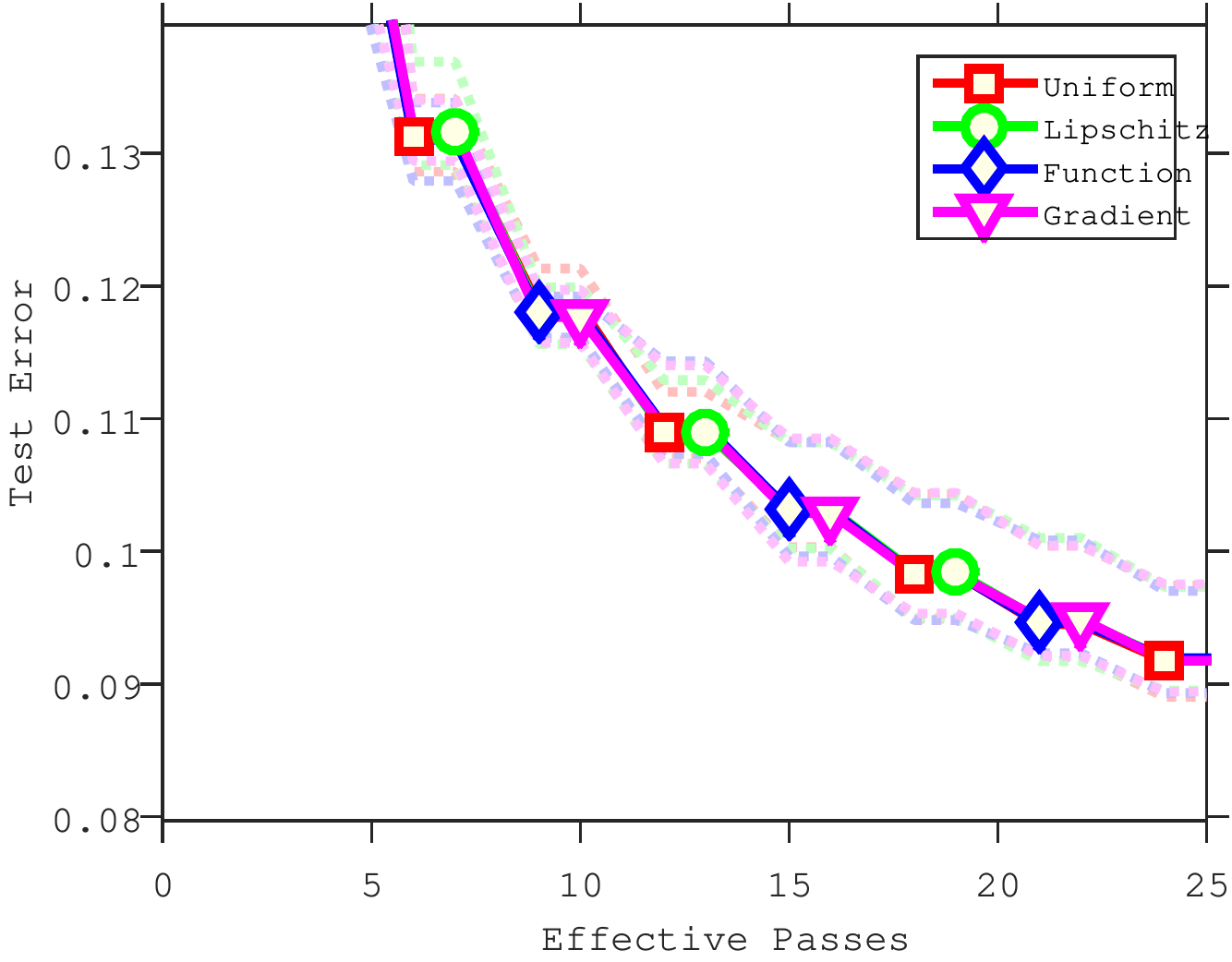}\\
\includegraphics[width=.32\textwidth]{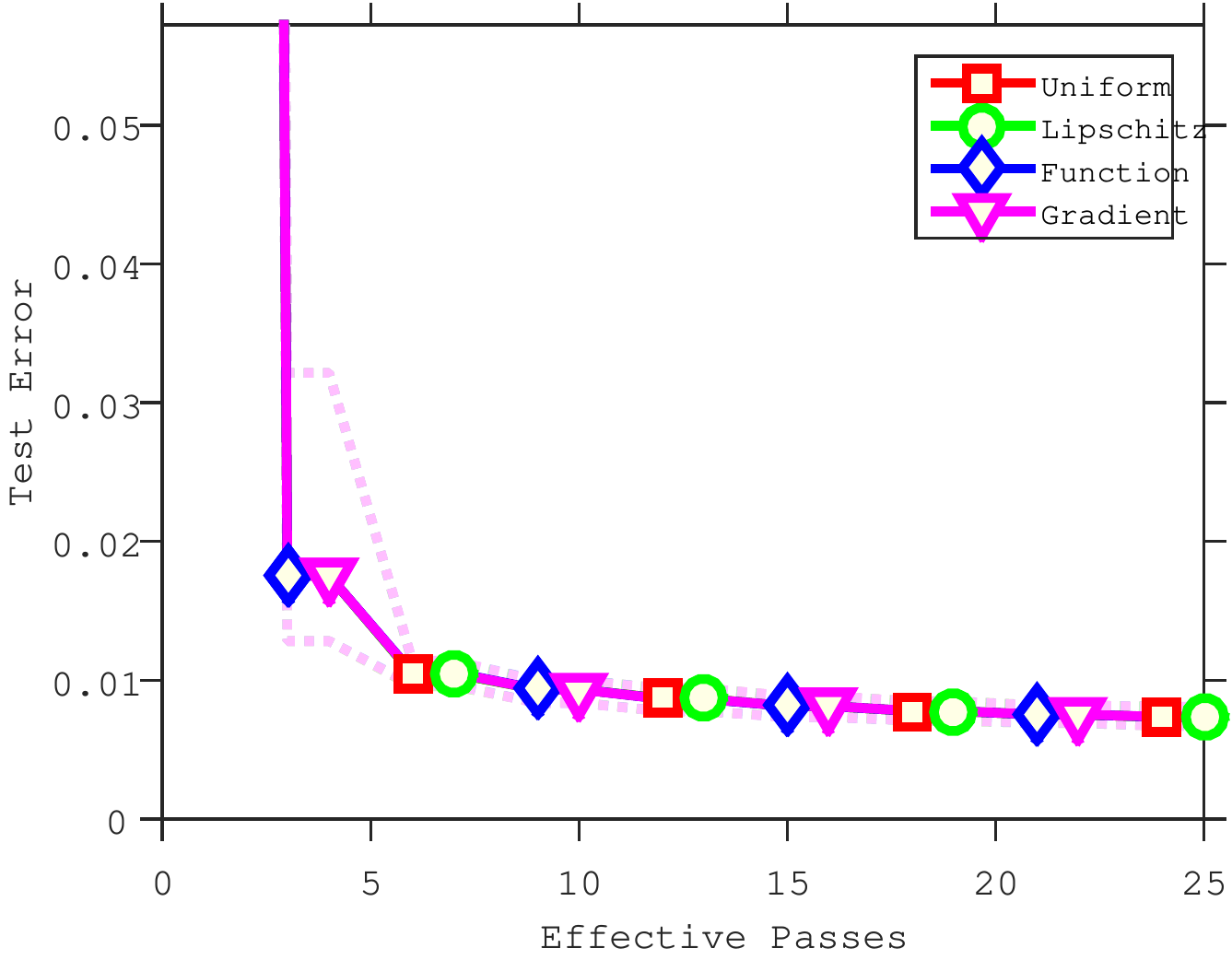}
\includegraphics[width=.32\textwidth]{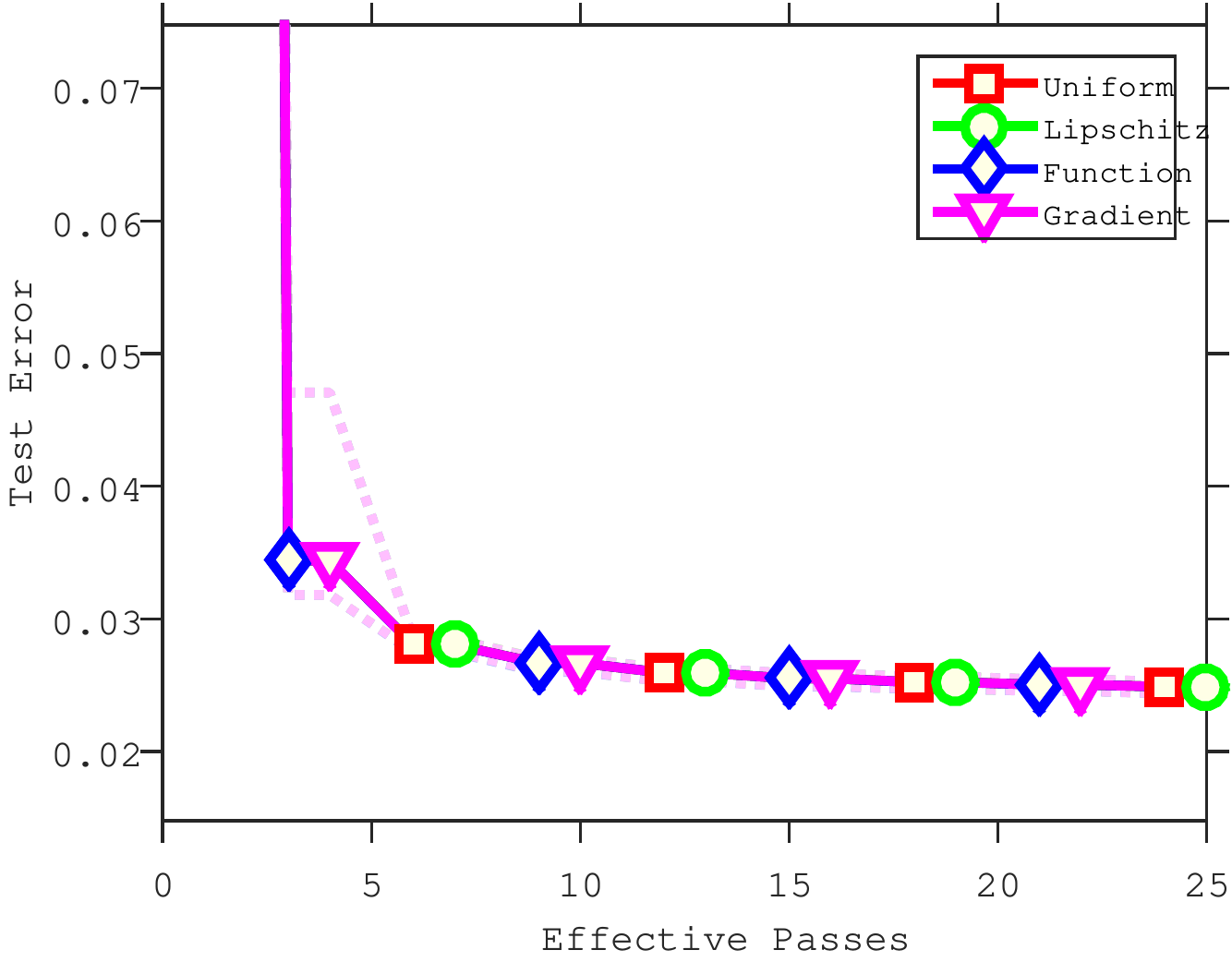}
\includegraphics[width=.32\textwidth]{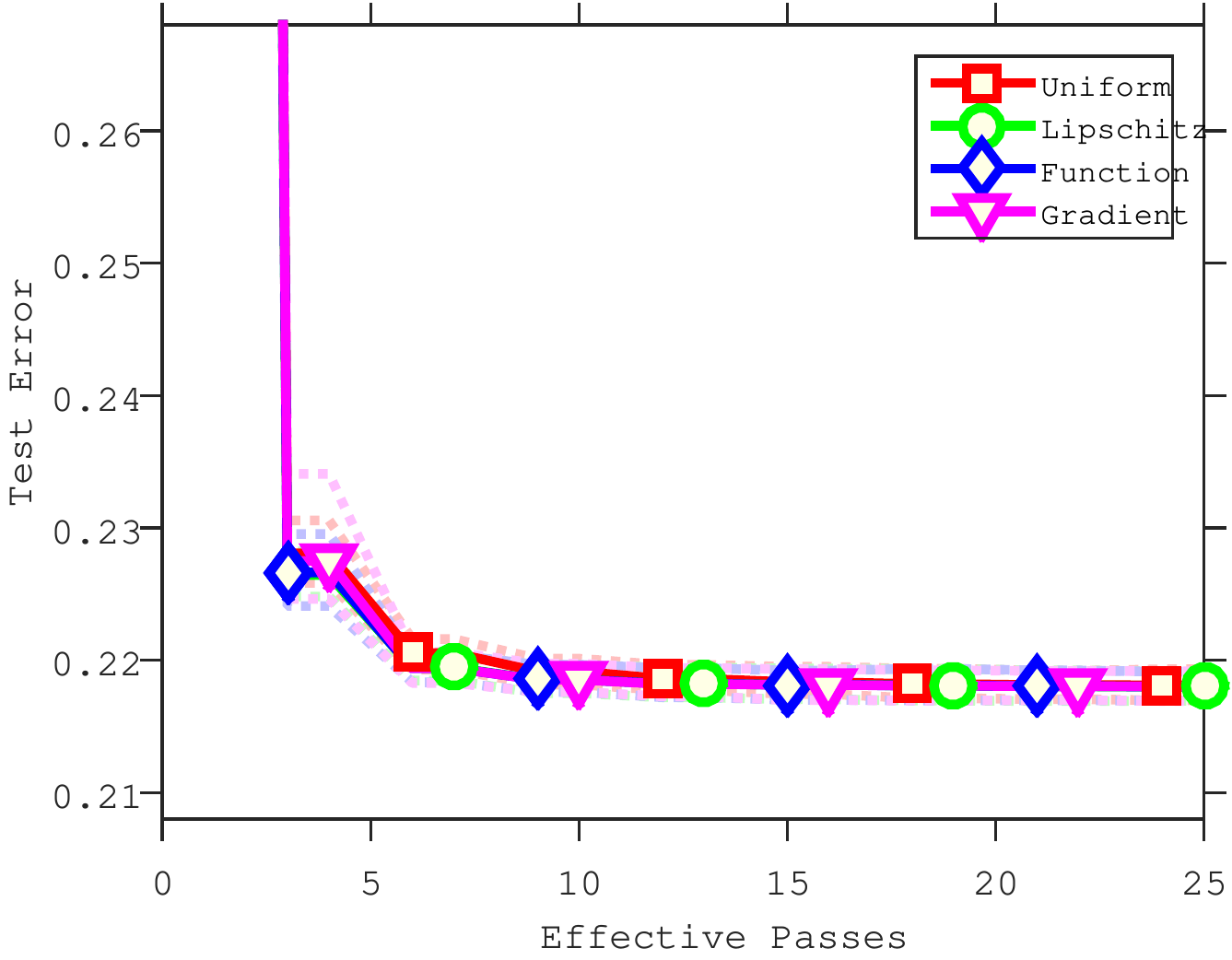}

\caption{Comparison of test error of  logistic regression with different mini-batch strategies.  The top row gives results on the \emph{quantum} (left), \emph{protein} (center) and \emph{sido} (right) datasets. The middle row gives results on the \emph{rcv11} (left), \emph{covertype} (center) and \emph{news} (right) datasets.  The bottom row gives results on the \emph{spam} (left), \emph{rcv1Full} (center), and \emph{alpha} (right) datasets.}
%\label{fig:6}
\end{figure*}

\bibliographystyle{ieeetr}
\bibliography{bib}

\end{document}